\documentclass[11pt,a4paper]{article}
\usepackage{xcolor}
\usepackage{booktabs}
\usepackage{float}
\usepackage{wrapfig}

\usepackage{color}
\usepackage{amsmath}
\usepackage{amssymb}
\usepackage{amsthm}
\usepackage{mathtools}
\usepackage{thmtools}
\usepackage[mathscr]{eucal}
\usepackage{bm}
\usepackage{bbm}
\usepackage{relsize}
\usepackage{graphicx}
\usepackage{multirow}
\usepackage{multicol}
\usepackage[toc,page]{appendix}
\usepackage{ml-defs}
\usepackage{annote-equations}

\makeatletter
\AddToHook{package/fancyhdr/after}{
  \@ifpackagelater{fancyhdr}{2022/01/01}{}{
    \renewcommand\fancyhead[2][]{
      \expandafter\long\expandafter\gdef\csname nxai@fh@#1\endcsname{#2}
      \f@nch@fancyhf\fancyhead h[#1]{\csname nxai@fh@#1\endcsname}}}}
\makeatother

\usepackage[nocustomfonts]{nxai}
\newcommand{\affmark}[1]{\raisebox{0.55ex}{\fontsize{7}{7}\selectfont #1}}

\usepackage{enumitem}

\usepackage{hyperref}
\usepackage{url}
\newcommand{\vertiii}[1]{{\left\vert\kern-0.25ex\left\vert\kern-0.25ex\left\vert #1
    \right\vert\kern-0.25ex\right\vert\kern-0.25ex\right\vert}}

\definecolor{linkcol}{HTML}{3073AD}
\definecolor{citecol}{HTML}{3073AD}
\definecolor{urlcol}{HTML}{3073AD}
\hypersetup{
   colorlinks=true,
   linkcolor=linkcol,
   citecolor=citecol,
   urlcolor=urlcol,
}

\title{Kronecker Delta Rule for Long Context Linear Attention}

\nxaisetalogo{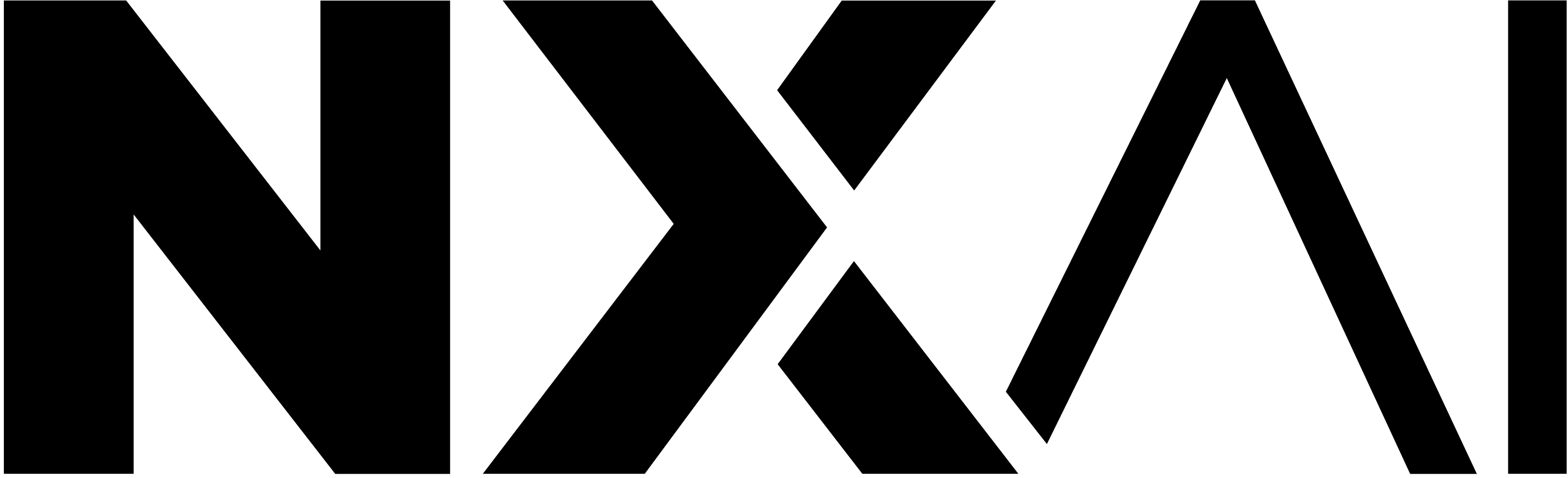}
\nxaisetlogoheight{4mm}
\nxaisetdate{}
\nxaisetfooter{Corresponding authors: Anamaria-Roberta Hartl, hartl@ml.jku.at, Levente Z\'olyomi, levente.zolyomi@nx-ai.com}
\author{Anamaria-Roberta Hartl*\affmark{1}\and Levente Z\'olyomi*\affmark{1,2}\and Michael List\affmark{1}\and G\"unter Klambauer\affmark{1,2}\and Sebastian B\"ock\affmark{2}\and Sepp Hochreiter\affmark{1,2}}
\nxaisetaffiliations{
\affmark{1} Institute for Machine Learning, Johannes Kepler University, Linz 
\affmark{2} NXAI GmbH, Linz 
\affmark{*} Equal contribution.
}

\nxaiabstract{
We propose KronLSTM, a linear attention architecture based on the Kronecker Delta Rule, a novel update rule for associative matrix memories. KronLSTM extends the xLSTM family with a memory-revision mechanism that corrects targeted associations while modifying the memory only along a single direction in matrix space. This minimally invasive memory revision reduces interference with previously stored information compared with delta-rule updates, making it particularly well suited for long-context language modeling. The resulting memory dynamics can be mathematically described as a Markov chain that converges to a steady-state distribution, while the expected norm of the memory matrix remains uniformly bounded throughout the entire input sequence. 
Hence, KronLSTM ensures stable memory dynamics without requiring additional normalization and stabilization mechanisms for the memory state.
Experiments show that KronLSTM is competitive with, and often outperforms, existing linear attention architectures, including models of the (Gated) DeltaNet family, with particularly strong performance on long context benchmarks. These results establish the Kronecker Delta Rule as an effective and analytically tractable foundation for stable associative memory in recurrent linear attention architectures.
Our \href{https://github.com/NX-AI/KronLSTM}{code} and \href{https://huggingface.co/collections/NX-AI/kronlstm}{models} are open-source.
}

\begin{document}

\maketitle

\begin{figure}[H]
  \centering
  \includegraphics[width=0.85\textwidth]{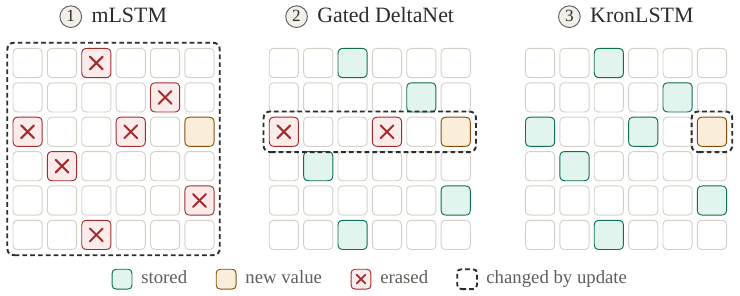}
  \caption{Representation of the memory matrix state revision for mLSTM, Gated DeltaNet, and KronLSTM. 
Interference with stored values: mLSTM affects all values, Gated DeltaNet affects those with overlapping keys, and KronLSTM affects only those with overlapping key and value pairs.}
  \label{fig:main-figure}
\end{figure}

\section{Introduction}
\label{sec:introduction}

\textbf{Large language models (LLMs)} 
predominantly use the Transformer architecture \citep{vaswani17attention,touvron2023llama,bai2023qwen,openai2025gpt-oss},
whose softmax attention mechanism 
has computational costs that grow quadratically with
context length, 
while its key--value cache grows linearly.
Consequently, inference becomes increasingly costly as contexts grow.
To reduce these costs, recent work has explored replacing softmax attention
with linear attention mechanisms \citep{katharopoulos20transformers,Schlag:21},
which can be traced back to the idea of fast weight programmers
\citep{Schmidhuber:92}. 
Linear attention compresses the context
into a fixed-size memory matrix that is updated as new tokens arrive,
yielding linear computational complexity and a constant memory footprint.
Recently, recurrent architectures based on associative memories
\citep{Kohonen:72,Anderson:72}, including xLSTM
\citep{beck24xlstm}, as well as related state-space and memory-based
approaches such as Mamba \citep{gu2024mamba,dao2024mamba2,lahoti2026mamba3},
DeltaNet \citep{Yang:24,yang2025gated,kimiteam2025kda,hatamizadeh2026gdn2}, RWKV-5 and 6 \citep{peng2024eagleArxiv},
and other recent architectures \citep{Sun:24,Liu:25,Behrouz:25titans},
have emerged as competitive alternatives to Transformers.

\begin{wrapfigure}{R}{0.50\linewidth}
\vspace{-\intextsep}
\centering
\includegraphics[width=\linewidth]{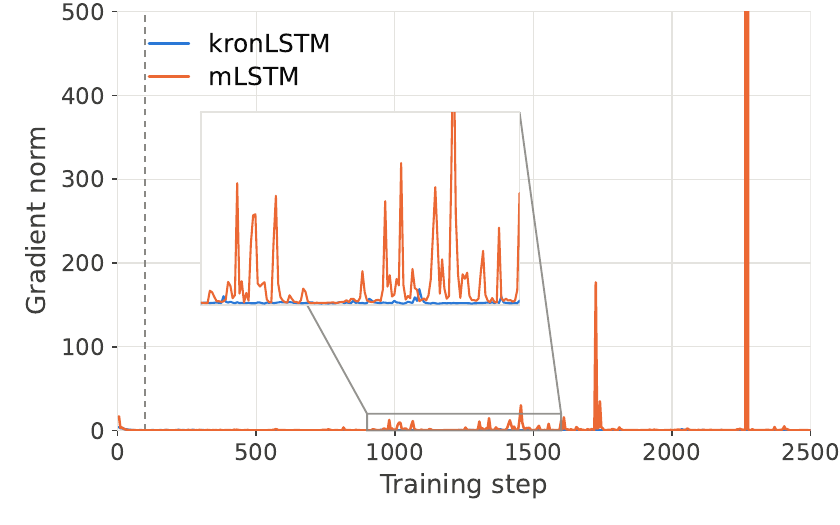}
\caption{Stress test with high learning rates for the same 
 mLSTM and KronLSTM architectures on language modeling. 
 Learning becomes unstable for mLSTM 
as the gradient norm spikes show, while KronLSTM yields stable learning.}
\label{fig:stress}
\end{wrapfigure}

\textbf{A crucial challenge for linear attention is controlling
memory growth.} 
If the memory of recurrent
architectures grows without being bounded, its readout can become
arbitrarily large, leading to numerical instabilities,
particularly during learning. The original LSTM
\citep{Hochreiter:91,hochreiter97lstm} uses a bounded squashing function to
control the exposed cell state, while its forget gate regulates the
accumulation of information in memory. RWKV-5 and RWKV-6
\citep{peng2024eagleArxiv} similarly control memory growth
through a forget gate. Mamba architectures
\citep{gu2024mamba,dao2024mamba2,lahoti2026mamba3} also restrain
memory growth through a forget mechanism derived from a stable
state-space ordinary differential equation.
However, forget-based memory control becomes increasingly restrictive for
long contexts. To retain information over long time spans, forget gates
must remain close to one; otherwise, information from the
beginning of the context is progressively attenuated. xLSTM
\citep{beck24xlstm} addresses this limitation by allowing forget gates close to
one and introducing normalizer and stabilizer states to control the scale
of its matrix memory and readout. However, normalization based on ratios
of potentially large quantities can itself lead to numerical difficulties
and learning instabilities, particularly in larger xLSTM networks
\citep{Beck:25xlstm7b}. 
mLSTM's numerical instabilities are illustrated by the stress test in
Figure~\ref{fig:stress}.

\textbf{The DeltaNet family} \citep{Yang:24,yang2025gated,kimiteam2025kda,hatamizadeh2026gdn2} controls memory growth through
Delta Rule updates, whose dynamics converge to a steady-state distribution
with bounded memory ensuring stable learning dynamics. 
However, revising a stored value can perturb previously
stored values, leading to unwanted memory interference for long sequences.
We introduce KronLSTM, which, like (Gated) DeltaNet, provides bounded memory
through a delta-style update, but uses the Kronecker Delta Rule.
The Kronecker Delta Rule is less invasive: it modifies the memory only along
a single direction in matrix space, reducing interference with
previously stored associations. This targeted memory revision results
in a less noisy gradient flow over
time.

Our contributions are: 
\begin{itemize}[topsep=0pt]
    \item \textbf{(i) Kronecker Delta Rule} is introduced as a new update rule for associative matrix memories, which corrects the scalar associations while having minimal interference with previously stored values. 
    \item \textbf{(ii) KronLSTM} is introduced as a novel linear attention architecture based on the Kronecker Delta Rule equipped with a vector forget gate, a $\beta$ gate, and a target gate.
    \item \textbf{(iii) Stability properties} for the Kronecker Delta Rule are mathematically proved via bounds on the expected memory norm.  These bounds ensure stable learning and proper gradient flow. 
    \item \textbf{(iv) Experiments} empirically validate the strong performance of KronLSTM at long-context language modeling and retrieval tasks. KronLSTM performs strongly across language modeling, recall, and state tracking benchmarks.
\end{itemize}

\section{KronLSTM}
\label{sec:KronLSTM}

In Subsection~\ref{subsec:update-rule}, 
we introduce the Kronecker Delta Rule and analyze its memory dynamics, and \\ in Subsection~\ref{subsec:kronlstm-block}, the Kronecker Delta Rule is embedded in the KronLSTM sequence-mixer block.

\subsection{Kronecker Delta Rule}
\label{subsec:update-rule}

\paragraph{Linear attention update rules for matrix memory.}
Linear attention \citep{katharopoulos20transformers} compresses the context
into a matrix memory $\BS$ that is updated per token and read out per query,
\begin{equation} \label{eq:linatt}
\BS_{t+1} \;=\; \BS_t \;+\; \Bv_t \Bk_t^\top \, ,
\qquad
\Bh_t \;=\; \BS_t \Bq_t \;=\; \sum_{i < t} \Bv_i \,\big(\Bk_i^\top \Bq_t\big) \, ,
\end{equation}
with $\BS_t \in \dR^{d_v \times d_k}$, keys $\Bk_t$ and queries $\Bq_t$ in
$\dR^{d_k}$, and values $\Bv_t$ in $\dR^{d_v}$. Therefore, $\BS$ is an associative
memory \citep{Kohonen:72,Anderson:72} with a Hebbian rule, i.e.,\, 
an outer product update rule.
Modern linear attention architectures like 
xLSTM \citep{beck24xlstm},
Mamba \citep{gu2024mamba,dao2024mamba2,lahoti2026mamba3},
RWKV-5 and 6 \citep{peng2024eagleArxiv}
rely on these associative memories equipped with gates and transformations.
The memory growth is controlled by a forget gate, which
stabilizes the memory dynamics.

In contrast, the Delta Rule
\citep{Widrow:60,Schlag:21} corrects the memory with the 
difference of the retrieval to the current value when writing into the memory
\begin{equation}
\BS_{t+1} \;=\; \BS_t \;-\; \beta_t \big(\BS_t \Bk_t - \Bv_t\big)\,\Bk_t^\top
\;=\; \BS_t \big(\BI - \beta_t\, \Bk_t \Bk_t^\top\big) \;+\; \beta_t\, \Bv_t \Bk_t^\top ,
\label{eq:deltarule}
\end{equation}
with step size $\beta_t$. Each step is gradient descent on
$\tfrac12 \norm{\BS_t \Bk_t - \Bv_t}^2$, and the transition 
$\BI - \beta_t \Bk_t \Bk_t^\top$ is a generalized Householder matrix, which admits chunkwise-parallel training
\citep{Yang:24}. 

The two rules differ in how an association, once stored, is changed.
The Hebbian write of \eqref{eq:linatt} is additive, so the memory grows with
the sequence and an earlier value can be buried by later writes but never
revised. The delta rule \eqref{eq:deltarule} revises: it removes the value
currently stored at $\Bk_t$ before writing $\Bv_t$. The correction, however,
proceeds along $(\BS_t \Bk_t - \Bv_t)\Bk_t^\top$, a direction determined by
the current memory content, so values stored at keys that overlap with
$\Bk_t$ are affected as well. We aim at an error-correcting write 
that acts solely on the pair $(\Bk_t, \Bv_t)$.

\paragraph{Kronecker Delta Rule.}
Let $\BS_t \in \dR^{d_v \times d_k}$ be the memory matrix, $\Bk_t \in \dR^{d_k}$ and
$\Bv_t \in \dR^{d_v}$ the key and value at step $t$, $c_t \in \dR$ a target
and $\beta_t$ a step size. The scalar $\Bv_t^\top \BS_t \Bk_t$ measures how
strongly the memory associates $\Bk_t$ with $\Bv_t$. 
The Kronecker Delta Rule corrects this
scalar towards the target along the outer product of the pair as follows
\begin{equation} \label{eq:KDR}
\BS_{t+1} \;=\; \BS_t \;-\; \beta_t \big(\Bv_t^\top \BS_t \Bk_t \, - \, c_t\big)\,
\Bv_t \Bk_t^\top  \;=\;
\BS_t
\;-\;
\underbrace{\beta_t \big(\Bv_t^\top \BS_t \Bk_t\big)\, \Bv_t \Bk_t^\top}_{\text{erase}}
\;+\;
\underbrace{\beta_t \, c_t\, \Bv_t \Bk_t^\top}_{\text{write}} \, . 
\end{equation}
In place of the $d_v$-dimensional residual $\BS_t \Bk_t - \Bv_t$ of the delta
rule, a scalar is corrected, and the direction $\Bv_t \Bk_t^\top$ of
the correction is determined by the input pair alone without 
utilizing the memory content.
A target $c_t \equiv 1$ stores each association with unit strength,
while $c_t = 0$ erases it. The target $c_t$ may depend on the input, 
i.e.,\ serves as a gate, 
so that what is written becomes input-dependent. 
The right-hand side of \eqref{eq:KDR} separates
the update into an erase and a write term, both scaled by $\beta_t$. Choosing
the target $\tilde{c}_t = c_t / \beta_t$ removes the step size from the write,
so that $\beta_t$ alone determines how much of the stored association is
erased and $c_t$ alone how much is written. KronLSTM uses this form, with both
quantities computed from the input as gates
(Section~\ref{subsec:kronlstm-block}). 

\paragraph{Properties of the Kronecker Delta Rule.} 
{\em (i) Error correction} (see Appendix Property~\ref{prop:Agradient} 
and Property~\ref{prop:Atimecontraction}). An update is a step of gradient descent with step
size $\beta_t$ on $\tfrac12 \big(\Bv_t^\top \BS \Bk_t - c_t\big)^2$.
Re-presenting the same pair gives
$\Bv_t^\top \BS_{t+1} \Bk_t - c_t = \left( \Bv_t^\top \BS_t \Bk_t - c_t \right) (1 - \beta_t \norm{\Ba_t}^2)$, so it
contracts the error iff $0 < \beta_t < 2 / \norm{\Ba_t}^2$, and the Newton step
$\beta_t = 1 / \norm{\Ba_t}^2$ meets the target in a single update; for
$c_t = 0$ this deletes the association. 
{\em (ii) Delta Rule in vectorized form} (see Appendix Property~\ref{prop:Avector}). 
With the Kronecker product $\Ba_t = \Bk_t \otimes \Bv_t = \vecop(\Bv_t \Bk_t^\top)$
and the vectorized memory $\Bs = \vecop(\BS)$, we have
$\Bv_t^\top \BS \Bk_t = \Ba_t^\top \Bs$ and
$\norm{\Ba_t}^2 = \norm{\Bk_t}^2 \norm{\Bv_t}^2$, and \eqref{eq:KDR}
becomes
\begin{equation} \label{eq:vecrule}
\Bs_{t+1} \; = \; \big(\BI - \beta_t\, \Ba_t \Ba_t^\top\big)\, \Bs_t
\; + \; \beta_t\, c_t\, \Ba_t \, ,
\end{equation}
the Widrow--Hoff Delta Rule \citep{Widrow:60} with input $\Ba_t$ and target
$c_t$, which gives the rule its name. The transition factor is a generalized
Householder matrix that does not depend on $c_t$, so the chunkwise-parallel
formulation of DeltaNet \citep{Yang:24} applies with $\Ba_t$ in
place of $\Bk_t$, the target entering only the driving term. For
$\beta_t \norm{\Ba_t}^2 \in (1,2)$ the transition factor has a negative eigenvalue, the
idea behind the state-tracking capabilities of Delta Rule models
\citep{grazzi25unlocking}.

{\em (iii) Interference} (see Appendix Property~\ref{prop:Ainterference}). For any other stored pair $(\Bk_i, \Bv_i)$, the update
changes the association by
\begin{equation} \label{eq:interference}
\Delta\big(\Bv_i^\top \BS \Bk_i\big) \; = \;
-\,\beta_t\, \big(\Bv_t^\top \BS_t \Bk_t - c_t\big)\,
\big(\Bv_i^\top \Bv_t\big)\big(\Bk_t^\top \Bk_i\big) \, ,
\end{equation}
the product of value overlap and key overlap. Orthogonality in either factor
protects a stored association, whereas the Hebbian and the Delta
rule interference is controlled by key overlap alone.

{\em (iv) Capacity} (see Appendix Property~\ref{prop:Atimecapacity}). Each triple $(\Bk_t, \Bv_t, c_t)$ imposes one linear
constraint $\Ba_t^\top \Bs = c_t$ on the $d_k d_v$ entries of $\BS$, so up to
$d_k d_v$ scalar associations can be stored exactly when the features
$\{\Ba_i\}$ are linearly independent, against $d_k$ value vectors for
$\BS \Bk_i = \Bv_i$. In particular, a single key can hold several mutually
orthogonal values, which the read-out $\BS \Bk$ returns as their sum.
See Appendix~\ref{sec:Akdr} for more properties and more details on the above properties. 

\paragraph{Convergence, steady state, and memory norm bounds.}
We will present expected bounds on the memory norm, which
in turn imply stable memory dynamics that obviate
numerical instabilities.
Memories grow with the length of the context and the number of vectors or associations to store; therefore
long contexts are most relevant when analyzing memory stability.
We model a long context as a stream of i.i.d.\ triples
$(\Bk_t,\Bv_t,c_t) \sim p$ with fixed step size $\beta$ and set
\begin{equation} \label{eq:ssetting}
\bar{\BA} := \EXP[\Ba \Ba^\top] \succ \BZe ,
\quad
\bar{\Bb} := \EXP[c\, \Ba] ,
\quad
\Bs^{\rL\rS} := \bar{\BA}^{-1} \bar{\Bb} ,
\quad
\rho := c - \Ba^\top \Bs^{\rL\rS} ,
\quad
\kappa := \beta \max_i \norm{\Ba_i}^2 ,
\end{equation}
so $\BS^{\rL\rS} = \argmin_{\BS} \EXP[(\Bv^\top \BS \Bk - c)^2]$ is the
least-squares memory and $\rho$ the irreducible residual, with
$\EXP[\rho\,\Ba] = \BZe$. Throughout, $0 < \beta\norm{\Ba_i}^2 < 2$, and
$q_\beta := 1 - \beta(2-\kappa)\lambda_{\min}(\bar{\BA})$.

\begin{theorem}[Unbiased mean and steady state]
\label{thm:leastsquares}
The expected iterate obeys the {\em exact} recursion
$\EXP[\Bs_{t+1}] = (\BI - \beta\bar{\BA})\EXP[\Bs_t] + \beta\bar{\Bb}$,
whose unique fixed point is $\Bs^{\rL\rS}$, so
$\norm{\EXP[\BS_t] - \BS^{\rL\rS}}_{\rF} \le r^t
\norm{\BS_0 - \BS^{\rL\rS}}_{\rF}$ with
$r = \max_k \abs{1 - \beta\lambda_k(\bar{\BA})} < 1$. Moreover the
iterates form a time-homogeneous Markov chain on $\dR^{d_k d_v}$ whose
kernel contracts on laws with finite second moment; by the standard
theory of Markov chains on general state spaces \citep{Meyn:93} it has a
{\em unique} stationary distribution $\pi_\beta$, the law of $\BS_t$
converges to it geometrically from any initialization, and the
stationary mean is {\em exactly} $\EXP_{\pi_\beta}[\BS] = \BS^{\rL\rS}$.
\end{theorem}
Theorem~\ref{thm:leastsquares} combines Appendix Theorem~\ref{thm:Amean}
and Appendix Theorem~\ref{thm:Asteady} which have been proven in
Appendix~\ref{sec:Akdr}.

\begin{theorem}[Memory norm bounds]
\label{thm:bounded}
With $\lambda_{\min}, \lambda_{\max}$ the extreme eigenvalues of
$\bar{\BA}$, the least-squares memory is sandwiched by the
{\em explained} target energy,
\begin{equation} \label{eq:twosided}
\frac{\EXP[c^2] - \EXP[\rho^2]}{\lambda_{\max}(\bar{\BA})}
\;\; \le \;\; \norm{\BS^{\rL\rS}}_{\rF}^2 \;\; \le \;\;
\frac{\EXP[c^2] - \EXP[\rho^2]}{\lambda_{\min}(\bar{\BA})} \, ,
\end{equation}
and all its singular values obey the same upper bound; residual energy
can only shrink it. In the non-expansive regime $\kappa \le 1$, the
steady state and --- started at $\BS_0 = \BZe$ --- every iterate satisfies
\begin{equation} \label{eq:cleanone}
\EXP_{\pi_\beta}\big[\norm{\BS}_{\rF}^2\big]
\; \le \; \frac{\EXP[c^2]}{\lambda_{\min}(\bar{\BA})} \, ,
\qquad
\EXP\big[\norm{\BS_t}_{\rF}^2\big]
\; \le \;
\Big(1 + q_\beta^{\,t/2}\Big)^{2}
\frac{\EXP[c^2]}{\lambda_{\min}(\bar{\BA})}
\; \le \;
\frac{4\, \EXP[c^2]}{\lambda_{\min}(\bar{\BA})} \, ,
\end{equation}
the transient factor decays geometrically from $4$ to $1$.
\end{theorem}
Theorem~\ref{thm:bounded} combines Appendix Theorem~\ref{thm:Anorm},
Appendix Theorem~\ref{th:normbound}, and Appendix Theorem~\ref{th:online}
which are proven in
Appendix~\ref{sec:Akdr}.
From the first step onward, the expected memory of the rule is
bounded by a constant determined by the data distribution alone.
More details and theoretical results are
given in Appendix~\ref{sec:Akdr}.

\subsection{The KronLSTM architecture}
\label{subsec:kronlstm-block}

In modern deep learning architectures for sequence modeling,
the matrix update rules are stabilized and enhanced with gating mechanisms.
The mLSTM \citep{beck24xlstm} places the Hebbian write \eqref{eq:linatt}
under a sigmoid forget gate and an exponential input gate, and Gated DeltaNet
\citep{yang2025gated} decays the memory of \eqref{eq:deltarule} by a scalar
before the correction,
\begin{equation}
\BS_{t+1} \;=\; \Rf_t\, \BS_t \;+\; \Ri_t\, \Bv_t \Bk_t^\top \, ,
\qquad
\BS_{t+1} \;=\; \Rf_t\, \BS_t\,\big(\BI - \beta_t\, \Bk_t \Bk_t^\top\big)
\;+\; \beta_t\, \Bv_t \Bk_t^\top \, .
\label{eq:gated-baselines}
\end{equation}
Kimi Delta Attention \citep{kimiteam2025kda} makes the decay (key) channel-wise, and
Gated DeltaNet-2 \citep{hatamizadeh2026gdn2} replaces the single $\beta_t$ by
separate channel-wise erase and write gates. KronLSTM proceeds in the same way: the step
size $\beta_t$ and the target $c_t$ of \eqref{eq:KDR} become gates
computed from the input; a forget gate decays the memory, and the rule is
embedded in the projections, short convolutions, normalizations, and output
gate of a standard sequence-mixer layer. All quantities below are per head,
with key and value dimensions $d_k$ and $d_v$; we suppress the head index and
write $\Bx_t \in \dR^{d}$ for the layer input at step $t$.

\paragraph{KronLSTM update.}
The forget gate is a diagonal matrix
$\BF_t$ with diagonal entries in $(0,1)$. It
acts on the key dimension and hence decays the memory from the right,
$\BS_t \BF_t$, retaining each key direction separately. 
We use target $\tilde{c}_t=c_t / \beta_t$ to decouple erase and write.
The erase term of
\eqref{eq:KDR} is evaluated on the decayed memory $\BS_t \BF_t$. We obtain
\begin{equation} \label{eq:gated-rule}
\BS_{t+1} \;=\;
\underbrace{\BS_t \BF_t}_{\text{forget}}
\;-\;
\underbrace{\beta_t \big(\Bv_t^\top \BS_t \BF_t\, \Bk_t\big)\, \Bv_t \Bk_t^\top}_{\text{erase}}
\;+\;
\underbrace{c_t\, \Bv_t \Bk_t^\top}_{\text{write}} ,
\end{equation}
so that forgetting, erasing, and writing are controlled by three separate
gates, $\BF_t$, $\beta_t$ and $c_t$.
The standard coupled and the unit-target variant $c_t \equiv 1$ are compared
in Appendix~\ref{sec:Afurther-results}, Table \ref{tab:write-rule-ablation}.

\paragraph{Gates.}
The erase gate $\beta_t$ and the write gate $c_t$ are per-head scalars obtained
from linear projections of the input followed by a sigmoid,
\begin{equation}
\beta_t = \sigma\big(\Bw_\beta^\top \Bx_t + b_\beta\big) \in (0,1), \qquad
c_t = \sigma\big(\Bw_{\Rc}^\top \Bx_t + b_{\Rc}\big) \in (0,1),
\label{eq:beta-c-gates}
\end{equation}
both bounded by construction. The diagonal vector $\bff_t$ of the forget gate
$\BF_t = \diag(\bff_t)$ is obtained by raising a sigmoid gate to a learned
positive per-head power, following the parameterization of Mamba
\citep{gu2024mamba} and Gated DeltaNet-2 \citep{hatamizadeh2026gdn2}
\begin{equation}  \label{eq:forget-gate}
  \bff_t \;=\; \Bsi(-(\Bs_{\Rf})_t)^{a}
           \;=\; \big(1 - \Bsi((\Bs_{\Rf})_t)\big)^{a},
  \ \
  (\Bs_{\Rf})_t = \BW_{\Rf} \Bx_t + \Bb_{\Rf} \in \dR^{d_k},
  \ \
  \alpha = \exp(a_{\log}) \in \dR_{>0},
\end{equation}
where $a_{\log}$ is a single learned scalar per head, broadcast over the $d_k$
key dimensions, and the exponent is taken component-wise. The sigmoid sets retention separately for each key direction as a function of
the input. At the same time, the exponent rescales the head's retention horizon by $1/a$,
so that heads can operate at widely different timescales without high
pre-activations. Since $a_{\log}$ is learned in log space, the baseline
retention timescale spans orders of magnitude across heads, while within a
head the per-direction variation comes entirely from the input. The read-out
is gated by a sigmoid output gate
$\bfo_t = \Bsi(\BW_{\Ro} \Bx_t + \Bb_{\Ro}) \in (0,1)^{d_v}$.

\begin{figure}[t]
  \centering
  \includegraphics[width=0.56\linewidth]{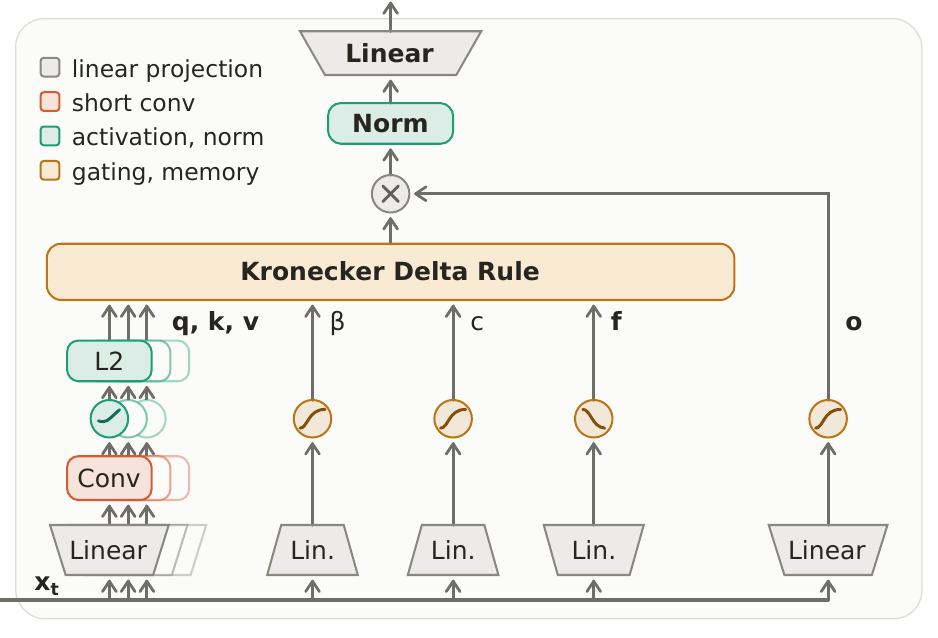}
  \caption{Architecture of the KronLSTM block.}
  \label{fig:block-structure}
\end{figure}

\paragraph{KronLSTM block.}
Queries, keys, and values are linear projections of $\Bx_t$, each followed by a
short causal depthwise convolution and a SiLU activation, and are then
$\ell_2$-normalized with $\Bq_t, \Bk_t \in \dR^{d_k}$ and $\Bv_t \in \dR^{d_v}$. The memory
$\BS_t \in \dR^{d_v \times d_k}$ is updated by \eqref{eq:gated-rule} and read
with the query, $\tilde{\Bh}_t = \BS_t \Bq_t \in \dR^{d_v}$. The read-out passes through
the sigmoid output gate and a per-head RMSNorm, and the concatenated heads are mixed by
the output projection,
\begin{equation} \nonumber
\Bh_t \;=\; \BW_{\text{out}}\, \mathrm{RMSNorm}\big(\bfo_t \odot \tilde{\Bh}_t\big).
\end{equation}\\
With $\ell_2$-normalized keys and values and a sigmoid $\beta_t$, we have
$\beta_t \norm{\Bk_t}^2 \norm{\Bv_t}^2 = \beta_t < 1$, so every update is non-expansive and
the memory operates in the regime of Theorem~\ref{thm:bounded}; the forget
gate, with entries below one, only contracts it further.
The architecture of the KronLSTM block is shown in Figure~\ref{fig:block-structure}. 
It carries no state
besides the memories $\BS_t$ of its heads: there is no normalizer state and no
stabilizer state as in the mLSTM. 
KronLSTM can serve as a drop-in replacement for the attention layer in a Transformer block,
analogous to mLSTM and sLSTM in the xLSTM architecture.

\section{Experimental Validation}
\label{sec:experiments}

We experimentally validate the KronLSTM architecture on different benchmarks
with a focus on language modeling tasks.
We compare KronLSTM with the following recurrent linear attention architectures:
Gated DeltaNet (GDN) \citep{yang2025gated}, Kimi Delta Attention (KDA) \citep{kimiteam2025kda}, Mamba-3 \citep{lahoti2026mamba3}, and Gated DeltaNet-2 (GDN-2) \citep{hatamizadeh2026gdn2}.
Recent deployments of linear attention in frontier models follow two hybrid-setting paradigms: either interleave the linear layers with a small number of softmax attention layers
\citep{nvidia25nemotron3nano,merrill2026olmohybrid,kimiteam2025kda} or use sliding-window attention (SWA) instead  \citep{cabannes2026short, hatamizadeh2026gdn2}. The advantage of the latter is keeping the memory footprint constant with respect to the context length.
Therefore, we evaluate every architecture in two settings: as a pure linear attention model, and as a hybrid with SWA layers. For the hybrid setting, we employ a window size of 2048 at a $5{:}1$ ratio
of linear attention to SWA layers.
All models are pretrained on the FineWeb-Edu dataset using the Llama2 tokenizer \citep{touvron2023llama2}.

We train at a context of 8k, motivated by previous works \citep{Guo:26loglinear, Beck:25tfla, Beck:25scaling}, where the throughput advantage of linear attention compared to the softmax attention emerges from 8k context onward. We evaluate the models within and extrapolating outside their context length (Section~\ref{subsec:retrieval}). The experiments in the following sections show 1.5B parameter models trained on 100B tokens. In Appendix~\ref{sec:Afurther-results}, we provide a set of experiments at 400M model size.

\subsection{Language Modeling}
\label{subsec:language-modeling}

We assess the performance of KronLSTM compared to leading linear attention variants on standard language modeling benchmarks, concretely on five commonsense reasoning tasks (HellaSwag \citep{zellers2019hellaswag}, PIQA \citep{bisk2020piqa}, ARC-easy and
ARC-challenge \citep{Clark2018ThinkYH}, WinoGrande \citep{sakaguchi2021winogrande}) and on LAMBADA \citep{paperno-etal-2016-lambada}, reporting accuracy on all six tasks, their average (Avg-6), and LAMBADA perplexity. Table~\ref{tab:mcqa-1p5b} shows the results of both linear and hybrid models. 
Among all competitors, KronLSTM has the highest average accuracy over the six tasks
and the lowest perplexity on LAMBADA. 
As a linear model, KronLSTM performs best on five out of six tasks 
and improves on average on the strongest competitor, KDA, by 1.6 accuracy points.
The hybrid setting confirms KronLSTM's performance advantage, leading on four of six tasks and achieving the lowest perplexity on LAMBADA.
Hybrid models with sliding-window attention change the averages by at most about one accuracy point. While mLSTM, Mamba-3 and GDN-2 benefit most from adding SWA, GDN and KDA change only minimally; KronLSTM remains robust, leading among all configurations.

\begin{table}[H]
\centering
\small
\setlength{\tabcolsep}{9pt}
\caption{Downstream performance of 1.5B/100B models.
Avg-6 is the mean accuracy over the five commonsense reasoning tasks and LAMBADA.
Best per column in \textbf{bold}, second best \underline{underlined}, separately for linear and hybrid models.}
\begin{tabular}{lcccccc|c|c}
\toprule
Model & Hella & PIQA & ARC-e & ARC-c & Wino & LMB acc & Avg-6 $\uparrow$ & LMB ppl $\downarrow$ \\
\midrule
\multicolumn{9}{c}{\textit{Linear models}} \\
\midrule
    GDN  & 57.00 & \textbf{73.61} & 70.54 & 41.04 & 57.22 & 46.94 & 57.73 & 12.58  \\
    Mamba-3 & 57.30 & 73.01 & 70.29 & 40.19 & \underline{57.46} & 45.64 & 57.31 & 12.95  \\
    KDA  & \underline{57.48} & 73.39 & \underline{71.84} & 41.55 & 57.14 & \underline{47.35} & \underline{58.12} & 12.67  \\
   GDN-2  & 57.45 & 72.96 & 71.13 & \underline{41.89} & 57.30 & 46.92 & 57.94 & \underline{12.54} \\
   mLSTM & 56.05 & 73.07 & 71.30 & 41.30 & 56.99 & 44.38 & 57.18 & 13.87 \\
   \midrule
  KronLSTM  & \textbf{59.19} & \underline{73.45} & \textbf{74.03} & \textbf{43.60} & \textbf{60.06} & \textbf{47.97} & \textbf{59.72} & \textbf{11.72} \\
  \midrule
\multicolumn{9}{c}{\textit{Hybrid models}} \\
\midrule
    GDN  & \underline{58.46} & 72.74 & 70.54 & 42.58 & 56.99 & 46.63 & 57.99 & 12.95  \\
    Mamba-3 & 58.14 & \textbf{74.10} & \underline{72.22} & \underline{43.09} & 57.30 & 46.38 & 58.54 & 13.01 \\
    KDA  & 57.93 & 72.74 & 71.04 & 41.21 & 56.99 & \underline{47.22} & 57.85 & 12.60 \\
    GDN-2  & 58.35 & 72.91 & \textbf{72.64} & 42.75 & \underline{58.88} & 47.18 & \underline{58.79} & \underline{11.85} \\
   mLSTM & 56.63 & 72.80 & 69.74 & 42.58 & 57.06 & 45.29 & 57.35 & 14.38 \\
     \midrule
    KronLSTM  & \textbf{59.80} & \underline{73.78} & 71.93 & \textbf{44.11} & \textbf{58.96} & \textbf{48.77} & \textbf{59.56} & \textbf{11.18} \\
\bottomrule
\end{tabular}

\label{tab:mcqa-1p5b}
\end{table}

\subsection{Retrieval Capabilities}
\label{subsec:retrieval}

On long-context tasks, the Kronecker Delta Rule should improve the most over previous update rules, as interference grows with context length.
We assess the performance of KronLSTM on two benchmarks with different context lengths.

The first benchmark is the real-world recall suite of \citet{arora2024simple} (SWDE, SQuAD, FDA,
TriviaQA, NQ, DROP) for 2k context.
Results are given in Table~\ref{tab:recall-1p5b}.
KronLSTM has the best average accuracy, and it
is the best method in three out of six tasks
in both pure linear attention models and hybrid models. 
In general, hybrid models using SWA gain five to nine accuracy points. 

The second benchmark is the RULER
needle-in-a-haystack (NIAH) tasks \citep{hsieh2024ruler} where 
models are evaluated at 4k, 8k and 16k context length,
with 16k assessing the extrapolation capabilities of the models.
Results are presented in Table~\ref{tab:ruler-k128}.
Among all linear attention models, 
KronLSTM is the best or second-best method across all sequence lengths,
except for task S1.
KronLSTM has the highest accuracy on S3 and MQ tasks
across all sequence lengths.
Among all hybrid models, KronLSTM achieves the highest accuracy 
(averaged over all lengths) on all tasks 
and has the highest accuracy at 16k on each task, confirming its strong long context capabilities.
The improvement of KronLSTM over competitors like Gated DeltaNet
is consistent with 
the interference analysis of Property~\ref{prop:Ainterference}. 
On S2 and S3, whose needles are unique sentences embedded in 
natural language text, KronLSTM outperforms GDN except for S2 at 16k.
For hybrid models, KronLSTM outranks GDN-2 and Mamba-3 by 15.2 accuracy points on S2 at 4k. KronLSTM's margin increases on MQ and MV with sequence length.
On MQ, KronLSTM's lead distance grows from 2.4 to 4.1 accuracy points when 
increasing the context from 4k to 16k.
On MV, KronLSTM is 3.5 accuracy points behind the best method, GDN,
for 4k context, while KronLSTM is the leading method for 
8k and 16k, with a 1.1 and 3.8 accuracy point lead, respectively.
\begin{table}[H]
\centering
\small
\setlength{\tabcolsep}{9pt}
\caption{Results for 1.5B models on the real-world recall suite, given in accuracies for 
2k context length. Best per column in \textbf{bold}, second best \underline{underlined}.
}
\begin{tabular}{lcccccc|c}
\toprule
Model & SWDE & SQD & FDA & TQA & NQ & DROP & Avg-6 $\uparrow$ \\
\midrule
\multicolumn{8}{c}{\textit{Linear models}} \\
\midrule
GDN                        & 28.62 & \underline{37.33} & \underline{14.79} & 60.49 & \textbf{16.47} & \textbf{21.18} & \underline{29.81} \\
Mamba-3               & 27.54 & 35.86 & 11.43 & 59.18 & 15.39 & \underline{21.03} & 28.41 \\
KDA                        & 25.83 & 37.30 & 12.98 & \textbf{62.14} & \underline{16.31} & 20.46 & 29.17 \\
GDN-2                      & \underline{29.70} & 36.29 & 14.43 & 60.07 & 15.90 & 19.93 & 29.39 \\
mLSTM & 27.54 & 36.96 & 11.98 & 59.77 & 15.11 & 21.37 & 28.79 \\
\midrule
KronLSTM  & \textbf{30.87} & \textbf{39.28} & \textbf{15.34} & \underline{61.61} & 15.46 & 19.31 & \textbf{30.31} \\
\midrule
\multicolumn{8}{c}{\textit{Hybrid models}} \\
\midrule
GDN          & 43.20 & 40.58 & 23.59 & 61.91 & 17.90 & \textbf{24.44} & 35.27 \\
Mamba-3 & 42.84 & \textbf{44.67} & \underline{29.22} & \underline{63.92} & 17.26 & \underline{23.86} & \underline{36.96} \\
KDA          & 42.30 & 42.79 & \textbf{30.04} & 61.02 & \underline{18.12} & 21.80 & 36.01 \\
GDN-2        & 43.38 & 42.19 & 28.31 & 60.84 & 17.36 & 23.05 & 35.85 \\
mLSTM & \underline{46.08} & 43.47 & 27.77 & 59.66 &  17.29 & 23.96 & 36.37 \\
\midrule
KronLSTM  & \textbf{46.89} & \underline{43.57} & 27.86 & \textbf{64.51} & \textbf{18.82} & 22.76 & \textbf{37.40} \\
\bottomrule
\end{tabular}
\label{tab:recall-1p5b}
\end{table}

For longer contexts, the memory updates of GDN more frequently interfere with 
already stored content and perturb it, resulting in more degraded performance at 16k.
In contrast, the Kronecker Delta Rule is less invasive and performs better
at longer contexts.
The exception is S1, where KronLSTM underperforms.
S1 is the most artificial task with frequently repeated noise sentences,
which is unrealistic for real-world text.
We attribute the performance drop of the KronLSTM in this task
to the residual accumulation for recurring keys.
Hybrid models with SWA have a considerably smaller S1 accuracy, except KronLSTM.
The hybrid models containing the baseline linear attention layers reach at most 51.8, 26.6, and 12.6 accuracy for 4k, 8k, and 16k
context length, respectively. 
These accuracies closely match the fraction of the context covered by the attention window, indicating that the needle is primarily retrieved when it lies within the window, with Mamba-3 and mLSTM dropping to zero at 16k.
The results suggest limited reliance on the recurrent architecture’s memory. KronLSTM exceeds these accuracies on average across all context lengths and achieves the best performance across all six tasks at 16k.
Section~\ref{subsec:long-context} examines how these methods perform
after context extension.

\subsection{Long Context Extension}
\label{subsec:long-context}

In applications, language models do not extrapolate to longer context but
are trained on them, typically by continued pretraining at increased context length
\citep{fu2024data,gao2024prolong}.
This stage unlocks retrieval far beyond the original pretraining length \citep{fu2024data}.
Linear attention models have no positional encoding to adapt; rather, they must learn to keep more associations over longer time spans. 
We therefore branch the 1.5B linear KronLSTM and GDN-2 runs at 75B tokens and train them for the
remaining 25B tokens at a sequence length of 32k on a long-document mixture
derived from ProLong-64K \citep{gao2024prolong}, so that the context-length-extended
models are token-matched to the 100B models of the previous sections.
GDN-2 serves as a reference model in this experiment since it also decouples erase and write.
Further details on training are described in Appendix~\ref{sec:Alongctx}. Table~\ref{tab:lc32k} shows the commonsense reasoning, LAMBADA, and recall results of the extended models.
Both models benefit from length extension on the reasoning and recall tasks, 
and perplexity improves by about 4 perplexity points.
KronLSTM keeps its lead: it has the higher average accuracy and the lower perplexity,
and it is better on five of the six recall tasks.

\begin{table}[H]
\centering
\caption{Results of 1.5B models on the NIAH tasks for 8k pretraining context and 4k, 8k, and 16k evaluation context. S1--S3: S-NIAH-1/2/3; MK: MK-NIAH-1; MQ: MQ-NIAH; MV: MV-NIAH.
Best per task and length in \textbf{bold}, second best \underline{underlined}, separately for linear and hybrid models. }
\scalebox{0.91}{
\begin{tabular}{p{1.65cm}p{0.3cm}cccccc | cccccc}
\toprule
& & \multicolumn{6}{c}{\textit{Linear models}} & \multicolumn{6}{c}{\textit{Hybrid models}} \\
\cmidrule(lr){3-8} \cmidrule(lr){9-14}
Model & Len. & S1 & S2 & S3 & MK & MQ & MV & S1 & S2 & S3 & MK & MQ & MV \\
\midrule
\multirow{3}{*}{GDN}
 & 4k  & \textbf{100.0} & \underline{96.0} & \underline{54.8} & \underline{28.0} & \underline{44.8} & \textbf{46.8} & \underline{51.8} & 52.6 & 21.8 & 31.0 & \textbf{40.6} & 38.3 \\
 & 8k  & \textbf{100.0} & 45.6 & 13.0 & \underline{24.0} & \underline{21.2} & \underline{23.5} & \underline{26.6} & \underline{31.2} & \textbf{22.2} & 24.6 & \underline{17.8} & 18.8 \\
 & 16k & \underline{99.2} & 17.8 & \underline{5.4} & \textbf{17.8} & \underline{8.6} & \underline{13.8} & \underline{12.6} & \underline{16.8} & 3.6 & \underline{18.4} & \underline{6.9} & 8.8 \\
\addlinespace
\multirow{3}{*}{Mamba-3}
 & 4k  & 92.0 & 68.2 & 26.4 & 26.8 & 20.0 & 18.4 & 51.6 & \underline{56.0} & 30.2 & 34.4 & 19.1 & 21.3 \\
 & 8k  & 59.0 & 27.0 & 5.4 & 20.6 & 4.2 & 3.6 & \underline{26.6} & \underline{31.2} & 15.0 & 24.2 & 10.7 & 22.4 \\
 & 16k & 28.0 & 13.6 & 3.0 & 13.4 & 5.3 & 4.5 & 0.0 & 0.0 & 0.0 & 0.0 & 0.0 & 0.0 \\
\addlinespace
\multirow{3}{*}{KDA}
 & 4k  & \underline{99.8} & 89.4 & 12.0 & 20.8 & 18.4 & 19.1 & \underline{51.8} & 55.8 & 31.4 & 44.8 & 33.0 & \underline{42.5} \\
 & 8k  & 99.0 & 35.4 & 11.6 & 15.8 & 13.3 & 13.3 & \underline{26.6} & \underline{31.2} & 5.0 & 22.8 & 16.3 & \underline{22.8} \\
 & 16k & 93.6 & 14.2 & 1.6 & 11.8 & 3.1 & 5.7 & \underline{12.6} & \underline{16.8} & 4.4 & 17.2 & 3.1 & 11.7 \\
\addlinespace
\multirow{3}{*}{GDN-2}
 & 4k  & \textbf{100.0} & 72.8 & 50.2 & 26.6 & 28.4 & 38.1 & \underline{51.8} & \underline{56.0} & \underline{45.6} & \textbf{51.8} & 33.2 & 40.5 \\
 & 8k  & \underline{99.8} & 44.6 & \underline{15.4} & 17.6 & 13.4 & 21.4 & \underline{26.6} & \underline{31.2} & \underline{18.8} & \underline{25.4} & \textbf{23.2} & 22.7 \\
 & 16k & \textbf{99.6} & 8.2 & 4.2 & 7.2 & 2.7 & 5.1 & 7.4 & 16.6 & \underline{6.0} & 17.6 & 3.0 & \underline{12.8} \\
\addlinespace
    \multirow{3}{*}{mLSTM}
   & 4k  & 92.4 & 89.2 & 1.4 & 27.0 & 22.5 & 19.7 & \underline{51.8} & \underline{56.0} & 30.6 & 49.0 & 21.9 & 38.0 \\
   & 8k  & 41.6 & \underline{49.4} & 2.4 & \textbf{26.2} & 5.7 & 12.4 & \underline{26.6} & \underline{31.2} & 16.8 & 26.8 & 14.9 & \textbf{24.6} \\
   & 16k & 17.8 & \textbf{22.6} & 2.0 & 15.4 & 5.0 & 6.9 & \underline{12.6} & 0.6 & 0.0 & 0.0 & 0.4 & 0.8 \\
  \addlinespace
\midrule
\multirow{3}{*}{KronLSTM}
 & 4k  & 66.2 & \textbf{98.2} & \textbf{73.4} & \textbf{28.2} & \textbf{47.2} & \underline{43.3} & \textbf{64.8} & \textbf{71.2} & \textbf{51.2} & \underline{51.0} & \underline{38.8} & \textbf{44.5} \\
 & 8k  & 29.4 & \textbf{60.4} & \textbf{20.4} & 23.4 & \textbf{24.8} & \textbf{24.6} & \textbf{33.6} & \textbf{35.0} & 17.4 & \textbf{28.6} & \underline{23.1} & 21.8 \\
 & 16k & 11.6 & 16.0 & \textbf{11.8} & \underline{16.0} & \textbf{12.7} & \textbf{17.6} & \textbf{16.2} & \textbf{17.4} & \textbf{7.0} & \textbf{21.0} & \textbf{12.2} & \textbf{13.2} \\
\bottomrule
\end{tabular}}

\label{tab:ruler-k128}
\end{table}

\begin{table}[H]
  \centering
  \small
  \setlength{\tabcolsep}{5pt}
 \caption{Long-context (32k) extension results for GDN-2 and KronLSTM models. Avg-6 is the mean accuracy of the five reasoning tasks plus LAMBADA acc; the remaining columns give accuracy on the real-world recall suite evaluated at 2048 tokens. The best per column is in \textbf{bold}.}
  \begin{tabular}{lcc|cccccc|c}
  \toprule
  Model & Avg-6 $\uparrow$ & LMB ppl $\downarrow$ & SWDE & SQD & FDA & TQA & NQ & DROP & Recall avg $\uparrow$ \\
  \midrule
    \multicolumn{10}{c}{\textit{Linear models}} \\
    \midrule
      GDN-2 & 58.75 & 8.36 & 29.61 & 36.39 & \textbf{19.60} & 60.19 & 14.54 & 19.36 & 29.95 \\
      KronLSTM & \textbf{60.19} & \textbf{7.66} & \textbf{31.68} & \textbf{39.88} & 18.15 & \textbf{60.37} & \textbf{15.36} & \textbf{21.03} & \textbf{31.08} \\
        \midrule
    \multicolumn{10}{c}{\textit{Hybrid models}} \\
    \midrule
        GDN-2 & 59.32 & 8.27 & 40.41 & 40.95 & \textbf{30.40} & 62.50 & 15.24 & 21.75 & 35.21 \\
        KronLSTM & \textbf{60.10} & \textbf{7.64} & \textbf{46.08} & \textbf{44.10} & 23.96 & \textbf{63.39} & \textbf{16.95} & \textbf{22.90} & \textbf{36.23} \\
  \bottomrule
  \end{tabular}
  \label{tab:lc32k}
  \end{table}

Results for context extension at the RULER tasks are given in Figure~\ref{fig:lc-32k} and Table~\ref{tab:ruler-lc32k}.
As expected, KronLSTM shows its advantage by extending the training context length. While both architectures perform strongly on single-needle retrieval at shorter contexts, KronLSTM increasingly outperforms GDN-2 on the more challenging NIAH tasks, ranging from 8k to 32k. At 16k, this advantage is particularly pronounced on S-3 and MV, with KronLSTM leading on most tasks except S1, where GDN-2 remains stronger. At 32k, KronLSTM maintains its lead across all five tasks, although both architectures experience a substantial drop in retrieval accuracy compared to 16k.
The gap between the two models widens with haystack size, as predicted by the interference analysis in Section~\ref{subsec:update-rule}: longer contexts accumulate more interference, and KronLSTM's update rule suffers less of it.

\begin{figure}[ht]
  \centering
  \includegraphics[width=0.8\linewidth]{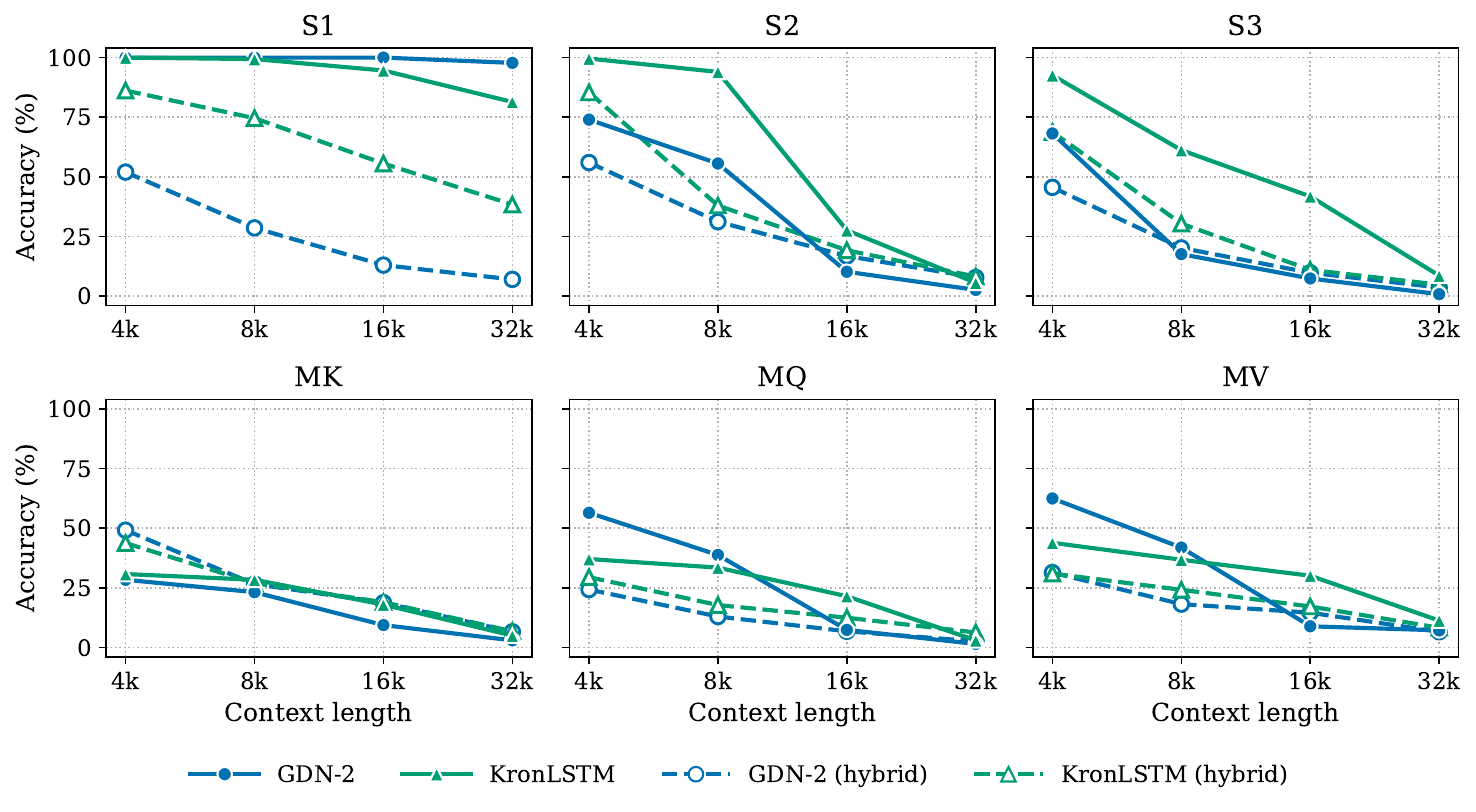}
   \caption{GDN-2 and KronLSTM results on NIAH for 32k context extension for linear and SWA-hybrid models.}
  \label{fig:lc-32k}
\end{figure} 

\subsection{State Tracking Capabilities}
\label{sec:state-tracking}

KronLSTM can perform state tracking by setting $\beta_t = 2\sigma(\Bw_{\beta}^\top \Bx_t + b_\beta) \in (0,2)$, as described in \citet{grazzi25unlocking} for DeltaNet.
We evaluate KronLSTM's state tracking capabilities on
parity, $S_3$ permutation, and modulo arithmetic.
We train up to sequence length 128 and evaluate up to 1024. 
Figure \ref{fig:state-tracking} shows that KronLSTM has similar state tracking capabilities as Gated DeltaNet on the parity and $S_3$ tasks, both ahead of Mamba-3, which wins on modulo arithmetic.

\begin{figure}[h]
  \centering
  \includegraphics[width=0.7\linewidth]{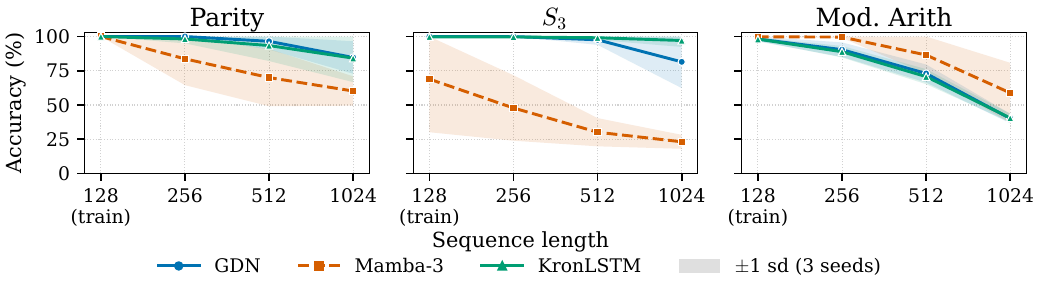}
   \caption{State tracking results for Gated DeltaNet (GDN), Mamba-3, and KronLSTM.}
   \label{fig:state-tracking}
\end{figure} 

\subsection{Training Efficiency}

\makeatletter
\providecommand\captionof[1]{\def\@captype{#1}\caption}
\makeatother

\noindent
\begin{minipage}[t]{0.52\linewidth}
\vspace{0pt}
 Figure~\ref{fig:throughput} reports end-to-end training throughput of the purely linear 1.5B models on a single NVIDIA H100 (80\,GB), following \citet{hatamizadeh2026gdn2}: each step processes 16,384 tokens, split from 2K$\times$8 to 16K$\times$1 (sequence length $\times$ batch size), and all models share the same training stack. KronLSTM reaches 36.8--44.5K tokens/s, 11--18\% more than GDN-2 at every shape and within 4--8\% of KDA up to 8K$\times$2. GDN, whose scalar decay admits the cheapest kernels, remains the fastest. Throughput is nearly flat while a step holds at least two sequences; at 16K$\times$1 all models slow down, since the inter-chunk recurrence parallelizes only over sequences and heads.
\end{minipage}\hfill
\begin{minipage}[t]{0.45\linewidth}
\vspace{0pt}
\centering
\includegraphics[width=\linewidth]{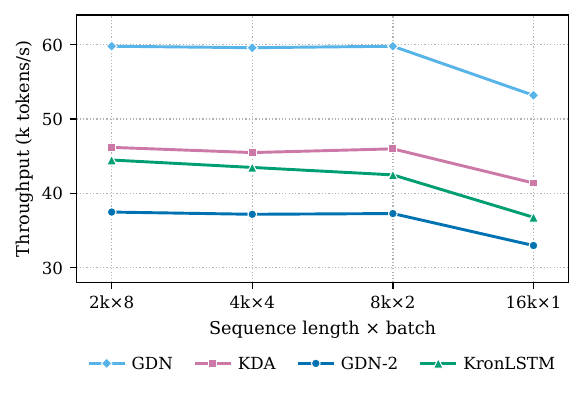}
\setlength{\abovecaptionskip}{2pt}
\setlength{\belowcaptionskip}{0pt}
\captionof{figure}{Training throughput of KronLSTM against baselines on a single H100-80GB.}
\label{fig:throughput}
\end{minipage}

\section{Conclusion}

We have introduced the Kronecker Delta Rule, an update rule for associative matrix memories that corrects a scalar association along a single direction, and KronLSTM, which embeds this rule in a gated linear attention mixer.
We proved that the memory dynamics converge to a unique steady state and that the expected memory norm is bounded at every step.
KronLSTM has the best average accuracy on commonsense reasoning and real-world recall, both as a linear model and as a hybrid.
Averaged over the context lengths, it is the best or second-best linear model on every natural-text RULER task, 
and on S1 it is the only one that retrieves needles beyond the sliding attention window in hybrid models. After extending the context to 32k, KronLSTM outperforms GDN-2 on all natural-text retrieval tasks from 16k onward.
Future work includes scaling up KronLSTM and training hybrid models containing full softmax attention layers.

\subsection*{AI use statement}

The Kronecker Delta Rule and the KronLSTM architecture originated with the authors without AI assistance. LLMs were used to assist in developing the theoretical results and in implementing the method. Claude Fable 5 and Claude Fable 5.1 assisted in deriving the properties of the Kronecker Delta Rule and the convergence and memory norm bounds of Section~\ref{subsec:update-rule} and Appendix~\ref{sec:Akdr}; every statement and proof was checked independently by the authors, who take full responsibility for their correctness. Claude Fable 5 was further used to examine candidate formulations of the Kronecker Delta Rule, in particular the coupled, decoupled, and unit-target write variants, for their viability; the choice of the form in \eqref{eq:gated-rule} based on was made by the authors on the basis of ablation studies such as in Appendix~\ref{sec:Afurther-results}. Fable 5.1 was used to assist in deriving the chunkwise-parallel formulation of KronLSTM. The formulation was subsequently checked thoroughly by our team. Claude Fable 5 and Claude Fable 5.1 were used to optimize the chunkwise-parallel and recurrent kernels of KronLSTM; the resulting kernels were numerically verified against naive reference implementations written by the authors. We did not use generative AI tools to interpret results or to design experiments. We have not used generative AI tools to generate synthetic data or to propose hypotheses, and translation, dataset cleaning and reformatting, and qualitative or thematic data analysis do not apply to this work. Additionally, we used Claude Fable 5 and Claude Opus 5 to draft parts of the manuscript and to edit it for readability. All AI-drafted text was revised by the authors. We have reviewed all AI-assisted work and take responsibility for the final content of this work, including text, claims, and artifacts produced with the aid of generative AI.

\subsubsection*{Acknowledgments}

This work was supported by the Austrian Science Fund (FWF) 10.55776/COE12 and the European Union’s Horizon Europe research and innovation program under grant agreement number 101214398 (ELLIOT).
The ELLIS Unit Linz, the LIT AI Lab, and the Institute for Machine Learning are supported by the Federal State of Upper Austria.
We acknowledge the EuroHPC Joint Undertaking for awarding us access to Leonardo at CINECA, Italy, to Marenostrum 5 at the Barcelona Supercomputing Centre (BSC), Spain, and to JUPITER at J\"ulich Supercomputing Centre (JSC), Germany.
We would like to thank Stephanie Holly for checking the mathematical proofs of \ref{checked} and B\'{a}lint L\'{a}szl\'{o} Szarvas for proofreading and making suggestions to this manuscript.

\bibliography{references}
\bibliographystyle{plainnat}

\newpage 

\newlength{\auxparskip}
\setlength{\auxparskip}{\parskip} 
\setlength{\parskip}{0pt}
\tableofcontents

\newpage 

\begin{appendices}

\renewcommand{\thesection}{A\arabic{section}}
\renewcommand{\thefigure}{A\arabic{figure}}
\renewcommand{\thetable}{A\arabic{table}}
\renewcommand{\theequation}{A\arabic{equation}}

\setcounter{section}{0}
\setcounter{figure}{0}
\setcounter{table}{0}
\setcounter{equation}{0}

\setcounter{theoremA}{0}
\setcounter{definitionA}{0}
\setcounter{corollaryA}{0}
\setcounter{propositionA}{0}
\setcounter{lemmaA}{0}
\setcounter{propertyA}{0}

\section{Notation Table}
\label{sec:Anotation}

\begin{table}[H]
\centering
\setlength{\tabcolsep}{25pt}
\scriptsize{%
\begin{tabular}{l c l}
\toprule
\textbf{Definition} & \textbf{Symbol} & \textbf{Type} \\
\midrule
\multicolumn{3}{l}{\textbf{Scalars and dimensions}} \\
time step (token index) & $t$ & $\dN$ \\
key and query dimension per head & $d_k$ & $\dN$ \\
value dimension per head & $d_v$ & $\dN$ \\
layer input dimension  & $d$ & $\dN$ \\
number of stored associations (App.) & $n$ & $\dN$ \\
step size / erase gate & $\beta_t$ & $(0,1)$ \\
fixed step size (theory) & $\beta$ & $\dR_{>0}$ \\
target / write gate & $c_t$
  & $(0,1)$ in Sec.~\ref{subsec:kronlstm-block}, $\dR$ in Sec.~\ref{subsec:update-rule} \\
decoupled target & $\tilde{c}_t = c_t / \beta_t$ & $\dR_{>0}$ (unbounded) \\
normalized step size & $\kappa = \beta \max_i \norm{\Ba_i}^2$ & $(0,2)$ \\
irreducible residual & $\rho = c - \Ba^\top \Bs^{\rL\rS}$ & $\dR$ \\
contraction factor of the mean & $r$ & $[0,1)$ \\
second-moment contraction factor
  & $q_\beta = 1 - \beta(2-\kappa)\lambda_{\min}(\bar{\BA})$ & $[0,1)$ \\
extreme eigenvalues of $\bar{\BA}$ & $\lambda_{\min}, \lambda_{\max}$ & $\dR_{>0}$ \\
forget-gate exponent 
  & $\alpha = \exp(a_{\log})$ & $\dR_{>0}$ \\
scalar read-out 
  & $h_t = \Bv_t^\top \BS_t \Bk_t$ & $\dR$ \\
residual 
& $\psi_t = y_t - c_t$ & $\dR$ \\
relaxation 
  & $\omega = \beta \norm{\Bk_t}^2 \norm{\Bv_t}^2$ & $(0,2)$ \\
driving-noise energy 
  & $\sigma^2 = \EXP\big[\rho^2 \norm{\Ba}^2\big]$ & $\dR_{\ge 0}$ \\
\midrule
\multicolumn{3}{l}{\textbf{Vectors}} \\
layer input & $\Bx_t$ & $\dR^{d}$ \\
query, key & $\Bq_t, \Bk_t$ & $\dR^{d_k}$ \\
value & $\Bv_t$ & $\dR^{d_v}$ \\
Kronecker feature
  & $\Ba_t = \Bk_t \otimes \Bv_t = \vecop(\Bv_t \Bk_t^\top)$ & $\dR^{d_k d_v}$ \\
vectorized memory & $\Bs_t = \vecop(\BS_t)$ & $\dR^{d_k d_v}$ \\
least-squares memory (vectorized)
  & $\Bs^{\rL\rS} = \bar{\BA}^{-1} \bar{\Bb}$ & $\dR^{d_k d_v}$ \\
cross moment & $\bar{\Bb} = \EXP[c\, \Ba]$ & $\dR^{d_k d_v}$ \\
forget-gate vector & $\bff_t$ & $(0,1)^{d_k}$ \\
forget-gate pre-activation & $(\Bs_{\Rf})_t$ & $\dR^{d_k}$ \\
output gate & $\bfo_t$ & $(0,1)^{d_v}$ \\
memory read-out & $\Bh_t = \BS_t \Bq_t$ & $\dR^{d_v}$ \\
block output & $\By_t$ & $\dR^{d}$ \\
\midrule
\multicolumn{3}{l}{\textbf{Matrices}} \\
memory matrix & $\BS_t$ & $\dR^{d_v \times d_k}$ \\
forget gate & $\BF_t = \diag(\bff_t)$ & $\dR^{d_k \times d_k}$ \\
feature second moment & $\bar{\BA} = \EXP[\Ba \Ba^\top]$
  & $\dR^{d_k d_v \times d_k d_v}$, $\succ \BZe$ \\
least-squares memory & $\BS^{\rL\rS}$ & $\dR^{d_v \times d_k}$ \\
stationary covariance (App.) & $\BV_\beta$ & $\dR^{d_k d_v \times d_k d_v}$ \\
transition factor (App.) & $\BM_i = \BI - \beta \Ba_i \Ba_i^\top$
  & $\dR^{d_k d_v \times d_k d_v}$ \\
\midrule
\multicolumn{3}{l}{\textbf{Parameters}} \\
erase-gate weights and bias & $\Bw_\beta, b_\beta$ & $\dR^{d}, \dR$ \\
write-gate weights and bias & $\Bw_{\Rc}, b_{\Rc}$ & $\dR^{d}, \dR$ \\
forget-gate weights and bias & $\BW_{\Rf}, \Bb_{\Rf}$
  & $\dR^{d_k \times d}, \dR^{d_k}$ \\
output-gate weights and bias & $\BW_{\Ro}, \Bb_{\Ro}$
  & $\dR^{d_v \times d}, \dR^{d_v}$ \\
output projection & $\BW_{\text{out}}$ & $\dR^{d \times H d_v}$ ($H$ undefined) \\
log forget exponent & $a_{\log}$ & $\dR$ (per head) \\
\midrule
\multicolumn{3}{l}{\textbf{Distributions and operators}} \\
distribution of triples $(\Bk, \Bv, c)$ & $p$ & distribution \\
stationary distribution & $\pi_\beta$ & law on $\dR^{d_k d_v}$ \\
logistic sigmoid & $\sigma(\cdot)$ & $\dR \to (0,1)$ \\
Kronecker product, vectorization & $\otimes$, $\vecop$ & \\
element-wise product & $\odot$ & \\
\bottomrule
\end{tabular}
}
\caption{Notation used throughout the paper and the appendix.}
\label{tab:notation}
\end{table}

\newpage

\section{Related Work}
\label{sec:Arelated-work}

To the best of our knowledge, no existing work uses {\em exactly} this update that is a scalar
bilinear read-out $\Bv^\top \BS \Bk$ regressed onto the constant target $1$. The established
delta-rule linear attention family uses the vector read-out $\BS \Bk \to \Bv$. The rule studied
here is therefore the {\em bilinear-readout, scalar-target} member of the
delta-rule. The constraints are
$\abra{\Bv \Bk^\top, \BS}_{\rF} = 1$ rather than $\BS \Bk = \Bv$. 

\subsection{Long Short-Term Memory}

KronLSTM belongs to the xLSTM \citep{beck24xlstm}
family that revisits LSTM \citep{Hochreiter:91,hochreiter97lstm}
for the era of large language models (LLMs) with two modifications:
exponential gating and a redesigned memory.
mLSTM replaces the scalar cell state by a matrix memory
and drops memory mixing, making the recurrence fully
parallelizable over the sequence.
mLSTM enables recurrent LLMs with fast, constant-memory inference
at scale \citep{Beck:25tfla}, and
\citet{Beck:25scaling} establish scaling behavior competitive with
Transformers at linear time complexity. 
xLSTM has also been applied
to language modeling \citep{Beck:25xlstm7b},
vision \citep{alkin2025vision}, biological and chemical sequences
\citep{Schmidinger:25}, and robotics \citep{schmied2025large}.

The matrix memory of mLSTM can be viewed as a classical
correlation-based associative memory \citep{Kohonen:72,Anderson:72}
embedded within LSTM-style gating.
In addition to providing an efficient mechanism for storing and
retrieving associations, the exponential gates allow mLSTM to
revise storage decisions. This distinguishes mLSTM from purely
accumulative linear attention mechanisms.
KronLSTM follows this line of research by retaining the matrix memory
and the ability to revise storage decisions, while introducing a
different memory-update rule.

A central challenge of recurrent architectures with growing memory is
maintaining stable memory dynamics over long sequences.
The original LSTM uses a bounded squashing function for the exposed
cell state, while xLSTM introduces a normalizer state and a stabilizer
state to control the scale of its matrix memory and retrievals.
However, normalization based on ratios of large quantities can lead
to numerical and optimization difficulties, particularly in large
networks \citep{Beck:25xlstm7b}. MC-LSTM \citep{hoedt2021mc} instead redistributes its cell state with
column-stochastic matrices, which conserve mass and never amplify the state.
These considerations motivate alternative matrix-memory update rules
that retain memory revision while providing stable memory dynamics.
The Kronecker Delta Rule addresses this issue and
provides convergence to a steady-state distribution with a
norm-bounded memory in expectation. Furthermore, KronLSTM yields advantageous
gradient flow over time compared with mLSTM
\citep{Hochreiter:00book}.

\subsection{The delta-rule linear attention family}

The direct ancestor is the fast weight programmer perspective of
\citet{Schlag:21}, who introduced the delta-rule write
$\BS_{t+1} = \BS_t - \beta\,(\BS_t \Bk_t - \Bv_t)\,\Bk_t^\top$ into linear attention
\citep{katharopoulos20transformers}, building on the fast-weight programmers of the
early 1990s \citep{Schmidhuber:92}. \citet{Yang:24} (DeltaNet)
made this rule hardware-efficient by exploiting the WY representation of the resulting
products of generalized Householder matrices. This is exactly the product structure the bilinear
rule inherits, 
with the Kronecker product $\Bk_t \otimes \Bv_t$ taking the role of $\Bk_t$. Subsequent extensions
include Gated DeltaNet \citep{yang2025gated}, which adds a scalar forget gate
$\alpha_t \BS_t$; RWKV-7 \citep{Peng:25}, which employs a generalized delta rule with
a decoupled in-context learning rate; and DeltaProduct \citep{Siems:25arXiv},
which applies multiple Householder factors per token to improve state tracking. This is directly
relevant to the negative-eigenvalue regime $\beta \norm{\Ba_t}^2 \in (1,2)$ of the
transition factors $\BI - \beta \Ba_t \Ba_t^\top$, which allows state-tracking.

\subsection{The online-learning perspective}

Longhorn \citep{Liu:25} derives state-space recurrences explicitly as
closed-form solutions of per-token online learning objectives; the bilinear rule fits
this template with the objective $\tfrac12\,(\Bv^\top \BS \Bk - 1)^2$. Test-time training (TTT)
layers \citep{Sun:24} and Titans \citep{Behrouz:25titans} treat the memory as a
model trained by gradient descent at inference time. Most relevant as a unifying
framework is the Miras taxonomy of \citet{Behrouz:25miras}, which classifies memory
architectures by their attentional-bias objective, retention mechanism, and optimizer.
The new bilinear rule occupies a new place of that taxonomy: it is a {\em bilinear} attentional
bias with a constant scalar target. A recent related observation is that test-time training
with key--value binding is secretly linear attention \citep{Liu:26tttarXiv}. The
mesa-optimization line \citep{Oswald:23} provides the complementary
least-squares-solver perspective on such layers.

\subsection{Classical update and learning rules}

When vectorized, the rule is the Widrow--Hoff LMS\,/\,delta rule \citep{Widrow:60}
applied to the Kronecker product input  $\Ba_t = \Bk_t \otimes \Bv_t$ with target $1$. 
The stability range
$\beta \norm{\Ba_t}^2 \in (0,2)$ is the classical LMS step-size condition. The
projection interpretation is Kaczmarz's method \citep{Kaczmarz:37}: with the Newton step
$\beta^\star = 1/(\norm{\Bk}^2 \, \norm{\Bv}^2)$, the update is the orthogonal
projection onto the hyperplane $\{\BS : \abra{\Bv \Bk^\top, \BS}_{\rF} = 1\}$, and general
$\beta$ gives relaxed projections. The modern convergence theory, including the linear
rate under random sampling of constraints, is due to \citet{Strohmer:09}. The
connections between Kaczmarz iterations, stochastic gradient descent, and learning are
surveyed by \citet{Chen:18}. Fej\'er monotonicity and relaxed projection
methods are standard in the projection-methods literature \citep{Cegielski:12}.

\subsection{Relation to cognitive and neuroscience}

Because $\BS$ scores key--value pairs rather than retrieving values, the closest relative
in the cognitive and neuroscience literature is familiarity (recognition) memory: the
anti-Hebbian model of familiarity discrimination in perirhinal cortex of
\citet{Bogacz:01} shows that scoring memories attain a far higher capacity than
retrieval memories, which is the same $d^2$-versus-$d$ capacity separation derived for the
bilinear rule. The classical retrieval-side background is given by correlation matrix
memories \citep{Kohonen:72,Anderson:72} for the Hebbian write, and by
delta-rule\,/\,pseudoinverse associative memories for capacity under Widrow--Hoff
learning \citep{Prados:89}.

\section{Kronecker Delta Rule: Setting, Name, Properties}
\label{sec:Akdr}

\subsection{Setup and notation.}
We use the notation of the main paper (Section~\ref{subsec:update-rule}), specialized
to $d_k = d_v =: d$ as in our experiments (where $d_{qk}=d_v$); the general
$d_v \times d_k$ case is identical with $d^2$ replaced by $d_k d_v$.
Let $\BS_t \in \dR^{d\times d}$ be the memory matrix, $\Bk_t \in \dR^{d}$ the key and
$\Bv_t \in \dR^{d}$ the value at step $t$, $c_t \in \dR$ a scalar target, and $\beta_t$
a step size. The scalar $\Bv_t^\top \BS_t \Bk_t$ measures how strongly the memory
associates $\Bk_t$ with $\Bv_t$; the Kronecker Delta Rule corrects it towards the
target $c_t$ along the outer product $\Bv_t \Bk_t^\top$,
\begin{equation} \label{eq:ATimeRule}
\boxed{\; \BS_{t+1} \;=\; \BS_t \;-\; \beta_t \big(\Bv_t^\top \BS_t \Bk_t \, - \, c_t\big)\,
  \Bv_t \Bk_t^\top \, , \;}
\end{equation}
which is Eq.~\eqref{eq:KDR} of the main paper. For the analysis, we abbreviate the
scalar read-out and its residual,
\begin{align}
y_t \; &:= \; \Bv_t^\top \BS_t \Bk_t && \text{(scalar read-out)}, \\
\psi_t \; &:= \; y_t - c_t && \text{(residual)},
\end{align}
so that the update reads compactly $\BS_{t+1} = \BS_t - \beta_t\, r_t\, \Bv_t \Bk_t^\top$.
The target sequence $c_t$ subsumes different variants:
$c_t \equiv 1$ is the basic Kronecker Delta Rule with target $1$,
$c_t \equiv c > 0$ the Constant-Target Kronecker Delta Rule, $c_t = \norm{\Bv_t}^2$ the
Norm-Target Kronecker Delta Rule, and $c_t = \norm{\Bk_t} \norm{\Bv_t}$ the
Double-Norm-Target Kronecker Delta Rule. The target $c_t$ may also be produced by a
gate or a network, so that {\em what} is written becomes itself input-dependent.
Throughout we use $\Bv^\top \BS \Bk = \TR\!\big(\BS\, \Bk \Bv^\top\big) = \abra{\Bv \Bk^\top,\BS}_{\rF}$.

\subsection{Name Kronecker Delta Rule}

We use the Kronecker product
\begin{equation}
\Ba_t \; = \; \Bk_t\otimes\Bv_t
\end{equation}
in a vectorized update rule.
$\Ba_t$ is a vector of length $d^2$, where
the vector $\Bv_t$ is multiplied by each component of
$\Bk_t$ and then stacked to one vector.
However, we have $\Bv_t \Bk_t^\top = \Bv_t\otimes\Bk_t^\top$.
Let $\Bs = \vecop(\BS) \in \dR^{d^2}$ and $\Ba_t = \Bk_t \otimes \Bv_t
\in \dR^{d^2}$,
so that $\Bv_t^\top \BS \Bk_t = \Ba_t^\top \Bs$. The rule becomes
\begin{equation}
\Bs_{t+1} \; = \; \big(\BI - \beta\, \Ba_t \Ba_t^\top\big)\, \Bs_t +
\beta\, c_t \, \Ba_t,
\end{equation}
which is exactly {\em Delta Rule with Kronecker product} input $\Ba_t$
and target $c_t$.

\subsection{Properties}
\begin{propertyA}[Memory Space is Hilbert Space and Kronecker Update Rule is a Projection onto a Hyperplane] 
 \label{prop:AHilbert} 
The memory space $\dR^{d \times d}$ is a Hilbert space
endowed with the Frobenius inner product
\[
\abra{\BX, \BY}_{\rF} \, = \, \TR \left( \BX^\top
  \BY \right) \, ,
\]
where $\BX,\BY \in \dR^{d \times d}$.
Hence the target condition
\[
\Bv_t^\top \BS \Bk_t \, = \, c_t
\]
is a {\em hyperplane} in this Hilbert space.
The update
\[
\BS_{t+1} \, = \, \BS_t - \beta \left( \Bv_t^\top \BS_t \Bk_t - c_t
\right) \Bv_t \Bk_t^\top
\]
is a {\em orthogonal projection} onto it, if
\[
\beta  \, = \, \frac{1}{\norm{\Bk_t}_2^2 \norm{\Bv_t}_2^2} \, .
\]
The update
\[
\BS_{t+1} \, = \, \BS_t - \beta \left( \Bv_t^\top \BS_t \Bk_t - c_t
\right) \Bv_t \Bk_t^\top
\]
is a {\em relaxed orthogonal projection} onto it, with relaxation
parameter
\[
\omega_t \, = \, \beta \norm{\Bk_t}_2^2 \norm{\Bv_t}_2^2 \, .
\]
\end{propertyA}

\begin{propertyA}[Exact online gradient descent: a bilinear delta rule
  with target $c_t$]
 \label{prop:Agradient} 
  The rule \eqref{eq:ATimeRule} is stochastic gradient descent
  with step size $\beta$ on the instantaneous loss
\begin{equation}
\rL_t(\BS) = \tfrac{1}{2}\,\big(\Bv_t^\top \BS \Bk_t - c_t\big)^2,
\qquad
\nabla_S\, \rL_t(\BS) = \big(\Bv_t^\top \BS \Bk_t - c_t\big)\, \Bv_t \Bk_t^\top .
\end{equation}
It differs from DeltaNet, which performs SGD on $\tfrac12 \norm{ \BS
  \Bk_t - \Bv_t}^2$
with update $\BS_{t+1} = \BS_t - \beta (\BS_t \Bk_t - \Bv_t)
\Bk_t^\top$:
there the target is the {\em vector} $\Bv_t$ (imposing $d$ constraints
per pair),
whereas here the target is the {\em scalar} $c_t$ and
the pair $(\Bk_t, \Bv_t)$ jointly forms the input,
imposing a single bilinear constraint $\Bv_t^\top \BS \Bk_t = c_t$ per pair.
\end{propertyA}

\begin{propertyA}[Superposition: the memory is linear in the target sequence]
\label{prop:Atimesuperpos}
For a fixed pair sequence $(\Bk_t, \Bv_t)_{t=1}^{T}$ and step size
$\beta$, the map from the target sequence
$\Bc = (c_1, \dots, c_T)$ to the final memory $\BS_T(\Bc)$ with
$\BS_0 = \BZe$ is {\em linear}:
\begin{equation}
\BS_T(\Bc + \Bc') \; = \; \BS_T(\Bc) \, + \, \BS_T(\Bc'),
\qquad
\BS_T(\alpha\, \Bc) \; = \; \alpha \ \BS_T(\Bc) \, ,
\end{equation}
explicitly, with $\Ba_t = \Bk_t \otimes \Bv_t$,
\begin{equation}
\vecop\big(\BS_T(\Bc)\big) \; = \; \beta \sum_{t=1}^{T} c_t
\Bigg[\prod_{\tau = t+1}^{T} \big(\BI - \beta \Ba_\tau
\Ba_\tau^\top\big)\Bigg] \Ba_t \, .
\end{equation}
Memories written with different target sequences over the same pairs
superpose exactly. The pure-gain homogeneity of the constant-target
rule ($\BS_t^{(c)} = c\,\BS_t^{(1)}$) is the special case $\Bc =
c\,\BOn$. The pair sequence determines {\em where} the memory can
write (the reachable subspace $\spanop\{\Ba_1, \dots, \Ba_T\}$ and the
interference pattern); the target sequence linearly selects {\em what}
is written into it.
\end{propertyA}

\begin{propertyA}[Per-step reparameterization to the unit target]
\label{prop:Atarget}
For every step with $c_t \neq 0$, rule \eqref{eq:ATimeRule} with pair
$(\Bk_t, \Bv_t)$ and step $\beta$ coincides exactly with the
unit-target rule applied to the rescaled value $\tilde \Bv_t = \Bv_t /
c_t$ with the {\em time-varying} effective step $\tilde \beta_t =
\beta\, c_t^2$:
\begin{equation}
\BS - \beta \big(\Bv^\top \BS \Bk - c_t\big) \Bv \Bk^\top
\; = \;
\BS - \tilde\beta_t\,\big(\tilde \Bv^\top \BS \Bk - 1\big)\, \tilde \Bv \Bk^\top .
\end{equation}
In contrast to the constant-target case, the effective step size now
varies over time, so the time-varying rule is {\em not} globally a
reparameterized unit-target rule; steps with $c_t = 0$ have no
unit-target counterpart at all (see
Property~\ref{prop:Atimedistinguished}\,(iii)).
\end{propertyA}

\begin{propertyA}[Distinguished targets]
\label{prop:Atimedistinguished}
Three choices of $c_t$ are structurally special.
\begin{enumerate}
\item (i) {\em Norm target, $c_t = \norm{\Bv_t}^2$.} The residual becomes
  the value-projected retrieval error $r_t = \Bv_t^\top (\BS_t
  \Bk_t - \Bv_t)$; every exact retrieval memory ($\BS \Bk_i =
  \Bv_i$) satisfies all constraints, and in the autoassociative case
  $\Bk_t = \Bv_t$ the identity $\BS = \BI$ is a universal fixed point.
\item[] (ii) {\em Correlation target, $c_t = \Bv_t^\top \Bk_t$.} The identity
  $\BS = \BI$ is a universal fixed point even {\em
    hetero}\-associatively: $\Bv^\top \BI \Bk = \Bv^\top \Bk$ for all
  pairs. The rule then learns deviations of the bilinear map from the
  standard inner product, i.e.\ $\BS$ is naturally parameterized
  around the identity, and the read-out $\Bv^\top \BS \Bk -
  \Bv^\top \Bk$ measures the {\em learned} (in-context) part of the
  similarity.
\item[] (iii) {\em Erasure target, $c_t = 0$.} The update drives $\Bv_t^\top
  \BS \Bk_t \to 0$: with the Newton step $\beta^\star = 1/(\norm{\Bk_t}^2
  \norm{\Bv_t}^2)$ the association is {\em deleted in one shot}, while
  every read-out with $\Bk' \perp \Bk_t$ or $\Bv' \perp \Bv_t$ is
  exactly unchanged. Hence, $c_t=0$ is a targeted unlearning of a single
  key--value binding. Alternating write targets ($c_t \neq 0$) and
  erasure targets ($c_t = 0$) turns the memory into an editable
  scalar store (see also Property~\ref{prop:Atimekaczmarz} on
  overwriting).
\end{enumerate}
\end{propertyA}

\begin{propertyA}[One-step error contraction and stability range for $\beta$]
\label{prop:Atimecontraction}
The target never enters the homogeneous part of the dynamics.
Re-reading with the same pair and the same target after one update,
\begin{align}
&\qquad \; \Bv_t^\top \BS_{t+1} \Bk_t
\; = \; y_t - \beta\, r_t\, \norm{\Bk_t}^2 \, \norm{\Bv_t}^2 \\
&\Longleftrightarrow
\Bv_t^\top \BS_{t+1} \Bk_t - c_t
\; = \; y_t - c_t - \beta\, r_t\, \norm{\Bk_t}^2 \, \norm{\Bv_t}^2 \\
&\Longleftrightarrow
r_{t+1}
\; = \; r_t - \beta\, r_t\, \norm{\Bk_t}^2 \, \norm{\Bv_t}^2 \\
&\Longleftrightarrow
r_{t+1} \; = \; \big(1 - \beta\, \norm{\Bk_t}^2
\norm{\Bv_t}^2\big)\, r_t \; ,
\end{align}
where $r_{t+1}$ is computed for the same $(\Bk_t, \Bv_t, c_t)$.
The residual contracts if and only if
\begin{equation}
0 \; < \; \beta \; < \; \frac{2}{\norm{\Bk_t}^2 \, \norm{\Bv_t}^2},
\end{equation}
and the loss shrinks by
the factor $\big(1 - \beta \norm{\Bk_t}^2 \, \norm{\Bv_t}^2\big)^2$.
Since $\rL_t$ is quadratic along the ray $\BS - \eta\, r_t\, \Bv_t \Bk_t^\top$,,
the choice
\begin{equation}
\beta^\star \; = \; \frac{1}{\norm{\Bk_t}^2 \, \norm{\Bv_t}^2}
\end{equation}
is a Newton step: the constraint $\Bv_t^\top \BS \Bk_t = c_t$ is
interpolated exactly in one update, for every value of $c_t$.
\end{propertyA}

\begin{propertyA}[Delta Rule / Widrow--Hoff LMS in vectorized form]
 \label{prop:Avector}
 Let $\Bs = \vecop(\BS) \in \dR^{d^2}$ and $\Ba_t = \Bk_t \otimes \Bv_t
\in \dR^{d^2}$ the Kronecker product,
so that $\Bv_t^\top \BS \Bk_t = \Ba_t^\top \Bs$ and $\norm{\Ba_t}^2 =
\norm{\Bk_t}^2 \, \norm{\Bv_t}^2$. The rule becomes
\begin{equation} \label{eq:Aaffine}
\Bs_{t+1} \; = \; \big(\BI - \beta\, \Ba_t \Ba_t^\top\big)\, \Bs_t
\; + \; \beta\, c_t\, \Ba_t \, ,
\end{equation}
which is exactly the Widrow--Hoff LMS/delta rule with input $\Ba_t$
and target $c_t$. This is the classical setting of LMS, which regresses a
{\em time-varying} desired signal onto the input.
The transition matrix $\BI - \beta \Ba_t \Ba_t^\top$ does not depend
on $c_t$: it has eigenvalue $1 - \beta \norm{\Ba_t}^2$
on $\spanop(\Ba_t)$ and eigenvalue $1$ on the orthogonal complement
(dimension $d^2 - 1$).
Unrolling the recurrence,
\begin{equation}
\Bs_T \;=\; \prod_{t=1}^{T}\big(\BI - \beta \Ba_t \Ba_t^\top\big)\, \Bs_0
\;+\; \beta \sum_{t=1}^{T} c_t \Bigg[\prod_{\tau = t+1}^{T} \big(\BI
- \beta \Ba_\tau \Ba_\tau^\top\big)\Bigg] \Ba_t \; ,
\end{equation}
gives the same product-of-generalized-Householder structure as DeltaNet.
Hence, the WY-representation / chunkwise-parallel training techniques
apply verbatim with $\Ba_t = \Bk_t \otimes \Bv_t$ in place of $\Bk_t$,
and the targets $c_t$ enter only the driving (inhomogeneous) term.
Expressivity consequence: for $\beta \norm{\Ba_t}^2 \in (1,2)$ the
factor has a negative eigenvalue (reaching $-1$ at
$\beta\norm{\Ba_t}^2 = 2$),
the same mechanism that gives DeltaNet-style models state-tracking
power beyond linear RNNs restricted to nonnegative eigenvalues.
This is independent of the target sequence.
\end{propertyA}

\begin{propertyA}[Homogeneous part of the update in non-expansive]
 \label{prop:Astable}
 For $\beta \norm{\Ba_t}^2 \in [0,2]$ every factor
  $\BI - \beta \Ba_t \Ba_t^\top$ has spectral norm at most $1$, so the
  homogeneous part of
the recurrence never amplifies and $\norm{\BS_t}$ cannot blow
up. In contrast the Hebbian write $\BS_{t+1} = \BS_t + \Bv_t \Bk_t^\top$ of
vanilla linear attention leads to states that grow
without bound. Moreover, the update is {\em surgical}: it modifies
$\BS$ only along the rank-one direction $\Bv_t \Bk_t^\top$. For any
test pair with $\Bk' \perp \Bk_t$ {\bf or} $\Bv' \perp \Bv_t$ the
read-out $\Bv'^\top \BS \Bk'$ is exactly unchanged.
\end{propertyA}

\begin{propertyA}[Doubly gated interference]
\label{prop:Ainterference}
For a previously stored pair $(\Bk_i, \Bv_i)$, an update with $(\Bk_t,
\Bv_t, c_t)$ perturbs the stored read-out by
\begin{equation}
\Delta\big(\Bv_i^\top \BS \Bk_i\big) \; = \; -\,\beta\, r_t\, \big(\Bv_i^\top \Bv_t\big)\big(\Bk_t^\top \Bk_i\big) \, ,
\end{equation}
i.e.\ interference is proportional to the {\em product} of the key
overlap and the value overlap. Orthogonality in either factor suffices
to protect a memory, which is a milder interference profile than the Hebbian
rule (crosstalk from key overlap alone) and than DeltaNet (gated by
key overlap only). The formula is target-independent, since the target
enters only through the residual $r_t$. Consequently, writes whose read-out is
already correct ($r_t \approx 0$) cause no interference at all,
regardless of $c_t$.
\end{propertyA}

\begin{propertyA}[Relaxed Kaczmarz iteration, Fej\'er monotonicity,
  convergence, and overwriting]
\label{prop:Atimekaczmarz}
The constraint $\Bv_t^\top \BS \Bk_t = c_t$ defines the hyperplane
$\cH_t = \{\BS : \abra{\Bv_t \Bk_t^\top,\BS}_{\rF} = c_t\}$ in $\dR^{d\times d}$
with Frobenius geometry. With $\beta = \beta^\star$ the update
\eqref{eq:ATimeRule} is the orthogonal projection onto $\cH_t$;
general $\beta$ gives a relaxed projection. In contrast to a constant
target, consistency of the system $\{\Bv_i^\top \BS \Bk_i = c_i\}_i$
is now a genuine condition: it holds if and only if the target vector
$(c_i)_i$ lies in the range of the matrix stacking the rows
$\Ba_i^\top$. In particular, this always holds, whenever the $\{\Ba_i\}$ are
linearly independent (at most $d^2$ of them). If the system is
consistent, then for every solution $\BS^\star$ and
any $\beta_t \in \big[0,\, 2/\norm{\Ba_t}^2\big]$,
\begin{equation}
\norm{\BS_{t+1} - \BS^\star}_{\rF} \;\le\; \norm{\BS_t - \BS^\star}_{\rF}
\qquad \text{(Fej\'er monotonicity)},
\end{equation}
and cycling or randomly sampling over the pairs converges linearly to
a solution, at a rate governed by the smallest singular value of the
stacked $\Ba_i^\top$. From $\BS_0 = \BZe$ the limit is the {\em
  minimum-Frobenius-norm} solution.
The canonical {\em in}consistent case is a repeated pair with a
changed target, $(\Bk, \Bv, c)$ followed later by $(\Bk, \Bv, c')$
with $c' \neq c$: the two hyperplanes are parallel, and the iteration
resolves the conflict by {\em overwriting}. The contraction of
Property~\ref{prop:Atimecontraction} drives the read-out to the most
recent target, in one shot for $\beta = \beta^\star$. In expectation
over a stationary distribution of triples $(\Bk, \Bv, c)$, the
fixed-point condition is
\begin{equation}
\EXP\big[\big(\Ba^\top \Bs - c\big) \Ba\big] = \BZe \, ,
\end{equation}
the normal equation of the least-squares regression of the target $c$
on the feature $\Ba = \Bk \otimes \Bv$: the memory is an {\em online
  linear regressor} that learns, in context, to predict the scalar
signal $c_t$ from the key--value pair.
\end{propertyA}

\begin{propertyA}[Capacity: memory can store $d^2$ associations]
\label{prop:Atimecapacity}
Each triple $(\Bk_i, \Bv_i, c_i)$ contributes one linear constraint on
the $d^2$ parameters of $\BS$, so up to $d^2$ {\em arbitrary} scalars
$c_i$ can be stored exactly for pairs with linearly independent
$\{\Bk_i \otimes \Bv_i\}$  ($d^2$ scalar
components). Compare to DeltaNet 
with goal $\BS \Bk_i = \Bv_i$ that represent $d$ value vectors. 
In an orthonormal design ($\{\Bk_i\}$ orthonormal and $\{\Bv_i\}$
orthonormal), the minimum-norm solution has the closed form
\begin{equation}
\BS^\star \; = \; \sum_{i} c_i \, \Bv_i \Bk_i^\top \, ,
\qquad
\Bv_j^\top \BS^\star \Bk_i \; = \; c_i \ \delta_{ij} \, ,
\end{equation}
a {\em target-weighted Hebbian superposition} realizing the learned
Kronecker delta with per-pair gains. The memory is thus best
understood as a {\em pair-addressed scalar table}. It stores a number
$c_i$ at the address $(\Bk_i, \Bv_i)$, writes by rank-one
updates, reads by the bilinear probe $\Bv^\top \BS \Bk$, rewrites by
presenting the same pair with a new target, and deletes via $c = 0$
(Property~\ref{prop:Atimedistinguished}\,(iii)).
\end{propertyA}

Next, we exemplify the fact 
that KronLSTM can store $d^2$ numbers.
Assume $d$ orthonormal keys $\{\Bk_1,\ldots,\Bk_d\}$ and
$d$ orthonormal values $\{\Bv_1,\ldots,\Bv_d\}$.
Then we can build $d^2$ key-value pairs $(\Bk_i,\Bv_j)$ with
their associated $d^2$ numbers $c_{ij}$.
We set $\beta=1$ and start with $\BS_0 = \BZe$,
then we iteratively update $\BS_t$ with randomly chosen 
$i$ and $j$:
\begin{equation}
\BS_{t+1} \;=\;  \BS_t
  - \beta\,\big(\Bv_j^\top \,\BS_t \, \Bk_i \, - \, c_{ij}\big)\: \Bv_j \,
  \Bk_i^\top \, .
\end{equation}
If all pairs $(\Bk_i,\Bv_j)$ or numbers $c_{ij}$ are seen at least once, the update converged.
If pair $(\Bk_i,\Bv_j)$ has already been used in the update rule, then
$\Bv_j^\top \BS_t  \Bk_i = c_{ij}$ and $\BS_{t+1}=\BS_t$.
Therefore, if all pairs are seen, $\BS_t$ does not change any more and the update converged.
If pair $(\Bk_i,\Bv_j)$ has NOT been stored, then
$\Bv_j^\top \BS_t  \Bk_i = 0$ and 
$\BS_{t+1}=\BS_t + c_{ij} \Bv_j \Bk_i^\top$.
Therefore, if all pairs $(\Bk_i,\Bv_j)$ or numbers $c_{ij}$ 
are seen at least once, the memory matrix $\BS_t$ is
\begin{align}
\BS_t \; &= \;  \sum_{i=1,j=1}^{d,d} c_{ij} \, \Bv_j \, \Bk_i^\top 
  \; = \; \BV \, \BC^\top \, \BK^\top\, ,
\end{align}
with
\begin{align}
\BK \, &= \, \left[ \Bk_1, \, \dots, \, \Bk_d \right] \in \dR^{d \times d} \, , \quad \BV \, = \, \left[ \Bv_1, \, \dots, \, \Bv_d \right] \in \dR^{d \times d} \, , \quad
\BC \, = \, \left( c_{ij} \right)_{i,j=1}^{d} \in \dR^{d \times d} \, .
\end{align}
Since $\Bv_j^\top \BS_t  \Bk_i = c_{ij}$, we can now retrieve $d^2$
stored numbers.

We have
\begin{align}
\BC^\top \; &= \;  \BV^\top \, \BS_t \, \BK\, ,
\end{align}
since $\BK$ and $\BV$ are orthogonal. The transpose appears
because the first index of $\BC$ is the key index while the first
index of $\BS_t$ is the value index.

\begin{propertyA}[Norm and energy dynamics]
\label{prop:Aenergy}
Using $\abra{\Bv_t \Bk_t^\top,\BS_t}_{\rF} = y_t = r_t + c_t$,
\begin{equation}
\norm{\BS_{t+1}}_{\rF}^2 - \norm{\BS_t}_{\rF}^2
\;=\; -\,2\beta\, r_t (r_t + c_t) \;+\; \beta^2 r_t^2\,\norm{\Bk_t}^2 \,
\norm{\Bv_t}^2 ,
\end{equation}
so within the stable range, the state norm stays bounded whenever the
constraint system is consistent. This is a consequence of Fej\'er
monotonicity: the distance to any fixed solution is non-increasing.
\end{propertyA}

\paragraph{Summary.}
The time-varying-target rule \eqref{eq:ATimeRule} is the most general
member of the Kronecker Delta Rule family: SGD on the scalar quadratic
loss $\tfrac12 (\Bv_t^\top \BS \Bk_t - c_t)^2$, equivalently
Widrow--Hoff LMS on the feature $\Ba_t = \Bk_t \otimes \Bv_t$ with
desired signal $c_t$, equivalently relaxed Kaczmarz on the hyperplanes
$\abra{\Bv_t \Bk_t^\top, \BS}_{\rF} = c_t$. The entire analytic
framework is target-independent. That comprises stability range and Newton step
$\beta^\star$, generalized-Householder transition factors with their
parallelization and negative-eigenvalue state-tracking regime, doubly
gated interference, Fej\'er-monotone linear convergence, $d^2$
capacity. The targets enter only the driving term, {\em
  linearly}: memories superpose in $(c_t)$, and in an orthonormal
design the memory converges to the target-weighted Hebbian
superposition $\sum_i c_i \Bv_i \Bk_i^\top$ realizing $\Bv_j^\top \BS
\Bk_i = c_i\,\delta_{ij}$. Distinguished targets recover
$c_t \equiv 1$, $c_t \equiv c$, $c_t = \norm{\Bv_t}^2$ and
add two new members: the correlation target $c_t = \Bv_t^\top \Bk_t$,
for which the identity is a universal fixed point and $\BS$
parameterizes learned deviations from the standard inner product, and
the erasure target $c_t = 0$, which deletes a single binding in one
Newton step. Functionally, $\BS$ is a pair-addressed scalar table with
write, read, rewrite, and delete operations, therefore it is an editable $d^2$-slot
scalar memory. In the stochastic view, it is an online in-context
regressor of the signal $c_t$ on the bilinear feature.

\subsection{Minimally Invasive Memory Revision}
\label{sec:Aminimal}

The Kronecker Delta Rule is not merely {\em a} correction of the
addressed association: at the Newton step it is {\em the} smallest
possible one, in every sense in which smallness can reasonably be
asked for. This section elaborates on this topic. Throughout,
$\BS_t \in \dR^{d \times d}$, the read-out of the pair
$(\Bk_t, \Bv_t)$ is $\Bv_t^\top \BS_t \Bk_t$, the target is $c_t$, and
\begin{equation}
r_t \; := \; \Bv_t^\top \BS_t \Bk_t \, - \, c_t
\end{equation}
is the residual that should be pushed to zero.

\subsubsection{The admissible set}

A correction is a matrix $\BE \in \dR^{d \times d}$ added to the
memory. It {\em corrects the addressed association} if
$\Bv_t^\top (\BS_t + \BE) \Bk_t = c_t$. Using
$\Bv_t^\top \BE \Bk_t = \abra{\Bv_t \Bk_t^\top, \BE}_{\rF}$, the set of
such corrections is
\begin{equation} \label{eq:Aadmissible}
\cE_t \; := \; \Big\{ \BE \in \dR^{d \times d} \; : \;
\abra{\Bv_t \Bk_t^\top, \, \BE}_{\rF} \; = \; -\, r_t \Big\} \, ,
\end{equation}
an affine hyperplane in the Hilbert space
$(\dR^{d \times d}, \abra{\cdot,\cdot}_{\rF})$: it is nonempty, and it
is a proper affine subspace because $\Bv_t \Bk_t^\top \neq \BZe$
whenever the pair is nondegenerate. Every statement below compares
members of $\cE_t$. We hold the {\em requirement} fixed
(the addressed association must come out right) and ask which
correction achieves it with the least collateral change. Write
\begin{equation} \label{eq:AEstar}
\BE^\star \; := \; -\, \frac{r_t}{\norm{\Bk_t}^2 \norm{\Bv_t}^2}\;
\Bv_t \Bk_t^\top
\; = \; \frac{c_t - \Bv_t^\top \BS_t \Bk_t}
{\norm{\Bk_t}^2 \norm{\Bv_t}^2}\; \Bv_t \Bk_t^\top \, ,
\end{equation}
which is exactly the Kronecker Delta Rule update at the Newton step
$\beta = \beta^\star = 1/(\norm{\Bk_t}^2 \norm{\Bv_t}^2)$.

\subsubsection{Minimality in norm}

\begin{theoremA}[The Newton step is the smallest correction]
\label{thm:Aminnorm}
$\BE^\star \in \cE_t$, and for every $\BE \in \cE_t$
\begin{equation} \label{eq:Aminnorm}
\norm{\BE}_{\rF} \;\; \ge \;\;
\frac{\abs{r_t}}{\norm{\Bk_t}\, \norm{\Bv_t}}
\;\; = \;\; \norm{\BE^\star}_{\rF} \, ,
\end{equation}
with equality if and only if $\BE = \BE^\star$.
\end{theoremA}

\begin{proof}
Admissibility is a computation:
$\abra{\Bv_t \Bk_t^\top, \BE^\star}_{\rF} = -r_t
\norm{\Bv_t \Bk_t^\top}_{\rF}^2 / (\norm{\Bk_t}^2\norm{\Bv_t}^2) =
-r_t$, using
$\norm{\Bv_t \Bk_t^\top}_{\rF} = \norm{\Bk_t} \norm{\Bv_t}$.
For the bound, apply Cauchy--Schwarz to the constraint in
\eqref{eq:Aadmissible}:
\begin{equation}
\abs{r_t} \; = \; \abs{\abra{\Bv_t \Bk_t^\top, \BE}_{\rF}}
\; \le \; \norm{\Bv_t \Bk_t^\top}_{\rF}\, \norm{\BE}_{\rF}
\; = \; \norm{\Bk_t}\, \norm{\Bv_t}\, \norm{\BE}_{\rF} \, ,
\end{equation}
which is \eqref{eq:Aminnorm}. Cauchy--Schwarz holds with equality
exactly when $\BE$ is a scalar multiple of $\Bv_t \Bk_t^\top$, and the
constraint then pins that scalar to
$-r_t/(\norm{\Bk_t}^2\norm{\Bv_t}^2)$, i.e.\ $\BE = \BE^\star$.
\end{proof}

Equivalently the geometric view:
$\BS_t + \BE^\star$ is the orthogonal projection of $\BS_t$ onto the
constraint hyperplane $\cH_t = \{\BS : \Bv_t^\top \BS \Bk_t = c_t\}$,
so the Newton step moves the memory along the shortest path from
where it is to where the constraint is satisfied.

\begin{theoremA}[Minimality in {\em every} unitarily invariant norm]
\label{thm:Aallnorms}
Let $\vertiii{\cdot}$ be any unitarily invariant norm on
$\dR^{d\times d}$. Then
\begin{equation}
\min_{\BE \in \cE_t} \vertiii{\BE} \; = \; \vertiii{\BE^\star} \, ,
\end{equation}
i.e.\ the {\em same} $\BE^\star$ is optimal simultaneously for all of
them. In particular, for every Schatten-$p$ norm,
$1 \le p \le \infty$,
\begin{equation}
\min_{\BE \in \cE_t} \norm{\BE}_{S_p}
\; = \; \norm{\BE^\star}_{S_p}
\; = \; \frac{\abs{r_t}}{\norm{\Bk_t}\,\norm{\Bv_t}} \, ,
\end{equation}
the same number for all $p$. Therefore, the nuclear, Frobenius and spectral
norms of the optimal correction coincide. The minimizer is unique for
strictly convex norms (Schatten $1 < p < \infty$). For the spectral
and nuclear norms $\BE^\star$ is a minimizer but need not be the only
one.
\end{theoremA}

\begin{proof}
Let $\vertiii{\cdot}'$ be the dual norm, so that
$\abs{\abra{\BA,\BE}_{\rF}} \le \vertiii{\BA}' \vertiii{\BE}$ for all
$\BA, \BE$. Applying this to the constraint gives, for every
$\BE \in \cE_t$,
\begin{equation} \label{eq:Adualbound}
\vertiii{\BE} \;\; \ge \;\;
\frac{\abs{r_t}}{\vertiii{\Bv_t \Bk_t^\top}'} \, .
\end{equation}
By von Neumann's theorem a unitarily invariant norm is a symmetric
gauge function $\Phi$ of the singular values, and the dual norm is
the gauge $\Phi'$ of the singular values. The matrix
$\Bv_t \Bk_t^\top$ has the single nonzero singular value
$\norm{\Bk_t}\norm{\Bv_t}$, so
\begin{equation}
\vertiii{\Bv_t \Bk_t^\top} = \norm{\Bk_t}\norm{\Bv_t}\, \Phi(\Be_1) \, ,
\qquad
\vertiii{\Bv_t \Bk_t^\top}' = \norm{\Bk_t}\norm{\Bv_t}\, \Phi'(\Be_1) \, .
\end{equation}
Now $\Phi(\Be_1) \Phi'(\Be_1) = 1$: indeed
$\Phi'(\Be_1) = \sup\{x_1 : \Phi(\Bx) \le 1\}$, and since a symmetric
gauge is nondecreasing in each $\abs{x_i}$, replacing $\Bx$ by
$(x_1,0,\dots,0)$ does not increase $\Phi$, so the supremum is
attained on the first axis and equals $1/\Phi(\Be_1)$. Since
$\BE^\star$ is rank one with singular value
$\abs{r_t}/(\norm{\Bk_t}\norm{\Bv_t})$,
\begin{equation}
\vertiii{\BE^\star}
\; = \; \frac{\abs{r_t}}{\norm{\Bk_t}\norm{\Bv_t}}\, \Phi(\Be_1)
\; = \; \frac{\abs{r_t}}{\vertiii{\Bv_t \Bk_t^\top}'} \, ,
\end{equation}
which is exactly the lower bound \eqref{eq:Adualbound}: $\BE^\star$
attains it. The Schatten case follows from $\Phi_p(\Be_1) = 1$ for
every $p$. Uniqueness for a strictly convex norm holds because
$\cE_t$ is convex and a strictly convex function has at most one
minimizer on a convex set; for the spectral norm uniqueness fails
already at $d = 2$ with $\Bk = \Bv = \Be_1$ and $r_t = -1$, where
$\BE = \diag(1,\delta)$ is admissible and has spectral norm $1$ for
every $\abs{\delta} \le 1$.
\end{proof}

\subsubsection{Minimality of support: the correction is surgical}

Norm-optimality says the correction is small. The next theorem has a stronger statement even: it is the only correction
that stays inside the addressed directions at all. No optimization is
involved.

\begin{theoremA}[Unique correction with addressed support]
\label{thm:Asupport}
$\BE^\star$ is the unique $\BE \in \cE_t$ whose column space is
contained in $\spanop\{\Bv_t\}$ and whose row space is contained in
$\spanop\{\Bk_t\}$.
\end{theoremA}

\begin{proof}
A matrix has column space in $\spanop\{\Bv_t\}$ and row space in
$\spanop\{\Bk_t\}$ if and only if it is of the form
$\BE = \gamma\, \Bv_t \Bk_t^\top$ for some $\gamma \in \dR$.
Admissibility forces
$\gamma \norm{\Bk_t}^2 \norm{\Bv_t}^2 = -r_t$, hence
$\gamma = -r_t/(\norm{\Bk_t}^2\norm{\Bv_t}^2)$ and
$\BE = \BE^\star$.
\end{proof}

Every other admissible correction like every DeltaNet-style
correction necessarily writes into at least one key direction
other than $\Bk_t$ or one value direction other than $\Bv_t$, i.e.\
into directions the task never asked about.

\subsubsection{Minimality of interference}

Smallness in norm translates into
smallness of the damage done to {\em other} associations. 

\begin{theoremA}[Interference caused by the correction]
\label{thm:Ainterference}
Let $(\Bk', \Bv')$ be any probe pair and let
$\Delta(\Bk',\Bv') := \Bv'^\top \BE \Bk'$ be the change its read-out
suffers.
\begin{enumerate}
\item[] (i) {\em Exact form and protection.} For $\BE = \BE^\star$,
\begin{equation} \label{eq:Adisturb}
\Delta(\Bk',\Bv')
\; = \; -\, \frac{r_t}{\norm{\Bk_t}^2 \norm{\Bv_t}^2}\,
\big(\Bv'^\top \Bv_t\big)\big(\Bk_t^\top \Bk'\big) \, ,
\end{equation}
which vanishes whenever $\Bv' \perp \Bv_t$ {\bf or}
$\Bk' \perp \Bk_t$, and obeys
$\abs{\Delta(\Bk',\Bv')} \le
\abs{r_t}\,\norm{\Bk'}\norm{\Bv'}/(\norm{\Bk_t}\norm{\Bv_t})$.
The condition is a property of the inputs alone: it can be checked in
advance and does not depend on what $\BS_t$ currently holds.
\item[] (ii) {\em Minimum expected interference.} Suppose probe pairs are
drawn with $\Bk'$ and $\Bv'$ independent and isotropic,
$\EXP[\Bk'\Bk'^\top] = \sigma_k^2 \BI$ and
$\EXP[\Bv'\Bv'^\top] = \sigma_v^2 \BI$. Then for {\em every}
$\BE \in \dR^{d\times d}$
\begin{equation} \label{eq:Aisointerf}
\EXP\Big[\big(\Bv'^\top \BE \Bk'\big)^2\Big]
\; = \; \sigma_k^2\, \sigma_v^2\; \norm{\BE}_{\rF}^2 \, .
\end{equation}
Consequently
\begin{equation}
\BE^\star \; = \; \argmin_{\BE \in \cE_t}\;
\EXP\Big[\big(\Bv'^\top \BE \Bk'\big)^2\Big] \, ,
\qquad
\min_{\BE \in \cE_t} \EXP\big[\Delta^2\big]
\; = \; \frac{\sigma_k^2 \sigma_v^2\, r_t^2}
{\norm{\Bk_t}^2\norm{\Bv_t}^2} \, ,
\end{equation}
and the minimizer is unique.
\end{enumerate}
\end{theoremA}

\begin{proof}
(i) Substitute \eqref{eq:AEstar} and use
$\Bv'^\top (\Bv_t \Bk_t^\top) \Bk' = (\Bv'^\top\Bv_t)(\Bk_t^\top\Bk')$;
the inequality is Cauchy--Schwarz on each factor.
(ii) Vectorize: with $\Be = \vecop(\BE)$ and $\Ba' = \Bk' \otimes \Bv'$
we have $\Bv'^\top \BE \Bk' = \Ba'^\top \Be$, hence
\begin{equation}
\EXP\big[(\Ba'^\top \Be)^2\big]
\; = \; \Be^\top\, \EXP\big[\Ba' \Ba'^\top\big]\, \Be
\; = \; \Be^\top \Big(\EXP[\Bk'\Bk'^\top] \otimes
\EXP[\Bv'\Bv'^\top]\Big) \Be
\; = \; \sigma_k^2 \sigma_v^2 \norm{\Be}^2 ,
\end{equation}
using independence and $\norm{\Be} = \norm{\BE}_{\rF}$. The objective
is therefore a positive multiple of $\norm{\BE}_{\rF}^2$, so
minimizing it over $\cE_t$ is the problem already solved by
Theorem~\ref{thm:Aminnorm}, with the same unique solution.
\end{proof}

Part (ii) is what justifies the name: for isotropic probe statistics,
{\em minimum norm and minimum expected interference are literally the
same objective}, so the Newton step is minimally invasive in the
operational sense and not merely in the geometric one. 
It causes the least expected damage to the rest of the memory.

\subsubsection{Relaxed steps: the correction never moves away from a
  correct memory}

In practice, $\beta$ is a learned gate rather than $\beta^\star$. Write
$\beta = \lambda \beta^\star$, so that the actual update is
$\BS_{t+1} = \BS_t + \lambda \BE^\star$ with relaxation
$\omega = \beta \norm{\Bk_t}^2\norm{\Bv_t}^2$.

\begin{theoremA}[Exact Fej\'er identity]
\label{thm:Afejer}
Let $\BS^{\circ}$ be {\em any} memory that already satisfies the
addressed association, $\Bv_t^\top \BS^{\circ} \Bk_t = c_t$. Then
\begin{equation} \label{eq:Afejerid}
\norm{\BS_{t+1} - \BS^{\circ}}_{\rF}^2
\; = \; \norm{\BS_t - \BS^{\circ}}_{\rF}^2
\; - \; \lambda\,(2 - \lambda)\;
\frac{r_t^2}{\norm{\Bk_t}^2\norm{\Bv_t}^2} \, .
\end{equation}
Hence for every $\lambda \in (0,2)$ the update {\em strictly decreases}
the distance to every correct memory, the decrease is largest at the
Newton step $\lambda = 1$, and at $\lambda \in \{0,2\}$ the distance is
unchanged.
\end{theoremA}

\begin{proof}
Expand, using $\abra{\BE^\star, \BS_t - \BS^{\circ}}_{\rF} =
-\,\tfrac{r_t}{\norm{\Bk_t}^2\norm{\Bv_t}^2}\,
\big(\Bv_t^\top \BS_t \Bk_t - \Bv_t^\top \BS^{\circ} \Bk_t\big) =
-\,r_t^2/(\norm{\Bk_t}^2\norm{\Bv_t}^2) = -\norm{\BE^\star}_{\rF}^2$:
\begin{equation}
\norm{\BS_t - \BS^{\circ} + \lambda \BE^\star}_{\rF}^2
= \norm{\BS_t - \BS^{\circ}}_{\rF}^2
+ 2\lambda \abra{\BE^\star, \BS_t - \BS^{\circ}}_{\rF}
+ \lambda^2 \norm{\BE^\star}_{\rF}^2 ,
\end{equation}
which is \eqref{eq:Afejerid} after inserting
$\norm{\BE^\star}_{\rF}^2 = r_t^2/(\norm{\Bk_t}^2\norm{\Bv_t}^2)$.
\end{proof}

This is the strongest sense in which a single write is harmless: it
does not merely perturb other associations little, it moves the memory
{\em closer} to every memory that is consistent with the association
being written and that simultaneously, for all of them.

\subsubsection{Greedy minimality is global minimality}

Per-step optimality does not by itself imply that a {\em sequence} of
such steps is optimal. We show that our does.

\begin{proposition}[The limit is the minimum-norm memory]
\label{prop:Aglobal}
Let $\{(\Bk_i, \Bv_i, c_i)\}_{i=1}^n$ be a consistent system of
associations. Apply the rule from $\BS_0 = \BZe$ with relaxation
$\omega \in (0,2)$, visiting each pair infinitely often either cyclically or
by random sampling. Then $\BS_t$ converges to
\begin{equation}
\BS^\sharp \; = \; \argmin\Big\{ \norm{\BS}_{\rF} \; : \;
\Bv_i^\top \BS \Bk_i = c_i \;\; \text{for all } i \Big\} \, ,
\end{equation}
the minimum Frobenius-norm memory satisfying {\em all} the
associations.
\end{proposition}

\begin{proof}
This is the classical convergence theorem for relaxed Kaczmarz
iterations \citep{Kaczmarz:37,Cegielski:12}, applied to the vectorized
system $\Ba_i^\top \Bs = c_i$ with $\Ba_i = \Bk_i \otimes \Bv_i$. Each
iterate stays in $\Bs_0 + \spanop\{\Ba_1,\dots,\Ba_n\}$, which for
$\Bs_0 = \BZe$ is the row space of the system; the iterates converge
to a solution, and the unique solution in the row space is the
minimum-norm one.
\end{proof}

So the memory built by a stream of minimally invasive revisions is
itself the least-norm memory compatible with everything it was asked
to store: locally and globally, the rule writes as little as the task
permits.

\subsubsection{Less invasive than DeltaNet}

Theorems~\ref{thm:Aminnorm}--\ref{thm:Ainterference} hold {\em given the
requirement} encoded in $\cE_t$: make the scalar read-out of the
addressed pair equal to $c_t$, and do nothing else. DeltaNet's step
has to solve the $d$ constraints $\BS \Bk_t = \Bv_t$, therefore
its correction is necessarily larger. 
The claim here is therefore not that DeltaNet is
suboptimal for its own task, but that the two tasks differ: a
recognition memory needs one scalar corrected, and the Kronecker
Delta Rule corrects it with provably minimal disturbance, whereas a
retrieval memory must rewrite an entire value vector and pays for it
in footprint.

\subsection{Convergence of the KronLSTM memory}

We analyze the expected solution and steady state of the randomized
Kronecker Delta Rule with more than $d^2$ associations.
Running memories i.e.\ online updates are covered via the steady state
considerations.

\subsubsection{Overview}
 We analyze the Kronecker Delta Rule
  $\BS_{t+1} = \BS_t - \beta\,(\Bv_{i_t}^\top \BS_t \Bk_{i_t} -
  c_{i_t})\, \Bv_{i_t} \Bk_{i_t}^\top$ with randomly drawn indices
  when the number of associations exceeds the capacity: $n > d^2$.
  The system is then generically {\em inconsistent}, and the iterates
  no longer converge to a point. Instead we have:
  (i) The {\bf expected iterate} converges linearly to the
  sampling-weighted {\bf least-squares solution}
  $\BS^{\rL\rS} = \argmin_{\BS} \sum_i p_i (\Bv_i^\top \BS \Bk_i -
  c_i)^2$. This is for uniform sampling exactly the minimizer of the loss
  $\rL$.
  (ii) The iterates converge in distribution to a unique {\bf
    stationary distribution} whose mean is {\em exactly}
  $\BS^{\rL\rS}$ (no bias at any order in $\beta$) and whose
  covariance is $O(\beta)$, given by a Lyapunov-type equation driven
  by the residuals at $\BS^{\rL\rS}$.
  (iii) The steady-state loss exceeds the minimum by the classical LMS
  misadjustment,
  $\EXP[\rL] - \rL_{\min} = \tfrac{\beta}{2}\,
  \EXP\big[\rho_i^2 \norm{\Bk_i}^2 \norm{\Bv_i}^2\big] + O(\beta^2)$,
  where $\rho_i$ are the optimal residuals.
  
  This appendix contains {\bf complete proofs} of both theorems:
  the mean dynamics are solved exactly; existence, uniqueness, and
  geometric convergence of the stationary distribution are established
  by a synchronous-coupling contraction in the Wasserstein-2 metric;
  the stationary mean and the exact covariance equation are derived
  with all cross terms accounted for; the misadjustment follows from
  an exact trace identity with two-sided bounds; and the transient is
  controlled by an explicit geometric-plus-horizon estimate.
  Exact convergence can be recovered by diminishing step sizes or by
  iterate averaging, and it survives $n > d^2$ unharmed whenever the
  targets remain consistent. Notably, exact convergence 
  is always obtain for the correlation target
  $c_i = \Bv_i^\top \Bk_i$ and the norm target
  $c_i = \norm{\Bv_i}^2$ whenever an exact retrieval memory exists.

\subsubsection{Setting and notation}

Let $n > d^2$ associations $(\Bk_i, \Bv_i, c_i)$, $1 \le i \le n$, be
given. Vectorized, with $\Bs = \vecop(\BS) \in \dR^{d^2}$ and
$\Ba_i = \Bk_i \otimes \Bv_i$, therefore 
$\norm{\Ba_i} = \norm{\Bk_i} \norm{\Bv_i}$. The associations form
the linear system $\BA \Bs = \Bc$ with $\BA \in \dR^{n \times d^2}$
stacking the rows $\Ba_i^\top$. We assume $\BA$ has full column rank
$d^2$ i.e.\ the features span $\dR^{d^2}$. Since $n > d^2$, the system is
generically {\em inconsistent}: no $\BS$ satisfies all associations,
and
\begin{equation}
\rL(\BS) \; = \; \frac{1}{n} \sum_{i=1}^n
\big(\Bv_i^\top \BS \Bk_i - c_i\big)^2
\end{equation}
has a positive minimum $\rL_{\min} > 0$. At each step an index $i_t$
is drawn with probabilities $p_i > 0$, independently of the past, and
the rule is applied with fixed step size $\beta$ satisfying
\begin{equation} \label{eq:Astepcond1}
0 \; < \; \beta \, \norm{\Ba_i}^2 \; < \; 2
\qquad \text{for all } i \, .
\end{equation}
We define
\begin{align}
\bar{\BA} \; &:= \; \EXP\big[\Ba \Ba^\top\big]
\; = \; \sum_{i=1}^n p_i\, \Ba_i \Ba_i^\top \; \succ \; \BZe \, ,
&\bar{\Bb} \; &:= \; \EXP\big[c\, \Ba\big]
\; = \; \sum_{i=1}^n p_i\, c_i\, \Ba_i \, , \\
\Bs^{\rL\rS} \; &:= \; \bar{\BA}^{-1} \bar{\Bb}
\; = \; \argmin_{\Bs} \sum_{i=1}^n p_i \big(\Ba_i^\top \Bs -
c_i\big)^2 \, ,
&\rho_i \; &:= \; c_i - \Ba_i^\top \Bs^{\rL\rS}
\quad \text{(optimal residuals)} \, ,
\end{align}
so that the normal equations read
\begin{equation} \label{eq:Anormal}
\EXP\big[\rho\, \Ba\big] \; = \; \sum_{i=1}^n p_i\, \rho_i\, \Ba_i
\; = \; \BZe
\qquad \Longleftrightarrow \qquad
\bar{\Bb} \; = \; \bar{\BA}\, \Bs^{\rL\rS} \, .
\end{equation}
For uniform sampling, $\bar{\BA} = \tfrac{1}{n} \BA^\top \BA$ and
$\Bs^{\rL\rS} = (\BA^\top \BA)^{-1} \BA^\top \Bc$ is the ordinary
least-squares solution, i.e.\ the unique minimizer of $\rL$.
Throughout, expectations without subscript are over $i \sim p$, and
$\EXP_{\pi}$ denotes expectation of a random state distributed
according to a law $\pi$. Vectorized, the rule and its error form are
\begin{equation} \label{eq:Aerrrec}
\Bs_{t+1} \; = \; \BM_{i_t} \Bs_t + \beta\, c_{i_t} \Ba_{i_t} \, ,
\qquad
\Be_{t+1} \; = \; \BM_{i_t}\, \Be_t \; + \; \beta\, \rho_{i_t}\,
\Ba_{i_t} \, ,
\qquad
\BM_i \; := \; \BI - \beta\, \Ba_i \Ba_i^\top \, ,
\end{equation}
with $\Be_t := \Bs_t - \Bs^{\rL\rS}$; the error form follows by
inserting $c_i = \Ba_i^\top \Bs^{\rL\rS} + \rho_i$. The constants
\begin{equation}
\gamma_\beta \; := \; \min_i\, \beta\,\big(2 - \beta
\norm{\Ba_i}^2\big) \; > \; 0 \, ,
\qquad
\bar{\rho} \; := \; \gamma_\beta\, \lambda_{\min}(\bar{\BA}) \; \in \;
(0, 1] \, ,
\qquad
\sigma^2 \; := \; \EXP\big[\rho^2 \norm{\Ba}^2\big]
\end{equation}
appear throughout ($0<\bar{\rho} \le 1$ is shown in
Lemma~\ref{lem:Acontract}).

\subsubsection{Theorems on expected and steady state solution}

\begin{theoremA}[Expected solution: exact, unbiased least squares]
\label{thm:Amean}
Under \eqref{eq:Astepcond1}, the expected iterate obeys the {\em exact}
deterministic recursion
\begin{equation} \label{eq:Ameanrec}
\EXP\big[\Bs_{t+1}\big] \; = \;
\big(\BI - \beta \bar{\BA}\big)\, \EXP\big[\Bs_t\big]
\; + \; \beta\, \bar{\Bb}
\qquad \Longleftrightarrow \qquad
\EXP\big[\Be_{t+1}\big] \; = \;
\big(\BI - \beta \bar{\BA}\big)\, \EXP\big[\Be_t\big] \, ,
\end{equation}
whose unique fixed point is $\Bs^{\rL\rS}$, and
\begin{equation}
\norm{\EXP[\BS_t] - \BS^{\rL\rS}}_{\rF}
\; \le \; r^t\, \norm{\BS_0 - \BS^{\rL\rS}}_{\rF} \, ,
\qquad
r \; := \; \max_k\, \abs{1 - \beta \lambda_k(\bar{\BA})} \; < \; 1 \,
,
\end{equation}
so $\EXP[\BS_t] \to \BS^{\rL\rS}$ linearly. The convergence of the
mean is {\em exactly unbiased}: the fixed point of \eqref{eq:Ameanrec}
is $\Bs^{\rL\rS}$ itself, with no $O(\beta)$ correction.
\end{theoremA}

\begin{theoremA}[Steady state: stationary distribution around
  $\BS^{\rL\rS}$]
\label{thm:Asteady}
Under \eqref{eq:Astepcond1}, the iterates form a Markov chain on
$\dR^{d^2}$ that has a unique stationary distribution $\pi_\beta$
with finite second moment, and the law of $\Bs_t$ converges to
$\pi_\beta$ geometrically in the Wasserstein-2 metric:
\begin{equation} \label{eq:Aw2conv}
\cW_2\big(\text{law}(\Bs_t),\, \pi_\beta\big)
\; \le \; \big(1 - \bar{\rho}\big)^{t/2}\;
\cW_2\big(\text{law}(\Bs_0),\, \pi_\beta\big) \, .
\end{equation}
Under $\pi_\beta$, with $\Be := \Bs - \Bs^{\rL\rS}$:
\begin{enumerate}
\item[] (i) {\bf Mean:} $\EXP_{\pi_\beta}[\BS] = \BS^{\rL\rS}$ exactly.
\item[] (ii) {\bf Covariance:} $\BV_\beta := \EXP_{\pi_\beta}[\Be
  \Be^\top]$ satisfies the {\em exact} stationarity equation
\begin{equation} \label{eq:Alyap}
\bar{\BA} \BV_\beta + \BV_\beta \bar{\BA}
\; - \; \beta\, \EXP\big[\big(\Ba^\top \BV_\beta\, \Ba\big)\, \Ba
\Ba^\top\big]
\; = \; \beta\, \EXP\big[\rho^2\, \Ba \Ba^\top\big] \, ,
\end{equation}
and obeys the two-sided trace bounds
\begin{align} \label{eq:Atracebounds}
&\frac{\beta\, \sigma^2}{2} \; \le \; \tr\big(\bar{\BA} \BV_\beta\big)
\; \le \; \frac{\beta\, \sigma^2}{2 - \beta \max_i \norm{\Ba_i}^2}\,
, \\ \nonumber
&\EXP_{\pi_\beta}\!\left[\norm{\BS - \BS^{\rL\rS}}_{\rF}^2\right]
\; \le \;
\frac{\beta\, \sigma^2}
{\big(2 - \beta \max_i \norm{\Ba_i}^2\big)\,
\lambda_{\min}(\bar{\BA})} \, ;
\end{align}
in particular $\BV_\beta = O(\beta)$, and to leading order
\eqref{eq:Alyap} is the Lyapunov equation
$\bar{\BA} \BV + \BV \bar{\BA} = \beta \EXP[\rho^2 \Ba \Ba^\top]$.
\item[] (iii) {\bf Steady-state loss (misadjustment):} for uniform sampling,
\begin{equation} \label{eq:Amisadjust}
\EXP_{\pi_\beta}\big[\rL(\BS)\big] \; - \; \rL_{\min}
\; = \; \tr\big(\bar{\BA} \BV_\beta\big)
\; = \; \frac{\beta}{2}\Big(\sigma^2 +
\EXP\big[\big(\Ba^\top \BV_\beta \Ba\big) \norm{\Ba}^2\big]\Big) \, ,
\end{equation}
an exact identity, whence
\begin{equation} \label{eq:Amisadjustbound}
\frac{\beta}{2}\, \sigma^2
\; \le \;
\EXP_{\pi_\beta}\big[\rL(\BS)\big] - \rL_{\min}
\; \le \;
\frac{\beta}{2}\, \sigma^2 \cdot
\frac{2}{2 - \beta \max_i \norm{\Ba_i}^2} \, ,
\qquad
\sigma^2 = \EXP\big[\rho_i^2 \norm{\Bk_i}^2 \norm{\Bv_i}^2\big] \, ,
\end{equation}
i.e.\ $\EXP_{\pi_\beta}[\rL] = \rL_{\min} + \tfrac{\beta}{2}
\sigma^2 + O(\beta^2)$. This is the classical Widrow--Hoff LMS
misadjustment for the bilinear feature $\Ba = \Bk \otimes \Bv$.
\item[] (iv) {\bf Transient:} with $\Bw := \EXP\big[\rho \norm{\Ba}^2
  \Ba\big]$,
\begin{equation} \label{eq:Atransient}
\EXP\norm{\Be_t}^2 \; \le \;
\Big(1 - \tfrac{\bar{\rho}}{2}\Big)^{t} \norm{\Be_0}^2
\; + \; \frac{2 \beta^2 \sigma^2}{\bar{\rho}}
\; + \; \frac{4 \beta^4 \norm{\Bw}^2}{\bar{\rho}^2} \, ,
\end{equation}
a geometric approach to an $O(\beta)$ horizon (note $\bar{\rho} =
\Theta(\beta)$): overload does not slow the transient, it only
replaces the limit point by a limit distribution.
\end{enumerate}
\end{theoremA}

\subsubsection{Proofs of the theorems}

Throughout, $\cF_t := \sigma(\Bs_0, i_0, \dots, i_{t-1})$ denotes the
filtration of the past; $i_t$ is independent of $\cF_t$ by
assumption. We repeatedly use that \eqref{eq:Aerrrec} is {\em linear}
in the state with coefficients independent of the state.

\paragraph{Two lemmas crucial for the proofs of the theorems.}

\begin{lemmaA}[One-step mean-square contraction of the homogeneous
  part]
\label{lem:Acontract}
For every fixed $\Bx \in \dR^{d^2}$,
\begin{equation} \label{eq:Aonestep1}
\EXP\big[\norm{\BM_i\, \Bx}^2\big]
\; = \; \norm{\Bx}^2 \; - \; \sum_{i=1}^n p_i\, \beta \big(2 - \beta
\norm{\Ba_i}^2\big) \big(\Ba_i^\top \Bx\big)^2
\; \le \; \big(1 - \bar{\rho}\big)\, \norm{\Bx}^2 \, ,
\end{equation}
and $\bar{\rho} \in (0, 1]$.
\end{lemmaA}

\begin{proof}
Since $\BM_i$ is symmetric,
\begin{equation}
\norm{\BM_i \Bx}^2
\; = \; \Bx^\top \big(\BI - \beta \Ba_i \Ba_i^\top\big)^2 \Bx
\; = \; \norm{\Bx}^2 - 2 \beta \big(\Ba_i^\top \Bx\big)^2
+ \beta^2 \norm{\Ba_i}^2 \big(\Ba_i^\top \Bx\big)^2 \, ,
\end{equation}
which is the equality in \eqref{eq:Aonestep1} after averaging over
$i \sim p$. By \eqref{eq:Astepcond1}, $\beta (2 - \beta
\norm{\Ba_i}^2) \ge \gamma_\beta > 0$ for every $i$, and
$\sum_i p_i (\Ba_i^\top \Bx)^2 = \Bx^\top \bar{\BA} \Bx \ge
\lambda_{\min}(\bar{\BA}) \norm{\Bx}^2$, hence the inequality with
$\bar{\rho} = \gamma_\beta \lambda_{\min}(\bar{\BA})$. Finally
$\bar{\rho} \le 1$ because the left-hand side of \eqref{eq:Aonestep1}
is nonnegative for every $\Bx$. We assumed that 
$\bar{\BA}$ has full rank, therefore
$0<\lambda_{\min}(\bar{\BA})$ and with $0<\gamma_\beta$ we have 
$0<\bar{\rho}$.
\end{proof}

\begin{lemmaA}[The chain lives in $\cP_2$ and its kernel is a
  $\cW_2$-contraction]
\label{lem:Aw2}
Let $P$ denote the transition kernel of \eqref{eq:Aerrrec} and let
$\cP_2$ be the space of probability laws on $\dR^{d^2}$ with finite
second moment, equipped with the Wasserstein-2 metric $\cW_2$. Then
$P$ maps $\cP_2$ into $\cP_2$, and for all $\mu, \nu \in \cP_2$,
\begin{equation} \label{eq:Aw2contract}
\cW_2\big(\mu P,\, \nu P\big) \; \le \; \sqrt{1 - \bar{\rho}}\;\,
\cW_2(\mu, \nu) \, .
\end{equation}
Consequently $P$ has a unique invariant law $\pi_\beta \in \cP_2$
and \eqref{eq:Aw2conv} holds.
\end{lemmaA}

\begin{proof}
{\em $P$ preserves $\cP_2$:} for $\Be \sim \mu \in \cP_2$ and an
independent index $i$,
$\EXP \norm{\BM_i \Be + \beta \rho_i \Ba_i}^2 \le 2\, \EXP
\norm{\BM_i \Be}^2 + 2 \beta^2\, \EXP[\rho^2 \norm{\Ba}^2] \le 2\,
\EXP \norm{\Be}^2 + 2 \beta^2 \sigma^2 < \infty$,
using Lemma~\ref{lem:Acontract}.

{\em Contraction:} let $(\Bx, \By)$ be an optimal $\cW_2$-coupling of
$(\mu, \nu)$ and couple the two chains {\em synchronously}: draw one
index $i$ (independent of $(\Bx, \By)$) and set
$\Bx' = \BM_i \Bx + \beta \rho_i \Ba_i$,
$\By' = \BM_i \By + \beta \rho_i \Ba_i$. The additive terms cancel in
the difference,
\begin{equation}
\Bx' - \By' \; = \; \BM_i\, (\Bx - \By) \, ,
\end{equation}
so by Lemma~\ref{lem:Acontract}, first conditionally on $(\Bx, \By)$
and then in expectation,
\begin{equation}
\EXP \norm{\Bx' - \By'}^2 \; \le \; (1 - \bar{\rho})\,
\EXP \norm{\Bx - \By}^2 \; = \; (1 - \bar{\rho})\,
\cW_2(\mu, \nu)^2 \, .
\end{equation}
Since $(\Bx', \By')$ is {\em some} coupling of $(\mu P, \nu P)$,
$\cW_2(\mu P, \nu P)^2 \le \EXP\norm{\Bx' - \By'}^2$, which is
\eqref{eq:Aw2contract}.

{\em Conclusion:} $(\cP_2, \cW_2)$ is a complete metric space (see,
e.g., \cite[Theorem 6.18]{Villani:09}), so Banach's fixed-point theorem yields a unique
$\pi_\beta \in \cP_2$ with $\pi_\beta P = \pi_\beta$ and
$\cW_2(\mu P^t, \pi_\beta) \le (1 - \bar{\rho})^{t/2} \cW_2(\mu,
\pi_\beta)$ for every $\mu \in \cP_2$, which is \eqref{eq:Aw2conv}.
Convergence in $\cW_2$ implies convergence in distribution together
with convergence of first and second moments.
\end{proof}

{\bf Comments to the proof of Lemma~\ref{lem:Aw2} in 9 bullet points.} \label{checked} \newline
$\bullet$ {\em 1) Overview.}
How a mixture of non-contractive matrices produces a
  contraction: the kernel $P$ and the coupling argument.
The matrices $\BM_i$ are not themselves contractions - by
Lemma~\ref{lem:Acontract}, $\norm{\BM_i}_2 = 1$ for {\em every} single
$i$, since $\BM_i$ has eigenvalue $1$ on $\Ba_i^\perp$. Contractivity
in Lemma~\ref{lem:Aw2} is a property of the {\em mixture} $P$, not of
any individual $\BM_i$, and $P$ is not a matrix at all: it is a
Markov kernel acting on laws in $\cP_2$, and the relevant notion of
size for it is the Wasserstein metric $\cW_2$ on $\cP_2$ itself -
not any norm on $\dR^{d^2}$ from which $P$'s contraction could be
read off the $\BM_i$'s directly. This note spells out how $P$ is
built from the $\BM_i$'s, in what sense $P$ does and does not have an
``operator norm,'' and why the coupling argument nonetheless produces
the contraction \eqref{eq:Aw2contract}.

$\bullet$ {\em 2) Constructing $P$ from the $\BM_i$.}
For a fixed state $\Bx \in \dR^{d^2}$, define the $n$ deterministic
affine maps
\begin{equation}
\Phi_i(\Bx) \; := \; \BM_i \Bx + \beta \rho_i \Ba_i \, ,
\qquad i = 1, \dots, n \, .
\end{equation}
One step of the chain from $\Bx$ lands at $\Phi_i(\Bx)$ with
probability $p_i$, so the transition kernel is the mixture of point
masses
\begin{equation} \label{eq:Akerneldef}
P(\Bx, \cdot) \; := \; \sum_{i=1}^n p_i\, \delta_{\Phi_i(\Bx)}(\cdot)
\, ,
\end{equation}
and, for a law $\mu$, $\mu P$ is defined by
$(\mu P)(B) = \int P(\Bx, B)\, \mu(d\Bx)$ - equivalently, $\mu P$
is the law of $\Phi_I(\BX)$ for $\BX \sim \mu$ and $I \sim p$ drawn
{\em independently} of $\BX$.

$B$ is a Borel (measurable) subset of $\dR^{d^2}$ - an arbitrary
{\em test set}, held fixed, to which the kernel assigns a
probability. Concretely,
\begin{equation}
P(\Bx, B) \; := \; \PR\big(\Phi_I(\Bx) \in B\big)
\; = \; \sum_{i=1}^n p_i\; \mathbbm{1}\big[\Phi_i(\Bx) \in B\big]
\; = \; \sum_{i=1}^n p_i\; \delta_{\Phi_i(\Bx)}(B) \, ,
\end{equation}
where $\Phi_i(\Bx) = \BM_i \Bx + \beta \rho_i \Ba_i$ is the (fixed,
deterministic) image of $\Bx$ under transition $i$, and $I \sim p$ is
the random index selecting which transition fires. In words:
$P(\Bx, B)$ is the probability that one step of the chain, started at
$\Bx$, lands in $B$ - obtained by summing the weights $p_i$ of
every index $i$ whose image $\Phi_i(\Bx)$ happens to fall inside $B$.

$P$ as a function of each argument separately. 
The kernel $P : \dR^{d^2} \times \cB(\dR^{d^2}) \to [0,1]$ plays a
different role in each of its two slots:
\begin{itemize}[leftmargin=1.7em]
\item Fixed $\Bx$, varying $B$: $P(\Bx, \cdot)$ is itself a
  probability measure on $\dR^{d^2}$ - concretely, the discrete
  measure
\begin{equation}
P(\Bx, \cdot) \; = \; \sum_{i=1}^n p_i\, \delta_{\Phi_i(\Bx)}(\cdot)
\end{equation}
  putting mass $p_i$ on the point $\Phi_i(\Bx)$ for each $i$ (masses
  are summed if several $\Phi_i(\Bx)$ coincide). This is the law of
  the chain's next state, given that it is currently at $\Bx$.
\item Fixed $B$, varying $\Bx$: $P(\cdot, B)$ is a bounded
  measurable function of $\Bx$, since each indicator
  $\mathbbm{1}[\Phi_i(\Bx) \in B]$ is measurable in $\Bx$ and the sum is
  finite.
\end{itemize}
This is the standard convention for a Markov transition kernel:
measurable in the first argument, a probability measure in the
second.

$\bullet$ {\em 3) the pushforward measure $\mu P$ is well defined.}
For a law $\mu$ of the current state, the one-step-ahead law $\mu P$
is defined by integrating the {\em second} argument $B$ (held fixed)
against $\mu$ acting on the {\em first} argument $\Bx$:
\begin{equation}
(\mu P)(B) \; := \; \int_{\dR^{d^2}} P(\Bx, B) \; \mu(d\Bx) \, .
\end{equation}
Equivalently, $\mu P$ is the law of $\Phi_I(\BX)$ for $\BX \sim \mu$
and $I \sim p$ drawn independently of $\BX$: draw a random current
state $\BX$, draw an independent random index $I$, and output
$\Phi_I(\BX)$. The two roles of $B$ above are exactly what make this
integral meaningful: for each fixed $\Bx$ in the integral,
$P(\Bx, B)$ is a number in $[0,1]$ (the probability mass the
one-step law from $\Bx$ places on $B$); integrating that number
against $\mu(d\Bx)$ averages it over the randomness in the current
state, and the result, as a function of $B$, is again a probability
measure. It is the law $\mu P$ appearing throughout the steady-state
analysis of the rule.

$\bullet$ {\em 4) $P$ has no operator norm on $\dR^{d^2}$.}
Thus $P$ is linear as an operator on (signed) measures
($(\alpha \mu + (1-\alpha)\nu) P = \alpha \mu P + (1-\alpha) \nu P$),
but this linearity lives on the infinite-dimensional space of
measures, not on $\dR^{d^2}$: there is no $d^2 \times d^2$ matrix
representing $P$, so $P$ has {\em no} operator norm on
$(\dR^{d^2}, \norm{\cdot})$ the way each $\BM_i$ does. What $P$ {\em
  does} have is an operator norm in the sense appropriate to its
actual domain: a Lipschitz modulus on the metric space
$(\cP_2, \cW_2)$, namely the smallest $L$ such that
$\cW_2(\mu P, \nu P) \le L\, \cW_2(\mu, \nu)$ for all $\mu, \nu \in
\cP_2$. That modulus is exactly what \eqref{eq:Aw2contract} bounds by
$\sqrt{1 - \bar{\rho}} < 1$. It cannot be read off a per-map bound
such as $\max_i \norm{\BM_i}_2$, which equals $1$ by
Lemma~\ref{lem:Acontract} and would certify no contraction whatsoever;
obtaining a modulus strictly below $1$ requires the coupling argument
below, precisely because it is a statement about $P$ acting on laws,
not about any $\BM_i$ acting on $\dR^{d^2}$.

$\bullet$ {\em 5) The Wasserstein metric is essential to prove the contraction.}
The definition
\begin{equation}
\cW_2(\mu, \nu)^2 \; = \; \inf_{\gamma \in \Gamma(\mu, \nu)}
\EXP_\gamma\big[\norm{\BX - \BY}^2\big]
\end{equation}
is an {\em infimum over all couplings} $\gamma$ of $(\mu, \nu)$
(joint laws with the prescribed marginals). This has an immediate but
crucial consequence: to upper-bound $\cW_2(\mu P, \nu P)$ it suffices
to exhibit {\em any single} coupling of $(\mu P, \nu P)$ and compute
its second moment. Note that we never need the optimal coupling of the
outputs, let alone an explicit formula for $P$ as an operator.

$\bullet$ {\em 6) The synchronous coupling.}
We are free to choose how the randomness that drives one step of
{\em each} chain is generated, and we exploit that freedom by reusing
the same random index for both chains. Concretely: let $(\Bx, \By)$
be an {\em optimal} coupling of $(\mu, \nu)$, i.e.\ a joint law with
$\Bx \sim \mu$, $\By \sim \nu$, and
$\EXP\norm{\Bx - \By}^2 = \cW_2(\mu, \nu)^2$ (such a coupling exists
because $\dR^{d^2}$ is Polish; see \cite[Theorem 4.1]{Villani:09}).
Draw a single index
$i \sim p$, independent of $(\Bx, \By)$, and set
\begin{equation}
\Bx' \; := \; \Phi_i(\Bx) \; = \; \BM_i \Bx + \beta \rho_i \Ba_i \, ,
\qquad
\By' \; := \; \Phi_i(\By) \; = \; \BM_i \By + \beta \rho_i \Ba_i \, .
\end{equation}
Marginally, $\Bx' \sim \mu P$ and $\By' \sim \nu P$ by the very
definition \eqref{eq:Akerneldef} of $\mu P$, applied separately to each
starting law; hence $(\Bx', \By')$ {\em is} a valid coupling of
$(\mu P, \nu P)$. This coupling is not necessarily the optimal one, 
but that is all the argument requires.

$\bullet$ {\em 7) The additive target terms cancel.}
Because $\Bx'$ and $\By'$ are built from the {\em same} draw $i$, the
inhomogeneous term $\beta \rho_i \Ba_i$is identical in both and cancels
exactly in the difference:
\begin{equation}
\Bx' - \By'
\; = \; \big(\BM_i \Bx + \beta \rho_i \Ba_i\big)
- \big(\BM_i \By + \beta \rho_i \Ba_i\big)
\; = \; \BM_i\, (\Bx - \By) \, .
\end{equation}
The part of the update that cancels out is exactly the part
carrying the targets $c_i$.
This is the crux of the construction: it converts a statement about
two {\em inhomogeneous, randomly driven} affine updates into a
statement about a single {\em homogeneous linear} map, $\BM_i$,
applied to the fixed vector $\Bx - \By$. This is exactly the object
Lemma~\ref{lem:Acontract} controls. Had we instead coupled the two
chains using {\em independent} indices for $\Bx'$ and $\By'$, the
noise terms would not cancel and the argument would not go through;
the whole content of the synchronous-coupling trick is this
cancellation.

$\bullet$ {\em 8) Deriving the contraction.}
Since $i$ is independent of $(\Bx, \By)$, condition on $(\Bx, \By)$
first and average over $i$:
\begin{equation}
\EXP\big[\norm{\Bx' - \By'}^2 \,\big|\, \Bx, \By\big]
\; = \; \sum_{i=1}^n p_i\, \norm{\BM_i (\Bx - \By)}^2
\; \le \; \big(1 - \bar{\rho}\big)\, \norm{\Bx - \By}^2
\end{equation}
by Lemma~\ref{lem:Acontract}, applied pointwise to the vector $\Bx -
\By$ (fixed once we condition). Taking the outer expectation over
$(\Bx, \By)$ and using that this coupling was chosen optimal,
\begin{equation}
\EXP\norm{\Bx' - \By'}^2
\; \le \; \big(1 - \bar{\rho}\big)\, \EXP\norm{\Bx - \By}^2
\; = \; \big(1 - \bar{\rho}\big)\, \cW_2(\mu, \nu)^2 \, .
\end{equation}
Finally, since $(\Bx', \By')$ is {\em a} coupling of $(\mu P, \nu
P)$, the infimum defining $\cW_2(\mu P, \nu P)^2$ is at most the cost
of this particular coupling:
\begin{equation}
\cW_2(\mu P, \nu P)^2 \; \le \; \EXP\norm{\Bx' - \By'}^2
\; \le \; \big(1 - \bar{\rho}\big)\, \cW_2(\mu, \nu)^2 \, ,
\end{equation}
and taking square roots recovers \eqref{eq:Aw2contract} exactly.

$\bullet$ {\em 9) Summary of the approach.}
Three ingredients combine, none of which is a bound on $P$ as an
operator on $\dR^{d^2}$: (1) $\cW_2$ is an infimum over couplings, so
any one coupling of the images gives a valid upper bound; (2) the
synchronous coupling is reusing the same random index for both
copies and is a legitimate choice precisely because $P$ is defined by
mixing over that same randomness, and it makes the target-dependent
noise cancel identically; (3) what remains is a bound on
$\EXP_i \norm{\BM_i \Bz}^2$ for the {\em fixed} vector $\Bz = \Bx -
\By$, which is exactly the pointwise mean-square contraction already
established in Lemma~\ref{lem:Acontract}. This lemma is a statement purely about
the $\BM_i$'s and $\bar{\BA} = \sum_i p_i \Ba_i \Ba_i^\top$, requiring
no coupling of its own. The contraction of $P$ in $\cW_2$ is thus
assembled entirely from the per-index contraction of
Lemma~\ref{lem:Acontract}, transported to the space of distributions
by the coupling device. It is not, and could not be, a statement
recoverable from $\max_i \norm{\BM_i}_2$ alone.

\paragraph{Proof of Theorem \ref{thm:Amean}.}

\begin{proof}
{\bf Exact mean recursion.} The update \eqref{eq:Aerrrec} is linear in
$\Bs_t$, and $i_t$ is independent of $\cF_t \supseteq
\sigma(\Bs_t)$. Hence, conditioning on $\cF_t$,
\begin{equation}
\EXP\big[\Bs_{t+1} \,\big|\, \cF_t\big]
\; = \; \EXP\big[\BM_{i_t}\big]\, \Bs_t + \beta\, \EXP\big[c_{i_t}
\Ba_{i_t}\big]
\; = \; \big(\BI - \beta \bar{\BA}\big)\, \Bs_t + \beta\, \bar{\Bb}
\, ,
\end{equation}
and the tower property gives \eqref{eq:Ameanrec}. 
This equation holds with equality and is 
not an approximation: no linearization was needed because the update
itself is linear. Subtracting the fixed-point relation
$\Bs^{\rL\rS} = (\BI - \beta \bar{\BA}) \Bs^{\rL\rS} + \beta
\bar{\Bb}$, which holds by the normal equations \eqref{eq:Anormal}
($\bar{\Bb} = \bar{\BA} \Bs^{\rL\rS}$), yields the error form and,
by induction,
\begin{equation}
\EXP\big[\Be_t\big] \; = \; \big(\BI - \beta \bar{\BA}\big)^t\,
\EXP\big[\Be_0\big] \, .
\end{equation}

{\bf Spectral radius.} $\bar{\BA}$ is symmetric positive definite;
let $0 < \lambda_{\min} = \lambda_1 \le \dots \le \lambda_{d^2} =
\lambda_{\max}$ be its eigenvalues. For any unit vector $\Bx$,
$\Bx^\top \bar{\BA} \Bx = \sum_i p_i (\Ba_i^\top \Bx)^2 \le \max_i
\norm{\Ba_i}^2$, so
\begin{equation}
\beta\, \lambda_{\max} \; \le \; \beta\, \max_i \norm{\Ba_i}^2 \; <
\; 2
\end{equation}
by \eqref{eq:Astepcond1}. Hence, every eigenvalue $1 - \beta \lambda_k$
of $\BI - \beta \bar{\BA}$ lies in $(-1, 1)$:
$1 - \beta\lambda_k < 1$ since $\beta \lambda_k > 0$, and
$1 - \beta\lambda_k > -1$ since $\beta \lambda_k < 2$. The matrix is
symmetric, so its spectral norm equals
$r = \max_k |1 - \beta \lambda_k| < 1$ and
$\norm{\EXP[\Be_t]} \le r^t \norm{\Be_0}$, i.e.\ linear convergence of
$\EXP[\BS_t]$ to $\BS^{\rL\rS}$ in Frobenius norm ($\norm{\Bs}_2 =
\norm{\BS}_{\rF}$).

{\bf Exact unbiasedness.} The fixed point of the affine map $\Bz
\mapsto (\BI - \beta \bar{\BA}) \Bz + \beta \bar{\Bb}$ is unique
($\beta \bar{\BA}$ invertible) and equals $\bar{\BA}^{-1} \bar{\Bb}
= \Bs^{\rL\rS}$ for {\em every} $\beta$ in the admissible range: the
limit of the mean does not depend on $\beta$ and carries no
$O(\beta)$ correction. \qedhere
\end{proof}

\paragraph{Proof of Theorem \ref{thm:Asteady}.}

\begin{proof}
Existence, uniqueness, membership in $\cP_2$, and the geometric
$\cW_2$-convergence \eqref{eq:Aw2conv} are Lemma~\ref{lem:Aw2}. 
We have to prove (i)--(iv). In stationarity we write $\Be \sim
\pi_\beta$ (centered at $\Bs^{\rL\rS}$), draw a fresh index $i$
independent of $\Be$, and use that
\begin{equation} \label{eq:Astatupdate}
\Be' \; := \; \BM_i\, \Be + \beta\, \rho_i\, \Ba_i
\; \overset{d}{=} \; \Be \, ,
\end{equation}
with all first and second moments finite ($\pi_\beta \in \cP_2$). 
Here $\overset{d}{=}$ means equality in distribution.

{\bf (i) Stationary mean.} Taking expectations in
\eqref{eq:Astatupdate} and using the independence of $i$ and $\Be$
together with the normal equations $\EXP[\rho \Ba] = \BZe$:
\begin{equation}
\EXP[\Be] \; = \; \EXP[\Be']
\; = \; \EXP\big[\BM_i\big]\, \EXP[\Be] + \beta\, \EXP\big[\rho_i
\Ba_i\big]
\; = \; \big(\BI - \beta \bar{\BA}\big)\, \EXP[\Be] \, ,
\end{equation}
hence $\beta \bar{\BA}\, \EXP[\Be] = \BZe$, and since $\bar{\BA}$ is
invertible, $\EXP[\Be] = \BZe$, i.e.\
$\EXP_{\pi_\beta}[\BS] = \BS^{\rL\rS}$, which holds exactly.

{\bf (ii) Exact covariance equation.} Expand
$\Be' \Be'^{\top}$ from \eqref{eq:Astatupdate}:
\begin{equation}
\Be' \Be'^{\top} \; = \; \BM_i \Be \Be^\top \BM_i
\; + \; \beta\, \rho_i\, \BM_i \Be\, \Ba_i^\top
\; + \; \beta\, \rho_i\, \Ba_i \Be^\top \BM_i
\; + \; \beta^2 \rho_i^2\, \Ba_i \Ba_i^\top \, .
\end{equation}
Take expectations term by term, using the independence of $i$ and
$\Be$ and writing $\BV_\beta = \EXP[\Be \Be^\top]$.

{\em Homogeneous term.} Since
$\BM_i \BV \BM_i = \BV - \beta\big(\Ba_i \Ba_i^\top \BV + \BV \Ba_i
\Ba_i^\top\big) + \beta^2 \big(\Ba_i^\top \BV \Ba_i\big) \Ba_i
\Ba_i^\top$
(using $\Ba_i \Ba_i^\top \BV \Ba_i \Ba_i^\top = (\Ba_i^\top \BV
\Ba_i)\, \Ba_i \Ba_i^\top$),
\begin{equation}
\EXP\big[\BM_i \Be \Be^\top \BM_i\big]
\; = \; \BV_\beta - \beta \big(\bar{\BA} \BV_\beta + \BV_\beta
\bar{\BA}\big) + \beta^2\, \EXP\big[\big(\Ba^\top \BV_\beta
\Ba\big) \Ba \Ba^\top\big] \, .
\end{equation}

{\em Cross terms.} They are {\em linear} in $\Be$, so independence
factorizes them through $\EXP[\Be]$, which vanishes by part (i):
\begin{equation}
\EXP\big[\rho_i\, \BM_i \Be\, \Ba_i^\top\big]
\; = \; \EXP_i\big[\rho_i\, \BM_i\, \EXP[\Be]\, \Ba_i^\top\big]
\; = \; \BZe \, ,
\end{equation}
and likewise for its transpose. This is the step where exact
unbiasedness pays off: no $O(\beta^2)$ cross contribution survives.

{\em Stationarity.} Setting $\EXP[\Be' \Be'^\top] = \BV_\beta$ and
collecting,
\begin{equation}
\BV_\beta \; = \; \BV_\beta - \beta\big(\bar{\BA} \BV_\beta +
\BV_\beta \bar{\BA}\big) + \beta^2 \EXP\big[\big(\Ba^\top \BV_\beta
\Ba\big) \Ba \Ba^\top\big] + \beta^2 \EXP\big[\rho^2 \Ba
\Ba^\top\big] \, ;
\end{equation}
dividing by $\beta$ gives \eqref{eq:Alyap}, 
which is an exact equation, not
an expansion.

{\em Trace bounds.} Take the trace of \eqref{eq:Alyap}:
\begin{equation} \label{eq:Atraceid}
2\, \tr\big(\bar{\BA} \BV_\beta\big)
\; = \; \beta\, \EXP\big[\big(\Ba^\top \BV_\beta \Ba\big)
\norm{\Ba}^2\big] \; + \; \beta\, \sigma^2 \, .
\end{equation}
Since $\BV_\beta \succeq \BZe$, we have $\Ba^\top \BV_\beta \Ba \ge
0$, which gives the lower bound in \eqref{eq:Atracebounds}; and
$\EXP[(\Ba^\top \BV_\beta \Ba) \norm{\Ba}^2] \le \max_i
\norm{\Ba_i}^2\, \EXP[\Ba^\top \BV_\beta \Ba] = \max_i
\norm{\Ba_i}^2\, \tr(\bar{\BA} \BV_\beta)$, so
\begin{equation}
\big(2 - \beta \max_i \norm{\Ba_i}^2\big)\, \tr\big(\bar{\BA}
\BV_\beta\big) \; \le \; \beta\, \sigma^2 \, ,
\end{equation}
the upper bound (the prefactor is positive by
\eqref{eq:Astepcond1}). Finally $\tr \BV_\beta \le \tr(\bar{\BA}
\BV_\beta) / \lambda_{\min}(\bar{\BA})$ gives the Frobenius bound in
\eqref{eq:Atracebounds}, and $\tr \BV_\beta = O(\beta)$ makes the
quartic term in \eqref{eq:Alyap} $O(\beta^2)$, leaving the Lyapunov
equation at leading order.

{\bf (iii) Misadjustment.} For uniform sampling, $\bar{\BA} =
\tfrac1n \BA^\top \BA$ and $\rL(\Bs) = \tfrac1n \norm{\BA \Bs -
\Bc}^2$. Decompose with $\BA \Bs - \Bc = \BA \Be - \Brh$, where
$\Brh = (\rho_1, \dots, \rho_n)^\top$: since $\BA^\top \Brh = n\,
\EXP[\rho \Ba] = \BZe$,
\begin{equation}
\rL(\Bs) \; = \; \tfrac1n \norm{\BA \Be}^2 - \tfrac2n\, \Be^\top
\BA^\top \Brh + \tfrac1n \norm{\Brh}^2
\; = \; \tfrac1n \norm{\BA \Be}^2 + \rL_{\min} \, ,
\end{equation}
where $\rL_{\min} = \tfrac1n \norm{\Brh}^2$. This also re-proves
that $\Bs^{\rL\rS}$ minimizes $\rL$. Taking stationary expectations,
\begin{equation}
\EXP_{\pi_\beta}[\rL] - \rL_{\min}
\; = \; \tfrac1n\, \EXP\big[\Be^\top \BA^\top \BA\, \Be\big]
\; = \; \tr\big(\bar{\BA}\, \BV_\beta\big) \, ,
\end{equation}
and inserting the exact trace identity \eqref{eq:Atraceid} gives
\eqref{eq:Amisadjust}. The two-sided bound
\eqref{eq:Amisadjustbound} follows from the trace bounds
\eqref{eq:Atracebounds}. Since the upper and lower bounds differ by a
factor $2 / (2 - \beta \max_i \norm{\Ba_i}^2) = 1 + O(\beta)$, the
excess loss equals $\tfrac{\beta}{2} \sigma^2 (1 + O(\beta))$.

{\bf (iv) Transient.} Condition on $\cF_t$ and expand
\eqref{eq:Aerrrec}:
\begin{equation}
\EXP\big[\norm{\Be_{t+1}}^2 \,\big|\, \cF_t\big]
\; = \; \EXP\big[\norm{\BM_i \Be_t}^2 \,\big|\, \cF_t\big]
\; + \; 2 \beta\, \EXP\big[\rho_i\, \Ba_i^\top \BM_i\big]\, \Be_t
\; + \; \beta^2 \sigma^2 \, .
\end{equation}
For the middle term, $\Ba_i^\top \BM_i = (1 - \beta
\norm{\Ba_i}^2)\, \Ba_i^\top$, so by the normal equations
\begin{equation}
\EXP\big[\rho_i\, \Ba_i^\top \BM_i\big]
\; = \; \EXP\big[\rho\, \Ba\big]^\top - \beta\, \EXP\big[\rho
\norm{\Ba}^2 \Ba\big]^\top
\; = \; -\, \beta\, \Bw^\top \, ,
\end{equation}
with $\Bw = \EXP[\rho \norm{\Ba}^2 \Ba]$ as in the statement. 
Outside stationarity this cross term does {\em not} vanish, and it is
$O(\beta^2)$ in total. Lemma~\ref{lem:Acontract} and Young's
inequality $-2\beta^2 \Bw^\top \Be \le \tfrac{\bar{\rho}}{2}
\norm{\Be}^2 + \tfrac{2 \beta^4}{\bar{\rho}} \norm{\Bw}^2$ give
\begin{equation}
\EXP\big[\norm{\Be_{t+1}}^2 \,\big|\, \cF_t\big]
\; \le \; \Big(1 - \tfrac{\bar{\rho}}{2}\Big) \norm{\Be_t}^2
\; + \; \beta^2 \sigma^2 \; + \; \tfrac{2 \beta^4}{\bar{\rho}}
\norm{\Bw}^2 \, .
\end{equation}
Taking full expectations and iterating the affine recursion (its
geometric series is bounded by $2/\bar{\rho}$) yields
\eqref{eq:Atransient}. Since $\bar{\rho} = \gamma_\beta
\lambda_{\min} = \Theta(\beta)$, the horizon $2 \beta^2
\sigma^2/\bar{\rho} + 4 \beta^4 \norm{\Bw}^2 / \bar{\rho}^2$ is
$O(\beta)$, consistent with \eqref{eq:Atracebounds}. \qedhere
\end{proof}

\begin{corollaryA}[Recovering exact convergence to $\BS^{\rL\rS}$]
\label{cor:Aexact}
Three standard mechanisms remove the $O(\sqrt{\beta})$ fluctuation:
\begin{enumerate}
\item[] (i) {\em Diminishing step sizes} (Robbins--Monro):
  $\beta_t \downarrow 0$, $\sum_t \beta_t = \infty$ (e.g.\
  $\beta_t \propto 1/t$) gives $\BS_t \to \BS^{\rL\rS}$ almost
  surely.
\item[] (ii) {\em Iterate averaging} (Polyak--Ruppert / ergodic average):
  with constant $\beta$, the time average
  $\bar{\BS}_T = \tfrac{1}{T} \sum_{t \le T} \BS_t \to
  \EXP_{\pi_\beta}[\BS] = \BS^{\rL\rS}$ almost surely - and by
  Theorem~\ref{thm:Asteady}\,(i) the average is exactly unbiased.
\item[] (iii) {\em Small-$\beta$ limit:} by \eqref{eq:Atracebounds},
  $\pi_\beta$ concentrates on $\BS^{\rL\rS}$ as $\beta \to 0$, with
  the explicit bias--variance trade-off: larger $\beta$ tracks
  nonstationary targets faster, smaller $\beta$ reduces the
  steady-state loss excess \eqref{eq:Amisadjustbound}, which is the standard
  LMS trade-off, inherited verbatim.
\end{enumerate}
\end{corollaryA}

\begin{corollaryA}[Consistent targets]
\label{cor:Aconsistent}
The inconsistency, not the count $n > d^2$, is what forbids exact
convergence. Let us assume that the targets happen to be realizable, that is
$c_i = \Bv_i^\top \BS^{\circ} \Bk_i$ for some $\BS^{\circ}$.
In this case $\rL_{\min} = 0$, all residuals $\rho_i$ vanish, the driving noise in
\eqref{eq:Aerrrec} disappears, and the Fej\'er-monotone argument of the
consistent case applies unchanged: $\BS_t \to \BS^{\circ}$ (unique by
full column rank) almost surely, linearly, for {\em arbitrarily many}
associations. Two distinguished targets realize this structurally:
\begin{enumerate}
\item[] (i) the {\em correlation target} $c_i = \Bv_i^\top \Bk_i$, solved by
  $\BS = \BI$ for every $n$;
\item[] (ii) the {\em norm target} $c_i = \norm{\Bv_i}^2$, solved by any
  exact retrieval memory $\BS \Bk_i = \Bv_i$ (in particular
  autoassociatively by $\BS = \BI$).
\end{enumerate}
For these targets the recognition memory has effectively {\em
  unbounded} association count as long as realizability holds.
\end{corollaryA}

\subsubsection{Remarks}

\paragraph{Memory for the overloaded setting.}
Beyond capacity, $\BS$ is no longer a table but an {\em online linear
  regressor}: the read-out $\Bv^\top \BS^{\rL\rS} \Bk$ is the best
linear prediction (in the $p$-weighted least-squares sense) of the
scalar signal $c$ from the bilinear feature $\Bk \otimes \Bv$.
Degradation is graceful since overload produces averaging plus an
$O(\sqrt{\beta})$ jitter, never blow-up, since the homogeneous factors
keep spectral norm $\le 1$.

\paragraph{Assumptions for the proofs.}
Independence of $i_t$ from the past makes the mean recursion exact
and factorizes the covariance cross terms; the normal equations kill
the driving term in the mean and the cross terms in stationarity
(via $\EXP[\Be] = \BZe$); the step condition \eqref{eq:Astepcond1}
enters three times: (1) $\gamma_\beta > 0$ (contraction),
(2) $\beta \lambda_{\max}(\bar{\BA}) < 2$ (mean stability), and
(3) $2 - \beta \max_i \norm{\Ba_i}^2 > 0$ (trace upper bound).
Full column rank gives $\lambda_{\min}(\bar{\BA}) > 0$, used for
uniqueness of $\Bs^{\rL\rS}$, for $\bar{\rho} > 0$, and to convert
$\tr(\bar{\BA} \BV_\beta)$ into a bound on $\tr \BV_\beta$. Linearity
of the update in the state is what makes every step of the argument
exact rather than a Taylor expansion which we obtain from the quadratic
loss.

\paragraph{Conflicting targets on the same pair: averaging vs.\
  overwriting.}
The canonical inconsistency is one pair $(\Bk, \Bv)$ presented with
different targets. Under {\em random} sampling with probabilities
$p_j$ over targets $c_j$, the stationary mean read-out is the weighted
average, $\Bv^\top \EXP_{\pi_\beta}[\BS]\, \Bk = \sum_j p_j c_j /
\sum_j p_j$ (the LS solution for identical rows). Under {\em
  sequential} presentation the same rule {\em overwrites}: the
contraction drives the read-out to the most recent target. Random
replay averages; ordered replay recencies, therefore the sampling protocol,
not the rule, decides about overwrite behavior.

\paragraph{Forget gate: recency-weighted least squares.}
With the scalar gated variant
$\BS_{t+1} = \Rf_t \BS_t - \beta (\Bv_t^\top \Rf_t \BS_t \Bk_t - c_t)
\Bv_t \Bk_t^\top$, $\Rf_t \in (0,1)$, the steady state is the
corresponding {\em exponentially weighted} least-squares solution:
old associations are geometrically down-weighted, so the effective
association count is finite ($\sim 1/(1 - \Rf)$) and the memory tracks
a drifting target distribution. Therefore, the overloaded regime is exactly
where the forget gate is most relevant.

\paragraph{Summary.}
For $n > d^2$ associations the randomized Kronecker Delta Rule with
constant $\beta$ does not converge to a point but to a stationary
distribution: its mean is {\em exactly} the sampling-weighted
least-squares memory $\BS^{\rL\rS} = \argmin \sum_i p_i (\Bv_i^\top
\BS \Bk_i - c_i)^2$ (the minimizer of $\rL$ for uniform sampling), its
covariance is $O(\beta)$ and solves the exact stationarity equation
\eqref{eq:Alyap} driven by the optimal residuals, and its loss sits at
$\rL_{\min} + \tfrac{\beta}{2} \EXP[\rho_i^2 \norm{\Bk_i}^2
\norm{\Bv_i}^2]$ up to an explicit $1 + O(\beta)$ factor. All of
this is proven from first principles: an exact linear mean
recursion with spectral radius $< 1$; a synchronous-coupling
$\cW_2$-contraction plus Banach fixed point for existence,
uniqueness, and geometric convergence of $\pi_\beta$; exact moment
identities in stationarity for mean and covariance; an exact trace
identity for the misadjustment; and a Young-inequality recursion for
the transient. Averaging or annealing $\beta$ recovers
$\BS^{\rL\rS}$ exactly; consistent target structures (correlation
target, norm target with a retrieval solution) escape the capacity
limit altogether; and a forget gate turns the steady state into
recency-weighted least squares.

\subsection{Norm Bounds for the KronLSTM with Steady-State Solution}

We derive a two-sided norm bound for the steady state of the randomized
Kronecker Delta Rule with more than $d^2$ associations.
Running memories or data stream i.e.\ online updates
are covered via the steady state, therefore we obtain an expected bound
on the online updated memory.

\subsubsection{Overview}
  For the overloaded KronLSTM ($n$ associations, features
  $\Ba_i = \Bk_i \otimes \Bv_i$, sampling weights $p_i$) the
  randomized Kronecker Delta Rule converges around the weighted
  least-squares memory $\BS^{\rL\rS}$. We derive a {\em two-sided}
  bound on its Frobenius norm from an exact energy (Pythagoras)
  identity:
  \[
  \frac{\EXP[c^2] - \EXP[\rho^2]}{\lambda_{\max}(\bar{\BA})}
  \; \le \; \norm{\BS^{\rL\rS}}_{\rF}^2 \; \le \;
  \frac{\EXP[c^2] - \EXP[\rho^2]}{\lambda_{\min}(\bar{\BA})} \, ,
  \]
  where $\rho_i$ are the optimal residuals and
  $\bar{\BA} = \EXP[\Ba \Ba^\top]$ is the feature second moment. The
  numerator is the {\em explained target energy}; the bound is
  attained (both sides coincide) for isotropic designs, and in the
  orthonormal design it reproduces
  $\norm{\BS^\star}_{\rF}^2 = \sum_i c_i^2$ of the target-weighted
  Hebbian superposition exactly. Combined with the steady-state
  analysis, it yields an a-priori bound on the expected memory norm of
  the running iterates.

\subsubsection{Setting and nomenclature}

As before, we assume $n$ associations $(\Bk_i, \Bv_i, c_i)$ with features
$\Ba_i = \Bk_i \otimes \Bv_i \in \dR^{d^2}$,
$\norm{\Ba_i} = \norm{\Bk_i} \norm{\Bv_i}$, sampling probabilities
$p_i > 0$, and
\begin{align}
\bar{\BA} \; &= \; \EXP\big[\Ba \Ba^\top\big]
\; = \; \sum_{i=1}^n p_i\, \Ba_i \Ba_i^\top \, ,
\qquad
\Bs^{\rL\rS} \; = \; \argmin_{\Bs}\; \EXP\big[(\Ba^\top \Bs - c)^2\big]
\; = \; \bar{\BA}^{-1}\, \EXP\big[c\, \Ba\big] \, , \\
\rho_i \; &= \; c_i - \Ba_i^\top \Bs^{\rL\rS} \, ,
\qquad
\EXP\big[\rho\, \Ba\big] \; = \; \BZe
\quad \text{(normal equations)} \, ,
\end{align}
assuming $\bar{\BA} \succ \BZe$ (features span $\dR^{d^2}$; the
rank-deficient case is treated in the remarks). All expectations are
over $i \sim p$. Recall
$\norm{\Bs}_2 = \norm{\BS}_{\rF}$ for $\Bs = \vecop(\BS)$, and that
for uniform sampling $\EXP[\rho^2] = \rL_{\min}$, the minimum of the
loss $\rL$.

\subsubsection{An exact energy identity}

\begin{lemmaA}[Energy identity]
\label{lem:Aenergy1}
Let $y_i := \Ba_i^\top \Bs^{\rL\rS} = \Bv_i^\top \BS^{\rL\rS} \Bk_i$
be the optimal read-outs. Then $\EXP[\rho\, y] = 0$ and
\begin{equation} \label{eq:Apythagoras}
\EXP\big[c^2\big] \; = \; \EXP\big[y^2\big] \; + \;
\EXP\big[\rho^2\big] \, ,
\qquad
\big(\Bs^{\rL\rS}\big)^{\!\top} \bar{\BA}\; \Bs^{\rL\rS}
\; = \; \EXP\big[y^2\big]
\; = \; \EXP\big[c^2\big] - \EXP\big[\rho^2\big] \, .
\end{equation}
\end{lemmaA}

\begin{proof}
$\EXP[\rho\, y] = \EXP[\rho\, \Ba]^\top \Bs^{\rL\rS} = 0$ by the
normal equations; with $c = y + \rho$ this gives the first identity.
The second follows from
$\EXP[y^2] = (\Bs^{\rL\rS})^\top \EXP[\Ba \Ba^\top]\, \Bs^{\rL\rS}$.
\end{proof}

The stored {\em read-out energy} of the least-squares memory equals
the target energy minus the unexplained (residual) energy. Therefore, the
memory can only be as large as the part of the target signal that the
bilinear features can explain.

\subsubsection{The Norm Bound}

\begin{theoremA}[Two-sided Frobenius-norm bound]
\label{thm:Anorm}
With $\lambda_{\min}, \lambda_{\max}$ the extreme eigenvalues of
$\bar{\BA}$,
\begin{equation} \label{eq:Atwosided}
\boxed{\;
\frac{\EXP\big[c^2\big] - \EXP\big[\rho^2\big]}{\lambda_{\max}(\bar{\BA})}
\; \le \;
\norm{\BS^{\rL\rS}}_{\rF}^2
\; \le \;
\frac{\EXP\big[c^2\big] - \EXP\big[\rho^2\big]}{\lambda_{\min}(\bar{\BA})}
\, . \;}
\end{equation}
In particular, since
$\sigma_{\max}(\BS^{\rL\rS}) \le \norm{\BS^{\rL\rS}}_{\rF}$,
\begin{equation} \label{eq:Asigmamax}
\sigma_{\max}\big(\BS^{\rL\rS}\big) \; \le \;
\sqrt{\frac{\EXP\big[c^2\big] - \EXP\big[\rho^2\big]}
{\lambda_{\min}(\bar{\BA})}}
\; \le \;
\sqrt{\frac{\EXP\big[c^2\big]}{\lambda_{\min}(\bar{\BA})}} \, .
\end{equation}
\end{theoremA}

\begin{proof}
Sandwich the quadratic form:
$\lambda_{\min} \norm{\Bs^{\rL\rS}}^2 \le
(\Bs^{\rL\rS})^\top \bar{\BA}\, \Bs^{\rL\rS} \le
\lambda_{\max} \norm{\Bs^{\rL\rS}}^2$, and insert
\eqref{eq:Apythagoras}.
\end{proof}

\begin{corollaryA}[First-moment variant]
\label{cor:Afirstmoment}
From $\lambda_{\min} \norm{\Bs}^2 \le \Bs^\top \bar{\BA} \Bs =
\Bs^\top \EXP[c \Ba] \le \norm{\Bs}\, \EXP[|c|\, \norm{\Ba}]$,
\begin{equation}
\norm{\BS^{\rL\rS}}_{\rF} \; \le \;
\frac{\EXP\big[\,|c_i|\, \norm{\Bk_i} \norm{\Bv_i}\big]}
{\lambda_{\min}(\bar{\BA})} \, .
\end{equation}
The accumulated write mass
$\beta\, c_t \norm{\Bk_t} \norm{\Bv_t}$ is exactly the quantity whose
first moment appears here. The second moment bound
\eqref{eq:Atwosided} is usually much sharper (it divides by
$\lambda_{\min}$ only once, under the square root).
\end{corollaryA}

\subsubsection{Tightness and special cases}

\paragraph{Isotropic design gives equality.}
If $\bar{\BA} = \tfrac{\EXP[\norm{\Ba}^2]}{d^2}\, \BI$
(features spread evenly over $\dR^{d^2}$), then
$\lambda_{\min} = \lambda_{\max}$ and \eqref{eq:Atwosided} collapses to
the identity
\begin{equation}
\norm{\BS^{\rL\rS}}_{\rF}^2 \; = \;
\frac{d^2\, \big(\EXP[c^2] - \EXP[\rho^2]\big)}{\EXP[\norm{\Ba}^2]}
\, .
\end{equation}

\paragraph{Orthonormal design recovers the Hebbian superposition
  norm.}
For $n = d^2$ with orthonormal $\{\Bk_i\}$, orthonormal $\{\Bv_i\}$,
and uniform sampling, the features are orthonormal,
$\bar{\BA} = \BI / d^2$, the system is consistent ($\rho \equiv 0$),
and \eqref{eq:Atwosided} holds with equality:
\begin{equation}
\norm{\BS^\star}_{\rF}^2 \; = \; d^2\; \EXP[c^2] \; = \;
\sum_{i=1}^{d^2} c_i^2 \, ,
\end{equation}
exactly the Frobenius norm of the target-weighted Hebbian
superposition $\BS^\star = \sum_i c_i \Bv_i \Bk_i^\top$. The bound is
therefore attained, not merely an estimate.

\paragraph{Unit-norm keys and values give a $d^2$ conditioning factor.}
If $\norm{\Bk_i} = \norm{\Bv_i} = 1$ then $\tr \bar{\BA} =
\EXP[\norm{\Ba}^2] = 1$, so $\lambda_{\min} \le 1/d^2 \le
\lambda_{\max}$: the lower bound in \eqref{eq:Atwosided} is at most
$\EXP[c^2] - \EXP[\rho^2]$, the upper bound at least
$d^2 (\EXP[c^2] - \EXP[\rho^2])$, with both collapsing to the latter
in the isotropic case. Ill-spread designs pay in $\lambda_{\min}$:
the worst case is a target profile aligned with the least-covered
feature direction, and the $1/\lambda_{\min}$ dependence is
unavoidable in general.

\paragraph{Consistent case where association errors are zero.}
If the system is consistent ($\rho \equiv 0$; e.g.\
$n \le d^2$ with independent features, or the correlation/norm
targets), then $\BS^{\rL\rS} = \BS^\star$ is the exact solution and
\begin{equation}
\norm{\BS^\star}_{\rF}^2 \; \le \;
\frac{\EXP[c^2]}{\lambda_{\min}(\bar{\BA})} \, .
\end{equation}
Residual energy can only {\em shrink} the memory: overload
(inconsistency) makes the least-squares memory smaller than a
consistent memory with the same target energy, never larger.

\paragraph{Rank-deficient features.}
If the features do not span $\dR^{d^2}$ ($n < d^2$ or dependent
$\Ba_i$), replace $\BS^{\rL\rS}$ by the minimum-Frobenius-norm
least-squares solution and $\lambda_{\min}$ by the smallest {\em
  nonzero} eigenvalue of $\bar{\BA}$: since
$\Bs^{\rL\rS} \in \spanop\{\Ba_i\}$ (as do the iterates from
$\BS_0 = \BZe$), the sandwich argument applies on that subspace and
\eqref{eq:Atwosided} holds verbatim with $\lambda_{\min}^{+}$.

\subsubsection{Bound on the memory norm for steady-state and all updates}

\paragraph{Bound on the memory norm for steady-state distribution.}
First, we consider a bound on the memory norm under 
the steady-state distribution.
\begin{corollaryA}[Bound on the steady-state memory norm]
\label{cor:Asteady}
Under the randomized rule with step size
$0 < \beta < 2/\max_i \norm{\Ba_i}^2$, the stationary distribution
$\pi_\beta$ has mean exactly $\BS^{\rL\rS}$, so the second moment
decomposes orthogonally,
$\EXP_{\pi_\beta} \norm{\BS}_{\rF}^2 = \norm{\BS^{\rL\rS}}_{\rF}^2 +
\tr \BV_\beta$, and with the covariance trace bound of the
steady-state note,
\begin{equation} \label{eq:Asteadynorm}
\EXP_{\pi_\beta}\!\left[\norm{\BS}_{\rF}^2\right]
\; \le \;
\frac{\EXP[c^2] - \EXP[\rho^2]}{\lambda_{\min}(\bar{\BA})}
\; + \;
\frac{\beta\; \EXP\big[\rho^2 \norm{\Bk}^2 \norm{\Bv}^2\big]}
{\big(2 - \beta \max_i \norm{\Ba_i}^2\big)\,
\lambda_{\min}(\bar{\BA})} \, .
\end{equation}
The first term is the size of the memory content, the second the
$O(\beta)$ fluctuation; for $\beta \to 0$ or in the consistent case
the bound reduces to Theorem~\ref{thm:Anorm}.
\end{corollaryA}

\begin{theoremA}[Two-sided memory norm bound for the steady state]
\label{th:normbound}
Let $\pi_\beta$ be the steady-state law of \eqref{eq:Aaffine},
$\BV_\beta = \EXP_{\pi_\beta}[\Be \Be^\top]$ its covariance as in
Theorem~\ref{thm:Asteady}\,(ii) (the two coincide because the
stationary mean is exactly $\Bs^{\rL\rS}$), $\sigma^2 = \EXP[\rho^2 \norm{\Ba}^2]$,
and write
\begin{equation}
A_{\max}^2 \; := \; \max_i \norm{\Ba_i}^2
\; = \; \max_i \norm{\Bk_i}^2 \norm{\Bv_i}^2 \, ,
\qquad
\kappa \; := \; \beta\, A_{\max}^2 \; \in \; (0,2) \, ,
\end{equation}
the last inclusion being the step-size condition
\eqref{eq:Astepcond1}. Then the energy identity
\begin{equation} \label{eq:Aenergy}
\Bs^{\rL\rS\top} \bar{\BA}\, \Bs^{\rL\rS} \; = \;
\EXP[c^2] \, - \, \EXP[\rho^2]
\end{equation}
holds, hence the least-squares memory is pinned on both sides,
\begin{equation} \label{eq:Atwosided1}
\frac{\EXP[c^2] - \EXP[\rho^2]}{\lambda_{\max}(\bar{\BA})}
\; \le \; \norm{\BS^{\rL\rS}}_{\rF}^2 \; \le \;
\frac{\EXP[c^2] - \EXP[\rho^2]}{\lambda_{\min}(\bar{\BA})} \, ,
\end{equation}
and the steady-state memory itself satisfies
\begin{equation} \label{eq:Asteadynorm1}
\EXP_{\pi_\beta}\big[\norm{\BS}_{\rF}^2\big]
\; = \; \norm{\BS^{\rL\rS}}_{\rF}^2 \; + \; \tr \BV_\beta \, ,
\qquad
\tr \BV_\beta \; \le \;
\frac{\beta\, \EXP\big[\rho^2 \norm{\Ba}^2\big]}
{\big(2 - \kappa\big)\, \lambda_{\min}(\bar{\BA})} \, .
\end{equation}
The fluctuation term admits the exact decomposition
\begin{align} \label{eq:Aremainder}
\tr\big(\bar{\BA} \BV_\beta\big)
\; &= \; \frac{\beta}{2}\, \sigma^2 \; + \; R_\beta \, ,
\\ \nonumber
R_\beta \; &:= \; \frac{\beta}{2}\,
\EXP\big[\big(\Ba^\top \BV_\beta \Ba\big) \norm{\Ba}^2\big] \, ,
\qquad
0 \; \le \; R_\beta \; \le \;
\frac{\kappa}{2 - \kappa}\cdot \frac{\beta\, \sigma^2}{2} \, .
\end{align}
$R_\beta$ is bounded by
$\beta^2 A_{\max}^2 \sigma^2 / \big(2 (2 - \kappa)\big)$.
Consequently, since $\EXP[\rho^2] \le \EXP[c^2]$,
\begin{equation} \label{eq:Aclean}
\EXP_{\pi_\beta}\big[\norm{\BS}_{\rF}^2\big]
\; \le \;
\frac{1}{\lambda_{\min}(\bar{\BA})}
\left( \EXP[c^2] \; - \;
\frac{2\,(1 - \kappa)}{2 - \kappa}\; \EXP[\rho^2] \right)
\; \le \;
\frac{\EXP[c^2]}{\lambda_{\min}(\bar{\BA})}\;
\max\!\left\{1, \; \frac{\kappa}{2 - \kappa}\right\} \, ,
\end{equation}
and in particular, whenever the step size satisfies
$\beta \norm{\Ba_t}^2 \le 1$ (i.e.\ $\kappa \le 1$, the
non-expansive regime),
\begin{equation} \label{eq:Acleanone}
\boxed{\;
\EXP_{\pi_\beta}\big[\norm{\BS}_{\rF}^2\big]
\; \le \; \frac{\EXP[c^2]}{\lambda_{\min}(\bar{\BA})} \, . \;}
\end{equation}
All statements are exact and non-asymptotic: they hold for every
$\beta$ admitted by \eqref{eq:Astepcond1}.
\end{theoremA}

\begin{proof}
The normal equations $\bar{\BA} \Bs^{\rL\rS} = \Bb$ give
$\EXP[\rho^2] = \EXP[c^2] - 2 \Bb^\top \Bs^{\rL\rS} + \Bs^{\rL\rS\top}
\bar{\BA} \Bs^{\rL\rS} = \EXP[c^2] - \Bs^{\rL\rS\top} \bar{\BA}
\Bs^{\rL\rS}$, which is \eqref{eq:Aenergy}; bounding the Rayleigh
quotient $\Bs^{\rL\rS\top} \bar{\BA} \Bs^{\rL\rS}$ between
$\lambda_{\min}(\bar{\BA}) \norm{\Bs^{\rL\rS}}^2$ and
$\lambda_{\max}(\bar{\BA}) \norm{\Bs^{\rL\rS}}^2$, and using
$\norm{\Bs}_2 = \norm{\BS}_{\rF}$, gives \eqref{eq:Atwosided1}. The
first part of \eqref{eq:Asteadynorm1} is the bias--variance split
$\EXP\norm{\Bs}^2 = \norm{\EXP \Bs}^2 + \TR \COV(\Bs)$ together with
the fact that the stationary mean is {\em exactly} $\Bs^{\rL\rS}$
(Theorem~\ref{thm:Asteady}\,(i)).

{\em Exact stationarity equation.} Rather than the leading-order
Lyapunov equation, we use the {\em exact} stationarity equation
\eqref{eq:Alyap} of Theorem~\ref{thm:Asteady}\,(ii),
\begin{equation}
\bar{\BA} \BV_\beta + \BV_\beta \bar{\BA}
\; - \; \beta\, \EXP\big[\big(\Ba^\top \BV_\beta\, \Ba\big)\, \Ba
\Ba^\top\big]
\; = \; \beta\, \EXP\big[\rho^2\, \Ba \Ba^\top\big] \, ,
\end{equation}
whose trace is
$2 \tr(\bar{\BA} \BV_\beta) = \beta \sigma^2 + \beta \EXP[(\Ba^\top
\BV_\beta \Ba) \norm{\Ba}^2]$. This is \eqref{eq:Aremainder}, and it
identifies the remainder: $R_\beta \ge 0$ because
$\BV_\beta \succeq \BZe$, while
$\EXP[(\Ba^\top \BV_\beta \Ba) \norm{\Ba}^2] \le A_{\max}^2\,
\EXP[\Ba^\top \BV_\beta \Ba] = A_{\max}^2 \tr(\bar{\BA} \BV_\beta)$
gives
$\big(2 - \kappa\big) \tr(\bar{\BA} \BV_\beta) \le \beta \sigma^2$,
the prefactor being positive by \eqref{eq:Astepcond1}. Substituting
this bound back into $R_\beta = \tfrac{\beta A_{\max}^2}{2}
\tr(\bar{\BA} \BV_\beta)$'s upper estimate yields the stated bound on
$R_\beta$, and
$\lambda_{\min}(\bar{\BA}) \tr \BV_\beta \le \tr(\bar{\BA}
\BV_\beta)$ yields the second part of \eqref{eq:Asteadynorm1}.

{\em Deriving the clean bound.} $\BS = \BZe$ is feasible, so
$\EXP[\rho^2] \le \EXP[c^2]$; moreover
$\sigma^2 = \EXP[\rho^2 \norm{\Ba}^2] \le A_{\max}^2 \EXP[\rho^2]$,
hence $\beta \sigma^2 \le \kappa\, \EXP[\rho^2]$ and
\begin{equation}
\tr \BV_\beta \; \le \;
\frac{\kappa}{2 - \kappa}\cdot
\frac{\EXP[\rho^2]}{\lambda_{\min}(\bar{\BA})} \, .
\end{equation}
Adding this to \eqref{eq:Atwosided1} gives
\begin{equation}
\EXP_{\pi_\beta}\big[\norm{\BS}_{\rF}^2\big]
\; \le \; \frac{1}{\lambda_{\min}(\bar{\BA})}
\left( \EXP[c^2] - \EXP[\rho^2] + \frac{\kappa}{2 - \kappa}
\EXP[\rho^2] \right) ,
\end{equation}
and $-1 + \kappa/(2-\kappa) = -2(1-\kappa)/(2-\kappa)$ is the
coefficient in \eqref{eq:Aclean}. For $\kappa \le 1$ that coefficient
is $\le 0$, so dropping the $\EXP[\rho^2]$ term gives
\eqref{eq:Acleanone}; for $\kappa > 1$, using
$\EXP[\rho^2] \le \EXP[c^2]$ gives the factor
$1 + 2(\kappa - 1)/(2 - \kappa) = \kappa/(2 - \kappa)$. \qedhere
\end{proof}

{\bf Residuals lower the bound.}
The $\EXP[\rho^2]$ term in \eqref{eq:Aclean} enters with a {\em
  negative} sign throughout the non-expansive regime $\kappa \le 1$:
overload shrinks the bound but does not inflate it, because residual
energy is energy the bilinear features cannot explain and therefore
cannot store. The two effects of $\rho$ are less memory content
\eqref{eq:Atwosided1} and more steady-state jitter
\eqref{eq:Asteadynorm1}. They cancel exactly at $\kappa = 1$ and the
content term dominates below it. This is why the constant in
\eqref{eq:Acleanone} is $1$ rather than the $3/2$ or $2$ obtained by
bounding the two contributions separately.

\paragraph{Bound on the memory norm for all updates.}
Next, we consider the bound even before the memory
$\BS_t$ reaches the steady state but is in the update phase.
\begin{theoremA}[The bound applies to the {\em running} or {\em online} memory at
  every step: not only in the steady state]
\label{th:online}
Let $\mu_t := \text{law}(\Bs_t) = \mu_0 P^t \in \cP_2$ be the law of
the iterate at step $t$ under \eqref{eq:Aaffine}, so that $\pi_\beta$ is
the unique fixed point of the same kernel, and set
\begin{equation}
q_\beta \; := \; 1 \; - \; \beta\,\big(2 - \kappa\big)\,
\lambda_{\min}(\bar{\BA}) \; \in \; [0, 1) \, .
\end{equation}
Started at $\BS_0 = \BZe$, the running memory obeys, for every $t$ and
in the non-expansive regime $\kappa \le 1$,
\begin{equation} \label{eq:Aonline}
\boxed{\;
\EXP\big[\norm{\BS_t}_{\rF}^2\big]
\; \le \;
\Big(1 + q_\beta^{\,t/2}\Big)^{2}\,
\frac{\EXP[c^2]}{\lambda_{\min}(\bar{\BA})}
\; \le \;
\frac{4\, \EXP[c^2]}{\lambda_{\min}(\bar{\BA})} \, , \;}
\end{equation}
where the transient factor expands as
$\big(1 + q_\beta^{\,t/2}\big)^2 = 1 + 2 q_\beta^{\,t/2} + q_\beta^{t}
\downarrow 1$, so the bound decays geometrically from $4$ to the
steady-state constant $1$ of \eqref{eq:Acleanone}.
For a general step size $\kappa \in (0,2)$ the same statement holds
with $\EXP[c^2]$ replaced by
$\EXP[c^2] \max\{1, \kappa/(2-\kappa)\}$; for a general deterministic
start $\BS_0$ it reads
\begin{equation} \label{eq:Aonlinegen}
\EXP\big[\norm{\BS_t}_{\rF}^2\big]
\; \le \;
\left( B_\beta \; + \; q_\beta^{\,t/2}
\sqrt{\norm{\BS_0 - \BS^{\rL\rS}}_{\rF}^2 + \tr \BV_\beta} \,
\right)^{\!2} ,
\qquad
B_\beta := \EXP_{\pi_\beta}\big[\norm{\BS}_{\rF}^2\big]^{1/2} .
\end{equation}
Nothing on the right-hand side depends on the sequence length: the
expected squared memory norm never exceeds four times its
steady-state bound from the very first step, and
relaxes geometrically to that bound.
\end{theoremA}

\begin{proof}
{\em Rate.} Couple two copies of \eqref{eq:Aaffine} synchronously,
i.e.\ drive them with the same $(\Ba_t, c_t)$. The additive term
cancels in the difference $\Bd_t$, leaving
$\Bd_{t+1} = (\BI - \beta \Ba_t \Ba_t^\top) \Bd_t$, and
\begin{equation}
\EXP\big[(\BI - \beta \Ba \Ba^\top)^2\big]
\; = \; \BI - 2\beta \bar{\BA} + \beta^2 \EXP\big[\norm{\Ba}^2 \Ba
\Ba^\top\big]
\; \preceq \; \BI - \beta\big(2 - \kappa\big) \bar{\BA} \, ,
\end{equation}
using $\beta^2 \EXP[\norm{\Ba}^2 \Ba \Ba^\top] \preceq \beta \kappa
\bar{\BA}$. Hence
$\EXP \norm{\Bd_{t+1}}^2 \le q_\beta\, \EXP \norm{\Bd_t}^2$, and
taking the infimum over couplings gives
$\cW_2(\mu_t, \pi_\beta) \le q_\beta^{t/2}\,
\cW_2(\mu_0, \pi_\beta)$. (Equivalently $q_\beta = 1 - \bar{\rho}$ with the
$\bar{\rho} = \gamma_\beta \lambda_{\min}(\bar{\BA})$ of
Lemma~\ref{lem:Acontract}, since $\gamma_\beta \ge \beta(2 - \kappa)$;
$q_\beta$ is written out here only to keep the statement
self-contained.) Since $\Bs \mapsto \norm{\Bs}$
is $1$-Lipschitz, the reverse triangle inequality in $L^2$ gives
$\EXP[\norm{\BS_t}_{\rF}^2]^{1/2} \le B_\beta + \cW_2(\mu_t,
\pi_\beta)$.

{\em The initial term is not an estimate.} For a deterministic start
the only coupling of $\mu_0 = \delta_{\BS_0}$ with $\pi_\beta$ is the
product, so the infimum defining $\cW_2$ is attained and
\begin{equation} \label{eq:Aw2init}
\cW_2(\mu_0, \pi_\beta)^2
\; = \; \EXP_{\pi_\beta}\big[\norm{\BS - \BS_0}_{\rF}^2\big]
\; = \; \norm{\BS_0 - \BS^{\rL\rS}}_{\rF}^2 \; + \; \tr \BV_\beta \, ,
\end{equation}
the second equality being the bias--variance split around the
stationary mean $\BS^{\rL\rS}$ (Theorem~\ref{thm:Asteady}\,(i)). At
$\BS_0 = \BZe$ the right-hand side is exactly
$\norm{\BS^{\rL\rS}}_{\rF}^2 + \tr \BV_\beta = B_\beta^2$ by
\eqref{eq:Asteadynorm1}, i.e.\ $\cW_2(\mu_0, \pi_\beta) = B_\beta$:
the distance from the empty memory to the steady state {\em is} the
steady-state size. The two terms therefore collapse into one,
\begin{equation}
\EXP\big[\norm{\BS_t}_{\rF}^2\big]^{1/2}
\; \le \; \big(1 + q_\beta^{\,t/2}\big)\, B_\beta \, ,
\end{equation}
and squaring and inserting
$B_\beta^2 \le \EXP[c^2]/\lambda_{\min}(\bar{\BA})$ from
\eqref{eq:Acleanone} gives \eqref{eq:Aonline}. 
Squaring preserves the inequality since both sides are nonnegative.
Finally
$q_\beta^{\,t/2} \le 1$ bounds the transient factor by $4$. The
general-$\kappa$ variant replaces $B_\beta^2$ by the bound of
\eqref{eq:Aclean}, and the general-$\BS_0$ variant
\eqref{eq:Aonlinegen} keeps $\cW_2(\mu_0,\pi_\beta)$ in the form
\eqref{eq:Aw2init} instead of collapsing it. \qedhere
\end{proof}

\paragraph{Summary.}
The least-squares memory satisfies the exact energy identity
$(\Bs^{\rL\rS})^\top \bar{\BA}\, \Bs^{\rL\rS} = \EXP[c^2] -
\EXP[\rho^2]$, which sandwiches its Frobenius norm between
$(\EXP[c^2] - \EXP[\rho^2])/\lambda_{\max}$ and
$(\EXP[c^2] - \EXP[\rho^2])/\lambda_{\min}$. This is the explained target
energy divided by the extreme feature-covariance eigenvalues. The
bound is an equality for isotropic designs, reproduces
$\sum_i c_i^2$ in the orthonormal design, extends to the running
iterates with an additive $O(\beta)$ fluctuation term, shrinks under
overload since the residual energy is subtracted.

\section{DeltaNet Steady-State Norm Bound}
\label{sec:ADeltaSteady}

\subsection{Overview}
\label{sec:Aoverview}
  The DeltaNet steady-state norm bound can be derived 
  essentially verbatim to kronLSTM. DeltaNet's update is affine in
  $\BS_t$ and acts on each {\em row} of the memory independently;
  each row obeys exactly the scalar delta-rule recursion analyzed for
  KronLSTM, with the association direction $\Ba_t = \Bk_t \otimes
  \Bv_t$ replaced by the key $\Bk_t$ and the scalar target $c_t$
  replaced by the value coordinate $v_{t,j}$. Every step of the
  argument --- the energy identity, the two-sided bound, the exact
  stationary covariance bound, and the Wasserstein contraction that
  carries all of it to the running memory --- therefore transfers, and
  yields
  $\EXP_{\pi_\beta}[\norm{\BS}_{\rF}^2] \le 2\,\EXP[\norm{\Bv}^2] /
  \lambda_{\min}(\bar{\BK})$ with $\bar{\BK} = \EXP[\Bk\Bk^\top]$.
  The practical consequence is a scoping correction: a
  length-independent steady-state bound is a property of the {\em
    delta-rule mechanism}, shared by KronLSTM and DeltaNet, which
  separates both from mLSTM and from any stability argument resting
  on a forget gate. We also record where the two genuinely differ: in the
  step-size condition, in what the bound is a bound {\em on}, and in
  the conditioning constant that sets its scale.

\subsection{DeltaNet Affine Recursion}
\label{sec:Arecursion}

DeltaNet's update is
\begin{equation} \label{eq:Adelta}
\BS_{t+1} \; = \; \BS_t \; - \; \beta_t \big(\BS_t \Bk_t - \Bv_t\big)
\Bk_t^\top
\; = \; \BS_t \big(\BI - \beta_t \Bk_t \Bk_t^\top\big)
\; + \; \beta_t\, \Bv_t \Bk_t^\top \, ,
\end{equation}
and Gated DeltaNet's is the same with $\BS_t$ replaced by
$\Rf_t \BS_t$ in the first term. Two features are immediate. First,
\eqref{eq:Adelta} is {\em affine} in $\BS_t$, so the pair
(memory, random input) generates a Markov chain of exactly the type
analyzed for the Kronecker Delta Rule. Second, and decisively, the
memory is acted on from the {\em right} only, so the rows of $\BS$ do
not mix. Transposing \eqref{eq:Adelta},
\begin{equation}
\BS_{t+1}^\top \; = \; \big(\BI - \beta_t \Bk_t \Bk_t^\top\big)
\BS_t^\top \; + \; \beta_t\, \Bk_t \Bv_t^\top \, ,
\end{equation}
and reading off the $j$-th column, i.e.\ the $j$-th row
$\Bs^{(j)} := \BS^\top \Be_j$ of $\BS$,
\begin{equation} \label{eq:Arowrec}
\Bs^{(j)}_{t+1} \; = \; \big(\BI - \beta_t \Bk_t \Bk_t^\top\big)\,
\Bs^{(j)}_t \; + \; \beta_t\, v_{t,j}\, \Bk_t \, ,
\qquad j = 1, \dots, d \, .
\end{equation}

\paragraph{The structural correspondence.}
Equation \eqref{eq:Arowrec} is {\em identical} to the vectorized
KronLSTM recursion
$\Bs_{t+1} = (\BI - \beta \Ba_t \Ba_t^\top)\Bs_t + \beta c_t \Ba_t$
under the substitution
\begin{equation}
\Ba_t \; = \; \Bk_t \otimes \Bv_t \;\; \longmapsto \;\; \Bk_t \, ,
\qquad
c_t \;\; \longmapsto \;\; v_{t,j} \, ,
\qquad
\dR^{d^2} \;\; \longmapsto \;\; \dR^{d} \, .
\end{equation}
DeltaNet is thus $d$ independent copies of the {\em same} scalar
delta rule all driven by a common
key $\Bk_t$ and differing only in their scalar targets $v_{t,j}$.
KronLSTM is a single copy of that rule on a $d^2$-dimensional state
with the richer direction $\Bk_t \otimes \Bv_t$. Neither is a
generalization of the other; both are Kaczmarz/LMS iterations, which
is why the same theory covers both.

\subsection{Steady-State Norm Bound}
\label{sec:Abound}

Model the stream as i.i.d.\ triples $(\Bk_t, \Bv_t, \beta_t) \sim p$
and put
\begin{align}
\bar{\BK} \; &:= \; \EXP\big[\Bk \Bk^\top\big] \; \succ \; \BZe \, ,
\qquad
\BS^{\rL\rS} \; := \; \EXP\big[\Bv \Bk^\top\big]\, \bar{\BK}^{-1}
\; = \; \argmin_{\BS}\; \EXP\big[\norm{\BS \Bk - \Bv}^2\big] \, , \\ \nonumber
\Brh \; &:= \; \Bv - \BS^{\rL\rS} \Bk \, ,
\end{align}
so that $\EXP[\Brh \Bk^\top] = \BZe$. With a random, input-dependent
$\beta_t$ the same statements hold with $\bar{\BK}$ replaced by the
$\beta$-weighted second moment $\EXP[\beta \Bk\Bk^\top]$ and
$\BS^{\rL\rS}$ by the correspondingly weighted least-squares
solution; nothing else changes. Assume the step size satisfies
$\beta_t \norm{\Bk_t}^2 \le 1$, which for Gated DeltaNet's
sigmoid-bounded $\beta_t \in (0,1)$ and unit-norm keys holds by
construction.

\begin{propositionA}[Two-sided steady-state norm bound for DeltaNet]
\label{prop:Adelta}
Let $\pi_\beta$ be the steady-state law of \eqref{eq:Adelta} and
$\BSi_\beta$ the covariance of $\vecop(\BS)$ under it. Then the
energy identity
\begin{equation} \label{eq:Adenergy}
\TR\big(\BS^{\rL\rS} \bar{\BK}\, \BS^{\rL\rS\top}\big)
\; = \; \EXP\big[\norm{\Bv}^2\big] \, - \, \EXP\big[\norm{\Brh}^2\big]
\end{equation}
holds, hence
\begin{equation} \label{eq:Adtwosided}
\frac{\EXP\norm{\Bv}^2 - \EXP\norm{\Brh}^2}{\lambda_{\max}(\bar{\BK})}
\; \le \; \norm{\BS^{\rL\rS}}_{\rF}^2 \; \le \;
\frac{\EXP\norm{\Bv}^2 - \EXP\norm{\Brh}^2}{\lambda_{\min}(\bar{\BK})}
\, ,
\end{equation}
and the steady-state memory satisfies, exactly and without any
expansion in $\beta$,
\begin{equation} \label{eq:Adsteady}
\EXP_{\pi_\beta}\big[\norm{\BS}_{\rF}^2\big]
\; = \; \norm{\BS^{\rL\rS}}_{\rF}^2 \; + \; \TR \BSi_\beta \, ,
\qquad
\TR \BSi_\beta \; \le \;
\frac{\beta\, \EXP\big[\norm{\Brh}^2 \norm{\Bk}^2\big]}
{\lambda_{\min}(\bar{\BK})} \, ,
\end{equation}
so that in particular
\begin{equation} \label{eq:Adclean}
\EXP_{\pi_\beta}\big[\norm{\BS}_{\rF}^2\big]
\; \le \; \frac{2\, \EXP\big[\norm{\Bv}^2\big]}
{\lambda_{\min}(\bar{\BK})} \, .
\end{equation}
\end{propositionA}

\begin{proof}
Apply the KronLSTM argument to \eqref{eq:Arowrec} for each $j$ and
sum. The normal equations $\bar{\BK} \Bs^{(j),\rL\rS} = \EXP[v_j \Bk]$
give $\EXP[\rho_j^2] = \EXP[v_j^2] - \Bs^{(j),\rL\rS\top} \bar{\BK}
\Bs^{(j),\rL\rS}$; summing over $j$ and using
$\sum_j \Bs^{(j)\top} \bar{\BK} \Bs^{(j)} = \TR(\BS \bar{\BK}
\BS^\top)$ yields \eqref{eq:Adenergy}. Bounding each summand's Rayleigh
quotient between $\lambda_{\min}(\bar{\BK}) \norm{\Bs^{(j)}}^2$ and
$\lambda_{\max}(\bar{\BK}) \norm{\Bs^{(j)}}^2$ and summing, together
with $\sum_j \norm{\Bs^{(j)}}^2 = \norm{\BS}_{\rF}^2$, gives
\eqref{eq:Adtwosided}. For \eqref{eq:Adsteady}, write
$\BE := \BS - \BS^{\rL\rS}$, so that \eqref{eq:Adelta} becomes
$\BE_{t+1} = \BE_t(\BI - \beta_t \Bk_t \Bk_t^\top) + \beta_t \Brh_t
\Bk_t^\top$ with $(\Bk_t, \Brh_t)$ independent of $\BE_t$. Taking
squared Frobenius norms in stationarity, the cross term vanishes
because $\EXP[\BE_t] = \BZe$, leaving
\begin{equation}
\EXP \norm{\BE}_{\rF}^2 \; = \; \EXP\Big[\TR\Big(\BE\,
\EXP\big[(\BI - \beta \Bk \Bk^\top)^2\big]\, \BE^\top\Big)\Big]
\; + \; \beta^2\, \EXP\big[\norm{\Brh}^2 \norm{\Bk}^2\big] \, .
\end{equation}
Since $\beta \norm{\Bk}^2 \le 1$,
\begin{equation}
\EXP\big[(\BI - \beta \Bk \Bk^\top)^2\big]
\; = \; \BI - 2\beta \bar{\BK} + \beta^2 \EXP\big[\norm{\Bk}^2
\Bk \Bk^\top\big] \; \preceq \; \BI - \beta \bar{\BK} \, ,
\end{equation}
and cancelling $\EXP \norm{\BE}_{\rF}^2$ from both sides gives
$\EXP[\TR(\BE \bar{\BK} \BE^\top)] \le \beta \EXP[\norm{\Brh}^2
\norm{\Bk}^2]$; the claim follows from
$\lambda_{\min}(\bar{\BK}) \EXP\norm{\BE}_{\rF}^2 \le \EXP[\TR(\BE
\bar{\BK} \BE^\top)]$ and $\TR \BSi_\beta = \EXP\norm{\BE}_{\rF}^2$.
Finally $\BS = \BZe$ is feasible, so
$\EXP\norm{\Brh}^2 \le \EXP\norm{\Bv}^2$, which bounds both terms of
\eqref{eq:Adsteady} by $\EXP\norm{\Bv}^2 / \lambda_{\min}(\bar{\BK})$
and gives \eqref{eq:Adclean}.
\end{proof}

\begin{corollaryA}[Transfer to the running or online memory]
\label{cor:AonlineDelta}
With $\mu_t$ the law of $\BS_t$ and
$B_\beta := \EXP_{\pi_\beta}[\norm{\BS}_{\rF}^2]^{1/2}$,
\begin{equation}
\EXP\big[\norm{\BS_t}_{\rF}^2\big]^{1/2}
\; \le \; B_\beta \; + \;
\big(1 - \beta \lambda_{\min}(\bar{\BK})\big)^{t/2}\,
\cW_2(\mu_0, \pi_\beta) \, ,
\qquad
B_\beta^2 \; \le \;
\frac{2\, \EXP\norm{\Bv}^2}{\lambda_{\min}(\bar{\BK})} \, .
\end{equation}
\end{corollaryA}

\begin{proof}
Synchronous coupling: driving two copies with the same
$(\Bk_t, \Bv_t, \beta_t)$ cancels the additive term, leaving
$\BD_{t+1} = \BD_t(\BI - \beta_t \Bk_t \Bk_t^\top)$, whence
$\EXP\norm{\BD_{t+1}}_{\rF}^2 \le (1 - \beta
\lambda_{\min}(\bar{\BK})) \EXP\norm{\BD_t}_{\rF}^2$ by the
same positive-semidefinite inequality as above. Take the infimum over
couplings and use that $\norm{\cdot}_{\rF}$ is $1$-Lipschitz.
\end{proof}

\subsection{Relation to the KronLSTM Bound}
\label{sec:ADeltakron}

The bound is a consequence of the delta rule being a relaxed
projection that is a correction proportional to a residual that shrinks as
the memory improves. Both models have that structure,
therefore separates the delta-rule family from mLSTM, whose Hebbian
write is uncorrected and has no fixed point to contract toward.
It separates {\em the steady-state argument} from {\em the
non-expansiveness argument}, whose forget-gate bound
$\beta C A_{\max}/(1-\Rf)$ is $\Omega(T)$ once $\Rf \to 1$ as long
contexts require. 

\paragraph{The constants are not comparable.}
The two bounds read
\begin{equation}
\text{KronLSTM:} \;\;
\frac{\EXP[c^2]}{\lambda_{\min}(\bar{\BA})} \, ,
\qquad
\text{DeltaNet:} \;\;
\frac{2\, \EXP[\norm{\Bv}^2]}{\lambda_{\min}(\bar{\BK})} \, ,
\qquad
\bar{\BA} = \EXP\big[\Ba \Ba^\top\big] \in \dR^{d^2 \times d^2},
\;\;
\bar{\BK} \in \dR^{d \times d} \, .
\end{equation}
These are bounds on the same object, $\norm{\BS}_{\rF}$, but they are
not bounds for the same task: KronLSTM's stream carries one scalar
$c_t$ per step, DeltaNet's a $d$-vector $\Bv_t$, so $\EXP[c^2]$ and
$\EXP\norm{\Bv}^2$ measure different amounts of information being
written. The denominators also differ structurally:
$\bar{\BK} \succ \BZe$ requires only that the keys span $\dR^d$,
whereas $\bar{\BA} \succ \BZe$ requires the $\Ba_i = \Bk_i \otimes
\Bv_i$ to span $\dR^{d^2}$, and $\lambda_{\min}(\bar{\BA})$ is
typically the smaller of the two. The reason is that for independent isotropic keys
and values, $\bar{\BA} = \EXP[\Bk\Bk^\top] \otimes \EXP[\Bv\Bv^\top]$
gives $\lambda_{\min}(\bar{\BA}) = \lambda_{\min}(\bar{\BK})\,
\lambda_{\min}(\EXP[\Bv\Bv^\top])$, a product of two small numbers. On the other hand 
the constant 2 appears in the bound of the DeltaNet. Which
constant is smaller in a given setting is an empirical question, not
a structural one.

\paragraph{DeltaNet has a smaller bound.}
The step-size condition is milder. KronLSTM requires
$\beta \norm{\Bk_t}^2 \norm{\Bv_t}^2 \le 1$, so a large value norm
eats directly into the admissible step size and hence into the
contraction rate $\beta \lambda_{\min}(\bar{\BA})$; DeltaNet requires
only $\beta_t \norm{\Bk_t}^2 \le 1$, in which $\Bv_t$ does not appear
at all. It enters the drive but never the stability condition.

\paragraph{KronLSTM has advantages.}
$d^2$ rather than $d$ exactly satisfiable associations
per scalar slot, interference gated by the product of key {\em and}
value overlap rather than key overlap alone, an update direction
$\Bv_t \Bk_t^\top$ fixed by the input pair rather than the
state-dependent $\Bu_t \Bk_t^\top$ so that a write's footprint on
unrelated memories does not compound with context length,
single-binding write/rewrite/erase, and linearity of the memory in
the target sequence. None of these follow from, or are affected by,
the steady-state bound.

\section{Retention Rates for KronLSTM and DeltaNet}
\label{sec:Aretention}

\subsection{KronLSTM Setting}

Let $n > d^2$ associations $(\Bk_i, \Bv_i, c_i)$, $1 \le i \le n$, be
given. Vectorized, with $\Bs = \vecop(\BS) \in \dR^{d^2}$ and
$\Ba_i = \Bk_i \otimes \Bv_i$
($\norm{\Ba_i} = \norm{\Bk_i} \norm{\Bv_i}$), the associations form
the linear system $\BA \Bs = \Bc$ with $\BA \in \dR^{n \times d^2}$
stacking the rows $\Ba_i^\top$; we assume $\BA$ has full column rank
$d^2$ (the features span $\dR^{d^2}$). At each step an index $i_t$
is drawn with probabilities $p_i > 0$, independently of the past, and
the rule is applied with fixed step size $\beta$ satisfying
\begin{equation} \label{eq:Astepcond2}
0 \; < \; \beta \, \norm{\Ba_i}^2 \; < \; 2
\qquad \text{for all } i \, .
\end{equation}
Define
\begin{align}
\bar{\BA} \; &:= \; \EXP\big[\Ba \Ba^\top\big]
\; = \; \sum_{i=1}^n p_i\, \Ba_i \Ba_i^\top \; \succ \; \BZe \, ,
&\bar{\Bb} \; &:= \; \EXP\big[c\, \Ba\big]
\; = \; \sum_{i=1}^n p_i\, c_i\, \Ba_i \, , \\
\Bs^{\rL\rS} \; &:= \; \bar{\BA}^{-1} \bar{\Bb}
\; = \; \argmin_{\Bs} \sum_{i=1}^n p_i \big(\Ba_i^\top \Bs -
c_i\big)^2 \, ,
&\rho_i \; &:= \; c_i - \Ba_i^\top \Bs^{\rL\rS}
\quad \text{(optimal residuals)} \, ,
\end{align}
so that the normal equations read
\begin{equation} \label{eq:Anormal2}
\EXP\big[\rho\, \Ba\big] \; = \; \sum_{i=1}^n p_i\, \rho_i\, \Ba_i
\; = \; \BZe
\qquad \Longleftrightarrow \qquad
\bar{\Bb} \; = \; \bar{\BA}\, \Bs^{\rL\rS} \, .
\end{equation}
For uniform sampling, $\bar{\BA} = \tfrac{1}{n} \BA^\top \BA$ and
$\Bs^{\rL\rS} = (\BA^\top \BA)^{-1} \BA^\top \Bc$ is the ordinary
least-squares solution, i.e.\ the unique minimizer of $\rL$.
Throughout, expectations without subscript are over $i \sim p$, and
$\EXP_{\pi}$ denotes expectation of a random state distributed
according to a law $\pi$. Vectorized, the rule and its error form are
\begin{equation} \label{eq:Aerrrec2}
\Bs_{t+1} \; = \; \BM_{i_t} \Bs_t + \beta\, c_{i_t} \Ba_{i_t} \, ,
\qquad
\Be_{t+1} \; = \; \BM_{i_t}\, \Be_t \; + \; \beta\, \rho_{i_t}\,
\Ba_{i_t} \, ,
\qquad
\BM_i \; := \; \BI - \beta\, \Ba_i \Ba_i^\top \, ,
\end{equation}
with $\Be_t := \Bs_t - \Bs^{\rL\rS}$; the error form follows by
inserting $c_i = \Ba_i^\top \Bs^{\rL\rS} + \rho_i$. The constants
\begin{equation}
\gamma_\beta \; := \; \min_i\, \beta\,\big(2 - \beta
\norm{\Ba_i}^2\big) \; > \; 0 \, ,
\qquad
\bar{\rho} \; := \; \gamma_\beta\, \lambda_{\min}(\bar{\BA}) \; \in \;
(0, 1] \, ,
\qquad
\sigma^2 \; := \; \EXP\big[\rho^2 \norm{\Ba}^2\big] 
\end{equation}
satisfy $0 < \bar{\rho} \le 1$.

Under \eqref{eq:Astepcond2}, the expected iterate obeys the {\em exact}
deterministic recursion
\begin{equation} \label{eq:Ameanrec2}
\EXP\big[\Bs_{t+1}\big] \; = \;
\big(\BI - \beta \bar{\BA}\big)\, \EXP\big[\Bs_t\big]
\; + \; \beta\, \bar{\Bb}
\qquad \Longleftrightarrow \qquad
\EXP\big[\Be_{t+1}\big] \; = \;
\big(\BI - \beta \bar{\BA}\big)\, \EXP\big[\Be_t\big] \, ,
\end{equation}
whose unique fixed point is $\Bs^{\rL\rS}$, and
\begin{equation}
\norm{\EXP[\BS_t] - \BS^{\rL\rS}}_{\rF}
\; \le \; r^t\, \norm{\BS_0 - \BS^{\rL\rS}}_{\rF} \, ,
\qquad
r \; := \; \max_k\, \abs{1 - \beta \lambda_k(\bar{\BA})} \; < \; 1 \,
,
\end{equation}
so $\EXP[\BS_t] \to \BS^{\rL\rS}$ linearly. The convergence of the
mean is {\em exactly unbiased}: the fixed point of \eqref{eq:Ameanrec2}
is $\Bs^{\rL\rS}$ itself, with no $O(\beta)$ correction.

\subsection{DeltaNet Setting}
\label{sec:ADeltaSetting}

DeltaNet's update is
\begin{equation} \label{eq:Adelta1}
\BS_{t+1} \; = \; \BS_t \; - \; \beta_t \big(\BS_t \Bk_t - \Bv_t\big)
\Bk_t^\top
\; = \; \BS_t \big(\BI - \beta_t \Bk_t \Bk_t^\top\big)
\; + \; \beta_t\, \Bv_t \Bk_t^\top \, ,
\end{equation}
and Gated DeltaNet's is the same with $\BS_t$ replaced by
$\Rf_t \BS_t$ in the first term. Two features are immediate. First,
\eqref{eq:Adelta1} is {\em affine} in $\BS_t$, so the pair
(memory, random input) generates a Markov chain of exactly the type
analyzed for the Kronecker Delta Rule. Second, and decisively, the
memory is acted on from the {\em right} only, so the rows of $\BS$ do
not mix. Transposing \eqref{eq:Adelta1},
\begin{equation}
\BS_{t+1}^\top \; = \; \big(\BI - \beta_t \Bk_t \Bk_t^\top\big)
\BS_t^\top \; + \; \beta_t\, \Bk_t \Bv_t^\top \, ,
\end{equation}
and reading off the $j$-th column, i.e.\ the $j$-th row
$\Bs^{(j)} := \BS^\top \Be_j$ of $\BS$,
\begin{equation} \label{eq:Arowrec1}
\Bs^{(j)}_{t+1} \; = \; \big(\BI - \beta_t \Bk_t \Bk_t^\top\big)\,
\Bs^{(j)}_t \; + \; \beta_t\, v_{t,j}\, \Bk_t \, ,
\qquad j = 1, \dots, d \, .
\end{equation}
We will use later the key covariance matrix $\bar{\BK} = \EXP[\Bk \Bk^\top]$ and the value covariance matrix $\bar{\BV} = \EXP[\Bv \Bv^\top]$. 

\paragraph{The structural correspondence.}
Equation \eqref{eq:Arowrec1} is {\em identical} to the vectorized
KronLSTM recursion
$\Bs_{t+1} = (\BI - \beta \Ba_t \Ba_t^\top)\Bs_t + \beta c_t \Ba_t$
under the substitution
\begin{equation}
\Ba_t \; = \; \Bk_t \otimes \Bv_t \;\; \longmapsto \;\; \Bk_t \, ,
\qquad
c_t \;\; \longmapsto \;\; v_{t,j} \, ,
\qquad
\dR^{d^2} \;\; \longmapsto \;\; \dR^{d} \, .
\end{equation}
DeltaNet is thus $d$ independent copies of the {\em same} scalar
delta rule all driven by a common
key $\Bk_t$ and differing only in their scalar targets $v_{t,j}$.
KronLSTM is a single copy of that rule on a $d^2$-dimensional state
with the richer direction $\Bk_t \otimes \Bv_t$. Neither is a
generalization of the other; both are Kaczmarz/LMS iterations, which
is why the same theory covers both.

\subsection{KronLSTM Retention of the Initial Memory}

We consider KronLSTM and the Kronecker Delta Rule.
We ask how much of an initial memory $\Bs_0$ survives after $t$ steps.
The mean recursion \eqref{eq:Ameanrec2} is affine, so it separates the
two contributions exactly. Writing $\BM := \BI - \beta \bar{\BA}$ and
unrolling,
\begin{equation} \label{eq:AmeanrecN}
\EXP\big[\Bs_t\big]
\; = \; \BM^t \Bs_0
\; + \; \beta \sum_{\tau=0}^{t-1} \BM^{\tau}\, \bar{\Bb}
\; = \; \underbrace{\BM^t\, \Bs_0}_{\text{initial memory}}
\; + \; \underbrace{\big(\BI - \BM^t\big)\, \Bs^{\rL\rS}}_{\text{acquired memory}} \, ,
\end{equation}
where the geometric sum was evaluated with
$\beta \sum_{\tau=0}^{t-1}\BM^\tau = (\BI - \BM^t)(\BI - \BM)^{-1}\beta
= (\BI - \BM^t)\bar{\BA}^{-1}$ and $\bar{\BA}^{-1}\bar{\Bb} =
\Bs^{\rL\rS}$; equivalently
$\EXP[\Bs_t] - \Bs^{\rL\rS} = \BM^t(\Bs_0 - \Bs^{\rL\rS})$, which is
the error recursion \eqref{eq:Aerrrec2} in the mean.
Everything in $\EXP[\Bs_t]$ that depends on $\Bs_0$ sits in the first
term, and it depends on it linearly. The natural measure of retention
is therefore the read-out of that term along the stored direction,
\begin{equation} \label{eq:Aretdef}
\cR_t(\Bs_0) \; := \;
\Bs_0^\top \frac{\partial\, \EXP[\Bs_t]}{\partial\, \Bs_0}\, \Bs_0
\; = \; \Bs_0^\top \BM^t\, \Bs_0 \, ,
\qquad \norm{\Bs_0} = 1 \, ,
\end{equation}
the fraction of a unit initial memory still legible after $t$ steps.
It equals $1$ at $t = 0$ and requires no assumption relating $\Bs_0$ to
$\Bs^{\rL\rS}$: the acquired term is filtered out by differentiation
rather than by orthogonality.

\begin{propositionA}[Retention rate]
\label{prop:Aretention}
Assume $\kappa := \beta\, A_{\max}^2 \le 1$ with
$A_{\max}^2 = \max_i \norm{\Ba_i}^2$, so that
$\beta\,\lambda_{\max}(\bar{\BA}) \le 1$ and $\BM \succeq \BZe$. Then
for every unit $\Bs_0$
\begin{equation} \label{eq:Aret2sided}
\big(1 - \beta\, \lambda_{\max}(\bar{\BA})\big)^t
\;\; \le \;\; \cR_t(\Bs_0) \;\; \le \;\;
\big(1 - \beta\, \lambda_{\min}(\bar{\BA})\big)^t \, ,
\end{equation}
both bounds being attained, at the corresponding eigenvectors of
$\bar{\BA}$. Averaged over a uniformly random unit $\Bs_0$ the
retention is exactly the normalized trace,
\begin{equation} \label{eq:Aretavg}
\EXP_{\Bs_0}\big[\cR_t(\Bs_0)\big]
\; = \; \frac{\TR\, \BM^t}{d^2}
\; = \; \frac{1}{d^2} \sum_{k=1}^{d^2}
\big(1 - \beta\, \lambda_k(\bar{\BA})\big)^t \, ,
\end{equation}
and the {\em retention horizon} --- the number of steps after which the
slowest mode has decayed to $1/\Rre$ --- is
$t_{1/\Rre} = 1/\log\big(1-\beta\lambda_{\min}(\bar{\BA})\big)^{-1}
\simeq 1/\big(\beta\, \lambda_{\min}(\bar{\BA})\big)$.
\end{propositionA}

\begin{proof}
$\bar{\BA}$ is symmetric, so $\BM = \BI - \beta\bar{\BA}$ is symmetric
with eigenvalues $1 - \beta\lambda_k(\bar{\BA}) \in [0,1)$ under
$\beta\lambda_{\max} \le 1$, hence $\BM^t$ is symmetric positive
semidefinite with eigenvalues $(1-\beta\lambda_k)^t$.
Equation~\eqref{eq:Aret2sided} is the Rayleigh-quotient bound for
$\BM^t$, with equality on its extreme eigenvectors, which are those of
$\bar{\BA}$ for $\lambda_{\max}$ and $\lambda_{\min}$ respectively.
For \eqref{eq:Aretavg}, $\EXP[\Bs_0 \Bs_0^\top] = \BI/d^2$ for a
uniform unit vector in $\dR^{d^2}$, so
$\EXP[\Bs_0^\top \BM^t \Bs_0] = \TR(\BM^t \EXP[\Bs_0\Bs_0^\top])
= \TR(\BM^t)/d^2$.
The step-size condition holds because
$\lambda_{\max}(\bar{\BA}) \le \max_i \norm{\Ba_i}^2 = A_{\max}^2$,
$\bar{\BA}$ being a convex combination of the $\Ba_i \Ba_i^\top$.
\end{proof}

\subsection{DeltaNet Retention of the Initial Memory}

\begin{corollaryA}[DeltaNet retention]
\label{cor:Aretdelta}
By the row decomposition \eqref{eq:Arowrec1}, DeltaNet's mean memory
obeys $\EXP[\BS_t] = \BS_0 \BN^t + \BS^{\rL\rS}(\BI - \BN^t)$ with
$\BN := \BI - \beta \bar{\BK}$, so with
$\cR_t(\BS_0) := \abra{\BS_0, \BS_0 \BN^t}_{\rF}$ and
$\norm{\BS_0}_{\rF} = 1$,
\begin{equation} \label{eq:Aretdelta}
\big(1 - \beta\, \lambda_{\max}(\bar{\BK})\big)^t
\;\; \le \;\; \cR_t(\BS_0) \;\; \le \;\;
\big(1 - \beta\, \lambda_{\min}(\bar{\BK})\big)^t \, ,
\qquad
\EXP_{\BS_0}\big[\cR_t\big] \; = \; \frac{\TR\, \BN^t}{d} \, ,
\end{equation}
under $\beta \max_i \norm{\Bk_i}^2 \le 1$. The statement is
Proposition~\ref{prop:Aretention} under the substitution
$\bar{\BA} \mapsto \bar{\BK}$, $d^2 \mapsto d$ of
Section~\ref{sec:Arecursion}.
\end{corollaryA}

\subsection{KronLSTM Retains Longer than DeltaNet}

The two retention rates are governed by the spectra of $\bar{\BA}$ and
$\bar{\BK}$, and for independent keys and values these are related
exactly.

\begin{propositionA}[Spectral comparison]
\label{prop:Aspectral}
Let $\Bk$ and $\Bv$ be independent and let
$\bar{\BV} := \EXP[\Bv \Bv^\top]$. Then
$\bar{\BA} = \bar{\BK} \otimes \bar{\BV}$, so the eigenvalues of
$\bar{\BA}$ are the $d^2$ pairwise products of those of $\bar{\BK}$ and
$\bar{\BV}$, and in particular
$\lambda_{\min}(\bar{\BA}) = \lambda_{\min}(\bar{\BK})\,
\lambda_{\min}(\bar{\BV})$ and
$\lambda_{\max}(\bar{\BA}) = \lambda_{\max}(\bar{\BK})\,
\lambda_{\max}(\bar{\BV})$. If the values are $\ell_2$-normalized,
$\norm{\Bv} = 1$, then $\TR \bar{\BV} = 1$ and hence
\begin{equation} \label{eq:Aspectral}
\lambda_{\min}(\bar{\BA}) \;\; \le \;\;
\frac{\lambda_{\min}(\bar{\BK})}{d} \, ,
\qquad
\lambda_{\max}(\bar{\BA}) \;\; \le \;\; \lambda_{\max}(\bar{\BK}) \, ,
\end{equation}
with equality in the first when the values are isotropic,
$\bar{\BV} = \BI/d$.
\end{propositionA}

\begin{proof}
$\Ba \Ba^\top = (\Bk \otimes \Bv)(\Bk \otimes \Bv)^\top
= \Bk\Bk^\top \otimes \Bv\Bv^\top$ by the mixed-product rule, and
independence factorizes the expectation. The spectrum of a Kronecker
product of symmetric matrices is the set of pairwise products.
For normalized values $\lambda_{\min}(\bar{\BV}) \le \TR\bar{\BV}/d
= 1/d$ and $\lambda_{\max}(\bar{\BV}) \le \TR \bar{\BV} = 1$, giving
\eqref{eq:Aspectral}; isotropy forces $\bar{\BV} = \BI/d$ and the first
inequality becomes an equality.
\end{proof}

Because keys and values are $\ell_2$-normalized in the KronLSTM block,
$\norm{\Ba_i}^2 = \norm{\Bk_i}^2 \norm{\Bv_i}^2 = 1 =
\norm{\Bk_i}^2$, so both models admit exactly the same step sizes
$\beta \le 1$ and the comparison at equal $\beta$ is well posed. Under
Propositions~\ref{prop:Aretention} and \ref{prop:Aspectral}:

\begin{itemize}
\item {\bf The slowest mode decays $d$ times more slowly.} KronLSTM's
worst-case retention $(1-\beta\lambda_{\min}(\bar{\BA}))^t$ dominates
DeltaNet's $(1-\beta\lambda_{\min}(\bar{\BK}))^t$, and the retention
horizons are
\begin{equation} \label{eq:Ahorizon}
t_{1/\Rre}^{\text{KronLSTM}} \;\simeq\;
\frac{1}{\beta\, \lambda_{\min}(\bar{\BA})}
\;\; \ge \;\;
\frac{d}{\beta\, \lambda_{\min}(\bar{\BK})}
\; = \; d\; t_{1/\Rre}^{\text{DeltaNet}} \, ,
\end{equation}
a factor $d$ in context length at which a stored association is still
legible, with equality for isotropic values.
\item {\bf No mode decays faster.} By the second half of
\eqref{eq:Aspectral}, $\lambda_{\max}(\bar{\BA}) \le
\lambda_{\max}(\bar{\BK})$, so even KronLSTM's fastest-decaying mode is
no faster than DeltaNet's. The comparison is therefore not a trade of
some directions against others: the whole spectrum of retention rates
moves in KronLSTM's favour. For isotropic values it moves by exactly
the factor $1/d$, $\lambda_k(\bar{\BA}) = \lambda_k(\bar{\BK})/d$ with
multiplicity $d$ each.
\item {\bf On average, likewise.} Comparing \eqref{eq:Aretavg} with
\eqref{eq:Aretdelta} for isotropic values,
$\TR \BM^t/d^2 = \tfrac1d \sum_k (1 - \beta
\lambda_k(\bar{\BK})/d)^t \ge \tfrac1d \sum_k (1 - \beta
\lambda_k(\bar{\BK}))^t = \TR \BN^t/d$ termwise.
\end{itemize}

\paragraph{Retention is bought from adaptation.}
\label{rem:Atradeoff}
The same matrix $\BM^t$ governs both terms of \eqref{eq:AmeanrecN}: the
acquired memory approaches $\Bs^{\rL\rS}$ at rate $\BI - \BM^t$.
Retaining the initial memory $d$ times longer therefore means
approaching the least-squares memory $d$ times more slowly, at fixed
$\beta$. The statement is not that KronLSTM forgets less for free, but
that its per-step write is $d$ times gentler in {\em every} direction
of an $d^2$-dimensional memory, where DeltaNet's is concentrated on a
$d$-dimensional one --- the same trade-off that appears in the capacity
count ($d^2$ against $d$ exactly satisfiable associations) and in the
norm bound, seen through its time
constant. Long contexts are precisely the regime in which the number of
associations, not the speed of fitting any one of them, is binding.

\section{Further Results}
\label{sec:Afurther-results}

\subsection{400M language models}
\label{subsec:400m-lm-results}

Models with 400M/40B (Table~\ref{tab:mcqa-400m}) are closer to each other 
in their performance. 
KronLSTM has the best average and performs best
on HellaSwag, ARC-easy, ARC-challenge, and WinoGrande.
Mamba-3 is best on PIQA, KDA on LAMBADA, and
GDN-2 on WinoGrande where it ties with KronLSTM.
Compared with mLSTM, the matrix-memory xLSTM variant that KronLSTM builds
on, KronLSTM is better on five of the six tasks, raises the average by
2.1 points, and lowers LAMBADA perplexity from 43.40 to 28.28.
mLSTM is only marginally ahead on PIQA.

\begin{table}[h]
\centering
\small
\setlength{\tabcolsep}{5pt}
\caption{Downstream performance of 400M/40B linear models.
Avg-6 is the mean accuracy over the five commonsense reasoning tasks and LAMBADA.
Best per column in \textbf{bold}, second best \underline{underlined}.}
\begin{tabular}{lcccccc|c|c}
\toprule
Model & Hella & PIQA & ARC-e & ARC-c & Wino & LMB acc & Avg-6 $\uparrow$ & LMB ppl $\downarrow$ \\
\midrule
  GDN              & 41.28 & 66.05 & 57.87 & \underline{31.48} & 49.64 & 32.31 & 46.44 & 34.36 \\
  Mamba-3          & 41.89 & \textbf{67.25} & \underline{60.82} & 31.23 & \underline{52.33} & \underline{35.11} & \underline{48.10} & 29.51 \\
  KDA              & 42.08 & 66.16 & 60.48 & 31.40 & 51.30 & \textbf{36.70} & 48.02 & \textbf{27.84} \\
  GDN-2            & \underline{42.33} & 66.87 & 60.23 & 29.69 & \textbf{52.41} & 33.84 & 47.56 & 29.14 \\
  mLSTM            & 40.29 & \underline{67.19} & 60.77 & 29.78 & 50.75 & 29.87 & 46.29 & 43.40 \\
  \midrule
  KronLSTM         & \textbf{43.09} & 67.03 & \textbf{61.03} & \textbf{32.00} & \textbf{52.41} & 35.05 & \textbf{48.43} & \underline{28.28} \\
\bottomrule
\end{tabular}
\label{tab:mcqa-400m}
\end{table}

\subsection{Ablation on different formulations of KronLSTM}
\label{sec:Aablation-subsec}

Table~\ref{tab:write-rule-ablation} ablates the components of the KronLSTM
write rule at 1.5B/100B. We report Avg-6, the mean accuracy over the five
commonsense reasoning tasks and LAMBADA, together with LAMBADA perplexity,
the average over the six-task real-world recall suite at 2k context, and
the average over the six RULER NIAH tasks of Table~\ref{tab:ruler-k128} at
4k, 8k and 16k. Three variants are compared with the base model of
\eqref{eq:gated-rule}: \\
(1) the $\ell_2$-normalization of $\Bv_t$ is removed; \\
(2) erase and write are coupled by using the target $c_t$ instead of
$\tilde{c}_t = c_t/\beta_t$, so that $\beta_t$ scales both terms as in Gated
DeltaNet; \\
(3) the write gate is fixed to the unit target $c_t \equiv 1$.

On language modeling, variants (1) and (2) stay within 0.3 accuracy points
and 0.25 perplexity points of the base model, with the coupled variant
marginally ahead on both. Fixing the target costs 1.6 accuracy points and
0.8 perplexity points, the largest degradation among the ablations. On the
recall suite, variants (1) and (2) gain 1.0 and 0.5 points, respectively.
On NIAH the ordering reverses: at 4k the base model leads the coupled
variant by 12.6 points and the unnormalized variant by 19.9 points, at 8k
the unnormalized variant is ahead by 2.2 points, and at 16k the base model
leads all three variants. The learned write gate and the decoupled write
therefore matter most for retrieval within the pretraining context, while
their effect on language modeling is small.

\begin{table}[h]
  \centering
  \caption{Write-rule ablation at 1.5B/100B, relative to \textsc{KronLSTM}.
  Avg-6 is the mean of the five commonsense reasoning tasks plus LAMBADA accuracy. Recall is the six-task
  real-world recall suite at 2k context. RULER is averaged at each evaluation length. Best per column in \textbf{bold}, second best \underline{underlined}.}
  {\small
  \begin{tabular}{lcc|c|ccc}
  \toprule
  Ablations & Avg-6 $\uparrow$ & LMB ppl $\downarrow$ & Recall $\uparrow$ & NIAH-Avg@4k & @8k & @16k  \\
  \midrule
      base & \underline{59.72} & \underline{11.716} & 30.3 & \textbf{59.5} & \underline{30.5} & \textbf{14.3} \\
      \midrule
      no $\ell_2$-norm on $\Bv_t$ & 59.41 & 11.844 & \textbf{31.3} & 39.6 & \textbf{32.7} & \underline{12.6}  \\
      $\beta$ on both sides & \textbf{60.00} & \textbf{11.479} & \underline{30.8} & \underline{46.9} & 29.4  & 11.8 \\
      non-learnable target & 58.17 & 12.495 & 30.1 & 41.9 & 26.2 & 10.1 \\
  \bottomrule
  \end{tabular}}
  
\label{tab:write-rule-ablation}
\end{table}

\subsection{Long context extension results}

\begin{table}[h]
\centering
\caption{RULER NIAH for the 32k-context-extended linear models, evaluated on identical documents. Better of the two per task and length in \textbf{bold}.}
{
\begin{tabular}{ll cccccc}
\toprule
Model & Len. & S1 & S2 & S3 & MK & MQ & MV \\
\midrule
\multicolumn{7}{c}{\textit{Linear models}} \\
\midrule
\multirow{4}{*}{GDN-2}             & 4k  & \textbf{100.0} & 74.0 & 68.2 & 28.4 & \textbf{56.5} & \textbf{62.5} \\
& 8k  & \textbf{100.0} & 55.6 & 17.6 & 23.2 & \textbf{38.8} & \textbf{41.9} \\
& 16k & \textbf{100.0} & 10.2 & 7.4 & 9.4 & 7.4 & 8.9 \\
& 32k & \textbf{97.8} & 2.6 & 0.8 & 3.0 & 1.5 & 7.2 \\
\midrule
\multirow{4}{*}{KronLSTM} & 4k  & \textbf{100.0} & \textbf{99.6} & \textbf{92.6} & \textbf{30.8} & 37.1 & 43.9 \\
& 8k  & 99.4 & \textbf{94.0} & \textbf{61.2} & \textbf{28.4} & 33.5 & 36.8 \\
& 16k & 94.6 & \textbf{27.6} & \textbf{41.8} & \textbf{18.0} & \textbf{21.5} & \textbf{30.1} \\
& 32k & 81.4 & \textbf{5.6} & \textbf{8.6} & \textbf{5.0} & \textbf{3.0} & \textbf{11.3} \\
\midrule
\multicolumn{7}{c}{\textit{Hybrid models}} \\
\midrule
  \multirow{4}{*}{GDN-2}       & 4k  & 52.0 & 56.0 & 45.6 & \textbf{49.2} & 24.3 & \textbf{31.4} \\
  & 8k  & 28.6 & 31.2 & 20.2 & 26.6 & 13.0 & 18.2 \\
  & 16k & 13.0 & 16.8 & 9.8 & 18.8 & 6.8 & 14.6 \\
  & 32k & 7.0 & 7.8 & 3.6 & 6.6 & 2.6 & 6.6 \\
\midrule
  \multirow{4}{*}{KronLSTM}    & 4k  & \textbf{86.2} & \textbf{85.4} & \textbf{68.8} & 43.8 & \textbf{29.5} & 31.0 \\
  & 8k  & \textbf{74.6} & \textbf{38.0} & \textbf{30.4} & \textbf{27.0} & \textbf{17.8} & \textbf{24.2} \\
  & 16k & \textbf{55.6} & \textbf{19.2} & \textbf{11.0} & \textbf{19.0} & \textbf{12.5} & \textbf{17.2} \\
  & 32k & \textbf{38.4} & \textbf{8.2} & \textbf{4.8} & \textbf{6.8} & \textbf{6.3} & \textbf{8.3} \\
\bottomrule
\end{tabular}}
\label{tab:ruler-lc32k}
\end{table}

\section{Pretraining Details}
  \label{sec:Apretraining}

All models are pretrained from scratch using \texttt{lm-engine} \cite{mishra24lm_engine} with Triton kernels from \texttt{flash-linear-attention} \cite{yang26fla} whenever available. Every mixer in a comparison shares the corpus, the
tokenizer, the optimizer, the schedule, the batch size, and the seed; we only change the sequence mixer blocks. 
The models were trained on 16 nodes of 4$\times$NVIDIA GH200. Training is pure data parallelism under the FSDP-2 with bf16 mixed-precision.

\paragraph{Data and tokenization}

We train on the FineWeb-Edu 100BT sample, tokenized with the Llama 2 tokenizer with its standard vocab size of 32,000. Documents are concatenated, separated by the EOS token, and packed into fixed-length sequences of 8192 tokens; This is applied identically to every experiment. 

\paragraph{Architecture}
All models are decoder-only and identical outside the sequence mixer. We implemented the block structures of every baseline as in the original works. Mamba-3  was trained using the standard expansion factor of 2, a dstate of 128, and a head dimension of 64. For Mamba-3, we utilized the SISO formulation. For the other baselines and our method, we utilized $d_{qk} = d_v = 128$. This makes the state size equal across all models.
1.5B models have $24$ layers, a hidden dimension of $2048$ with 16 heads. The 400M models have 24 layers with a hidden dimension of $1024$  and 4 heads (here $d_v = 256$).

\paragraph{Optimization}

\begin{center}
\begin{tabular}{ll}
\toprule
optimizer & AdamW, $\beta = (0.9, 0.95)$, $\epsilon = 10^{-10}$ \\
peak learning rate & $3 \times 10^{-4}$ \\
schedule & linear warmup over 10\% of steps, then cosine to $0.1\times$ peak ($3 \times 10^{-5}$) \\
weight decay & 0.1, excluding norms, biases, $A_{\log}$, $b_{\Delta}$ and the gate bias \\
gradient clipping & global norm 1.0 \\
sequence length & 8192 \\
global batch & 128 sequences $=$ 1{,}048{,}576 tokens per step \\
precision & bf16 mixed precision, fp32 optimiser state \\
\bottomrule
\end{tabular}
\end{center}

\subsection{Long-context extension}
\label{sec:Alongctx}

The context-extended models of Section~\ref{subsec:long-context} branch
from the 1.5B pretraining runs at 75B tokens and are
trained for a further 25B tokens at a sequence length of 32{,}768,
so that both extended models finish at exactly 100B
tokens and are token-matched to the 8k-pretrained models. The architecture,
tokenizer, and all model hyperparameters are unchanged. The global batch is
held at the same 1M tokens per step, now as 32 sequences of 32{,}768
tokens instead of 128 sequences of 8{,}192.

\paragraph{Optimization.} The optimizer state is re-initialized at the
branch point for both models and the schedule restarts: linear warmup over
 5\% to a peak learning rate of $8.4\times10^{-5}$, which is
the value the parent cosine schedule had reached at the checkpoint, followed
by cosine decay over 23{,}800 steps to $0.1\times$ the peak. This mirrors
the warmup-and-decay shape used by \citet{gao2024prolong}. AdamW
hyperparameters, weight decay, and gradient clipping are as in
Appendix~\ref{sec:Apretraining}.

\paragraph{Data.} The corpus takes the four long-document domains of
ProLong-64K \citep{gao2024prolong}, books, code, textbooks, and arXiv,
re-tokenized with the Llama~2 tokenizer, and blends them at the ratios of
\citet{gao2024prolong} with FineWeb-Edu in place of their short-context
mixture. Documents are concatenated and packed into sequences of 32{,}768
tokens as in pretraining.

\paragraph{Systems.} Training runs on 8 nodes of $4\times$NVIDIA GH200
with ZeRO stage~2 and block-wise gradient checkpointing under bf16 mixed
precision.

\subsection{State Tracking Experiments}

\paragraph{State-tracking tasks.}
  Each sample is a token sequence of length $ \sim \mathcal{U}[2, 128]$
  padded to the context length, with a single supervised position at its end.
  \emph{Parity}: i.i.d.\ uniform bits, label the parity of the number of ones
  (chance $50\%$). \emph{$S_3$}: i.i.d.\ uniform transpositions
  $(0\,1), (0\,2), (1\,2)$, label the prefix product $x_1 \circ \dots \circ x_\ell$
  as one of the six group elements (chance $16.7\%$; tracking parity alone gives
  $33.3\%$). \emph{Modular arithmetic}:  an expression
  $d_1\, o_1\, d_2 \dots d_k\, =$ with uniform digits in $\{0,\dots,4\}$ and
  operators in $\{+,-,\times\}$, evaluated with operator precedence modulo 5,
  the result supervised at the \texttt{=} token (chance $20\%$). Samples are
  generated on the fly and deterministically from the run seed (seeds 1-3);
  evaluation uses $50{,}000$ fresh sequences per length $L \in \{128, 256, 512,
  1024\}$.

  \paragraph{Models and training.}
  All models share a 128-dimensional embedding, pre-norm residual blocks
  (RMSNorm, no feed-forward sublayer) and a linear read-out;
  they differ only in the mixer: Gated DeltaNet with the eigenvalue range
  extended to $[-1, 1]$ \citep{grazzi25unlocking}, Mamba-3 (SISO, expansion 2,
  state size 128, data-dependent RoPE), and KronLSTM.
   For optimizer we used AdamW with weight decay $0.1$, batch size 128, bf16. Parity and
  $S_3$: two blocks, constant learning rate $\in \{10^{-3}, 10^{-4}\}$, 50k steps,
  gradient clipping 100. Modular arithmetic requires depth and a longer budget
  for all three architectures: three blocks, 150k steps. We report end-of-training accuracy, mean $\pm$ standard
  deviation over three seeds, at the learning rate with the best mean accuracy
  over $L \in \{256\}$.

\section{Chunkwise-Parallel Training of KronLSTM}
\label{sec:Achunkwise}

The KronLSTM update \eqref{eq:gated-rule} revises the matrix memory by a
rank-one term whose coefficient depends on the memory only through a
scalar. It therefore admits the chunkwise-parallel training scheme of
DeltaNet \citep{Yang:24}, Gated DeltaNet \citep{yang2025gated}, Kimi Delta
Attention \citep{kimiteam2025kda}, and Gated DeltaNet-2
\citep{hatamizadeh2026gdn2}: the sequence is split into chunks of length
$C$, all token interactions inside a chunk are expressed as dense matrix
products and one small triangular solve, and the memory is propagated only
between chunks. Three points differ from the delta rule.
(i) The triangular solve is built on the Gram matrix of the Kronecker
features $\Ba_t = \Bk_t \otimes \Bv_t$, which by the mixed-product property
is the Hadamard product of a (decay-weighted) key Gram and the value Gram;
the $d_k d_v$-dimensional features are never formed.
(ii) The pseudo-values are scalar multiples $u_r \Bv_r$ of the values.
(iii) The incoming memory of a chunk enters the solve through $C$ scalar
scores $\Bv_r^\top \BS_0 (\Bga_r \odot \Bk_r)$ instead of the
$d_v$-dimensional read-outs $\BS_0 \Bk_r$ of the delta rule. The vector
forget gate is handled as in KDA and Gated DeltaNet-2 by absorbing the
cumulative decay into an asymmetric pair of keys.

\subsection{Setting and notation}
\label{sec:Achunkwise-setting}

We consider one chunk of $C$ tokens and suppress the chunk index. Let
$\BS_0 \in \dR^{d_v \times d_k}$ be the memory entering the chunk and let
$\Bq_r, \Bk_r \in \dR^{d_k}$, $\Bv_r \in \dR^{d_v}$, $\beta_r, c_r \in \dR$
and $\bff_r \in (0,1)^{d_k}$, $\BF_r = \diag(\bff_r)$, be the inputs and
gates of the $r$-th token, $r = 1, \dots, C$. Writing $\BS_r$ for the memory
after token $r$, the update \eqref{eq:gated-rule} and the read-out
are\footnote{This is \eqref{eq:gated-rule} with the time index shifted by
one, so that $\BS_r$ denotes the memory after the write of token $r$. The
read-out uses this memory, $\Bh_r = \BS_r \Bq_r$, i.e.,\ the query of token
$r$ sees the association written at token $r$, as in DeltaNet and its gated
variants and in our kernels.}
\begin{equation} \label{eq:Acw-rec}
\BS_r \;=\; \BS_{r-1} \BF_r \;+\; u_r\, \Bv_r \Bk_r^\top ,
\qquad
u_r \;:=\; c_r - \beta_r\, y_r , \qquad
y_r \;:=\; \Bv_r^\top\, \BS_{r-1} \BF_r\, \Bk_r ,
\qquad
\Bh_r \;=\; \BS_r \Bq_r .
\end{equation}
We call $y_r$ the score of the pair $(\Bk_r, \Bv_r)$ on the decayed memory
and $u_r$ the write coefficient, the coefficient of the rank-one write along
$\Bv_r \Bk_r^\top$. It is the only quantity in \eqref{eq:Acw-rec} that
depends on the memory.

\paragraph{Cumulative decay.}
With $\bm{\ell}_r := \sum_{i \le r} \log \bff_i$ (a cumulative sum inside
the chunk) we set
\begin{equation} \label{eq:Acw-gamma}
\Bga_r \;:=\; \exp(\bm{\ell}_r) \;=\; \bff_1 \odot \cdots \odot \bff_r
\;\in\; (0,1]^{d_k}, \qquad
\Bga_0 \;:=\; \BOn, \qquad
\BGa_r \;:=\; \diag(\Bga_r) \;=\; \BF_1 \cdots \BF_r .
\end{equation}
All products, quotients and exponentials of $\Bga$'s are element-wise over
the key dimension. For $s \le r$ the ratio $\Bga_r / \Bga_s = \bff_{s+1}
\odot \cdots \odot \bff_r$ lies in $(0,1]^{d_k}$; it is the decay that a key
written at token $s$ has undergone by token $r$. Four decay-scaled copies of
the keys and queries appear:
\begin{equation} \label{eq:Acw-keys}
\widehat{\Bk}_r \;:=\; \Bga_r \odot \Bk_r , \qquad
\bar{\Bk}_r \;:=\; \Bk_r / \Bga_r , \qquad
\breve{\Bk}_r \;:=\; (\Bga_C / \Bga_r) \odot \Bk_r , \qquad
\widehat{\Bq}_r \;:=\; \Bga_r \odot \Bq_r .
\end{equation}
$\widehat{\Bk}_r$ is the key with which the memory is scored,
$\bar{\Bk}_r$ the key along which it is written in the decay-normalized frame
of Section~\ref{sec:Achunkwise-normalized}, and $\breve{\Bk}_r$ the key as it
appears in the memory at the end of the chunk.

\paragraph{Stacked matrices.}
$\BK, \BQ \in \dR^{C \times d_k}$ and $\BV \in \dR^{C \times d_v}$ stack
$\Bk_r^\top$, $\Bq_r^\top$, $\Bv_r^\top$ as rows; $\widehat{\BK},
\bar{\BK}, \breve{\BK}, \widehat{\BQ}$ stack the vectors of
\eqref{eq:Acw-keys}; $\Bbe, \Bc, \Bu, \Bp \in \dR^C$ stack the scalars;
$\BH \in \dR^{C \times d_v}$ stacks the read-outs $\Bh_r^\top$. The three
$C \times C$ matrices
\begin{equation} \label{eq:Acw-grams}
\BG^{kk}_{rs} \;:=\; \mathbbm{1}[r > s]\; \widehat{\Bk}_r^\top \bar{\Bk}_s
\;=\; \mathbbm{1}[r > s] \sum_{d=1}^{d_k} k_{rd}\, k_{sd}\,
\frac{\gamma_{rd}}{\gamma_{sd}} , \qquad
\BG^{vv} \;:=\; \BV \BV^\top , \qquad
\BG^{qk}_{rs} \;:=\; \mathbbm{1}[r \ge s]\; \widehat{\Bq}_r^\top \bar{\Bk}_s
\end{equation}
are the decay-weighted key Gram (strictly lower triangular), the value
Gram, and the decay-weighted read-out Gram (lower triangular with
diagonal). All entries of $\BG^{kk}$ and $\BG^{qk}$ involve only ratios
$\Bga_r / \Bga_s$ with $s \le r$, hence only factors in $(0,1]$.

\subsection{Decay-normalized recurrence}
\label{sec:Achunkwise-normalized}

The forget gate acts on the key dimension by a diagonal matrix, and
diagonal matrices commute. This allows the cumulative decay to be taken out
of the recurrence.

\begin{propositionA}[Decay-normalized recurrence]
\label{prop:Achunknorm}
Let $\widehat{\BS}_r := \BS_r \BGa_r^{-1}$, so that $\widehat{\BS}_0 =
\BS_0$. Then, for $r = 1, \dots, C$,
\begin{equation} \label{eq:Acw-normalized}
\widehat{\BS}_r \;=\; \widehat{\BS}_{r-1} \;+\; u_r\, \Bv_r \bar{\Bk}_r^\top ,
\qquad
y_r \;=\; \Bv_r^\top\, \widehat{\BS}_{r-1}\, \widehat{\Bk}_r ,
\end{equation}
and consequently
\begin{equation} \label{eq:Acw-unrolled}
\widehat{\BS}_r \;=\; \BS_0 \;+\; \sum_{s \le r} u_s\, \Bv_s \bar{\Bk}_s^\top .
\end{equation}
\end{propositionA}
\begin{proof}
Since $\BGa_r = \BGa_{r-1} \BF_r$, we have $\BS_{r-1} \BF_r =
\widehat{\BS}_{r-1} \BGa_{r-1} \BF_r = \widehat{\BS}_{r-1} \BGa_r$.
Inserting this into \eqref{eq:Acw-rec} gives $\BS_r = \widehat{\BS}_{r-1}
\BGa_r + u_r \Bv_r \Bk_r^\top$; right-multiplication by $\BGa_r^{-1}$
yields the first identity, because $\Bk_r^\top \BGa_r^{-1} =
\bar{\Bk}_r^\top$. The second identity is $y_r = \Bv_r^\top
(\widehat{\BS}_{r-1} \BGa_r) \Bk_r = \Bv_r^\top \widehat{\BS}_{r-1}
(\BGa_r \Bk_r)$. Unrolling the first identity from $\widehat{\BS}_0 = \BS_0$
gives \eqref{eq:Acw-unrolled}.
\end{proof}

In the normalized frame the recurrence is the Kronecker Delta Rule without
forget gate, applied to an asymmetric key pair: the association is scored
with the decayed key $\widehat{\Bk}_r$ and written along the un-decayed key
$\bar{\Bk}_r$. This is the same device by which Gated DeltaNet-2 absorbs
its channel-wise decay into asymmetric erase factors
\citep{hatamizadeh2026gdn2}. Equation~\eqref{eq:Acw-unrolled} is the
additive representation of the memory, $\BS_r = \big(\BS_0 + \sum_{s \le r}
u_s \Bv_s \bar{\Bk}_s^\top\big) \BGa_r$; the memory revisions do not have
to be materialized, only the $C$ scalars $u_s$.

\subsection{The write coefficients solve a unit lower triangular system}
\label{sec:Achunkwise-solve}

\begin{propositionA}[Write coefficients]
\label{prop:Achunksolve}
Let $p_r := \Bv_r^\top \BS_0\, \widehat{\Bk}_r$ be the score of the $r$-th
pair on the incoming memory, and let
\begin{equation} \label{eq:Acw-N}
\BN \;:=\; \diag(\Bbe)\, \big( \BG^{kk} \odot \BG^{vv} \big)
\;\in\; \dR^{C \times C} , \qquad
\BLa \;:=\; (\BI + \BN)^{-1} .
\end{equation}
$\BN$ is strictly lower triangular, so $\BI + \BN$ is unit lower
triangular and invertible for every $\Bbe$ and every input. The write
coefficients are
\begin{equation} \label{eq:Acw-u}
\Bu \;=\; \BLa\, \big( \Bc - \Bbe \odot \Bp \big)
\;=\; \underbrace{\BLa\, \Bc}_{=:\, \Bu^{\circ}}
\;-\; \BLa \diag(\Bbe)\, \Bp .
\end{equation}
\end{propositionA}
\begin{proof}
Insert \eqref{eq:Acw-unrolled} at $r-1$ into the score in
\eqref{eq:Acw-normalized}:
\[
y_r \;=\; \Bv_r^\top \BS_0 \widehat{\Bk}_r
\;+\; \sum_{s < r} u_s\, (\Bv_r^\top \Bv_s)(\bar{\Bk}_s^\top \widehat{\Bk}_r)
\;=\; p_r \;+\; \sum_{s < r} \big(\BG^{kk} \odot \BG^{vv}\big)_{rs}\, u_s .
\]
Hence $u_r = c_r - \beta_r y_r$ reads $u_r + \sum_{s<r} \BN_{rs} u_s = c_r
- \beta_r p_r$, which is row $r$ of $(\BI + \BN) \Bu = \Bc - \Bbe \odot
\Bp$.
\end{proof}

\paragraph{The Gram of the Kronecker features.}
With $\widehat{\Ba}_r := \widehat{\Bk}_r \otimes \Bv_r$ and $\bar{\Ba}_s :=
\bar{\Bk}_s \otimes \Bv_s$, the mixed-product property of the Kronecker
product gives $\widehat{\Ba}_r^\top \bar{\Ba}_s = (\widehat{\Bk}_r^\top
\bar{\Bk}_s)(\Bv_r^\top \Bv_s)$, so $\BG^{kk} \odot \BG^{vv}$ is the
strictly lower triangular part of the Gram matrix of the decay-weighted
Kronecker features. Without forget gate ($\bff_r \equiv \BOn$) it is the
strictly lower part of the Gram matrix $(\Ba_r^\top \Ba_s)_{rs}$ of the
features $\Ba_r = \Bk_r \otimes \Bv_r$, i.e.,\ exactly the matrix that the
UT transform of DeltaNet \citep{Yang:24} builds from $\BK \BK^\top$, with
the feature $\Ba_r$ in place of $\Bk_r$ as anticipated by the vectorized
form \eqref{eq:vecrule}. The $d_k d_v$-dimensional features are never
formed: two $C \times C$ Gram matrices of $d_k$- and $d_v$-dimensional
vectors and one Hadamard product suffice.

\paragraph{Decoupled erase and write.}
The write gate $c_r$ enters only the right-hand side of the solve, the
erase gate $\beta_r$ enters the system matrix and the coupling to $\BS_0$.
The state-independent part $\Bu^{\circ} = \BLa \Bc$ and the matrix $\BLa
\diag(\Bbe)$ can be computed for all chunks in parallel; they are the
analogues of the pseudo-values $\BU = \BT \BV$ and of $\BW = \BT \BK$ in
DeltaNet. The only sequential quantity is the vector of scores $\Bp$.

\subsection{Chunkwise form}
\label{sec:Achunkwise-form}

\begin{theoremA}[Chunkwise KronLSTM]
\label{thm:Achunkwise}
Let $\widetilde{\BV} := \diag(\Bu)\, \BV \in \dR^{C \times d_v}$ be the
matrix of pseudo-values, with rows $u_r \Bv_r^\top$, where $\Bu$ is given by
Proposition~\ref{prop:Achunksolve} and
\begin{equation} \label{eq:Acw-p}
\Bp \;=\; \big( (\BV \BS_0) \odot \widehat{\BK} \big)\, \BOn
\qquad \text{(row sums of a $C \times d_k$ matrix).}
\end{equation}
Then the memory at the end of the chunk and the read-outs of the chunk are
\begin{align}
\BS_C &\;=\; \BS_0\, \BGa_C \;+\; \widetilde{\BV}^\top \breve{\BK}
\;\in\; \dR^{d_v \times d_k} , \label{eq:Acw-state} \\
\BH &\;=\; \widehat{\BQ}\, \BS_0^\top \;+\; \BG^{qk}\, \widetilde{\BV}
\;\in\; \dR^{C \times d_v} . \label{eq:Acw-output}
\end{align}
\end{theoremA}
\begin{proof}
Equation~\eqref{eq:Acw-p} is $p_r = \Bv_r^\top \BS_0 \widehat{\Bk}_r =
\sum_d (\BV \BS_0)_{rd}\, \widehat{K}_{rd}$. For the state, $\BS_C =
\widehat{\BS}_C \BGa_C = \BS_0 \BGa_C + \sum_{s \le C} u_s \Bv_s
\bar{\Bk}_s^\top \BGa_C$ by \eqref{eq:Acw-unrolled}, and $\bar{\Bk}_s^\top
\BGa_C = \big((\Bga_C / \Bga_s) \odot \Bk_s\big)^\top = \breve{\Bk}_s^\top$,
which is \eqref{eq:Acw-state}. For the read-out, $\Bh_r = \BS_r \Bq_r =
\widehat{\BS}_r \BGa_r \Bq_r = \widehat{\BS}_r \widehat{\Bq}_r = \BS_0
\widehat{\Bq}_r + \sum_{s \le r} u_s\, (\bar{\Bk}_s^\top
\widehat{\Bq}_r)\, \Bv_s$ by \eqref{eq:Acw-unrolled}; transposing and
stacking over $r$ gives \eqref{eq:Acw-output}, the coefficient of $u_s
\Bv_s^\top$ in row $r$ being $\BG^{qk}_{rs}$.
\end{proof}

Indexing the chunks by $[n]$ and writing $\BS_{[n]}$ for the memory entering
chunk $n$, \eqref{eq:Acw-state} is the chunk-level recurrence
\begin{equation} \label{eq:Acw-chunkrec}
\BS_{[n+1]} \;=\; \BS_{[n]}\, \BGa_{[n],C} \;+\;
\widetilde{\BV}_{[n]}^\top \breve{\BK}_{[n]} ,
\end{equation}
with $\BS_{[0]} = \BZe$ (or a given initial memory). Everything in
\eqref{eq:Acw-grams}--\eqref{eq:Acw-output} except $\Bp$, $\Bu$ and
$\widetilde{\BV}$ is local to a chunk and is computed for all chunks at
once; the dependence $\BS_{[n]} \to \Bp \to \Bu \to \BS_{[n+1]}$ is the only
sequential part. Compared with the delta-rule family, the pseudo-value of
token $r$ is the value $\Bv_r$ scaled by the scalar $u_r$, and the memory
couples into the chunk through the $C$ scalars $\Bp$ rather than through
the $C$ vectors $\BS_0 \Bk_r \in \dR^{d_v}$: computing $\Bp$ needs one
matrix product $\BV \BS_0$ and a row-wise contraction over the value
dimension, which is where a kernel that tiles the memory over $d_v$ must
exchange partial sums.

\subsection{WY representation of the chunk transition}
\label{sec:Achunkwise-wy}

In vectorized form, $\Bs = \vecop(\BS)$ with $\vecop(\Bv \Bk^\top) = \Bk
\otimes \Bv$ and $\vecop(\BS \BF) = (\BF \otimes \BI_{d_v}) \vecop(\BS)$,
one step of \eqref{eq:Acw-rec} is $\Bs_r = (\BI - \beta_r \Ba_r
\Ba_r^\top)(\BF_r \otimes \BI_{d_v})\, \Bs_{r-1} + c_r \Ba_r$ with $\Ba_r =
\Bk_r \otimes \Bv_r$: a generalized Householder transition, as in
\eqref{eq:vecrule}, preceded by a diagonal decay.
Theorem~\ref{thm:Achunkwise} is the WY representation \citep{bischof87wy}
of the product of these factors.

\begin{corollaryA}[WY form]
\label{cor:Achunkwy}
Let $\widehat{\BA} := [\widehat{\Bk}_1 \otimes \Bv_1, \dots,
\widehat{\Bk}_C \otimes \Bv_C]$ and $\breve{\BA} := [\breve{\Bk}_1 \otimes
\Bv_1, \dots, \breve{\Bk}_C \otimes \Bv_C]$ in $\dR^{d_k d_v \times C}$.
Then $\Bs_C = \BP\, \Bs_0 + \Bd$ with
\begin{equation} \label{eq:Acw-wy}
\BP \;=\; \prod_{r=C}^{1} (\BI - \beta_r \Ba_r \Ba_r^\top)(\BF_r \otimes \BI_{d_v})
\;=\; (\BGa_C \otimes \BI_{d_v}) \;-\; \breve{\BA}\, \BLa \diag(\Bbe)\,
\widehat{\BA}^\top , \qquad
\Bd \;=\; \breve{\BA}\, \BLa\, \Bc .
\end{equation}
\end{corollaryA}
\begin{proof}
By \eqref{eq:Acw-state}, $\Bs_C = (\BGa_C \otimes \BI) \Bs_0 + \sum_r u_r
(\breve{\Bk}_r \otimes \Bv_r) = (\BGa_C \otimes \BI) \Bs_0 + \breve{\BA}
\Bu$. Since $p_r = \langle \Bv_r \widehat{\Bk}_r^\top, \BS_0 \rangle_F =
\widehat{\Ba}_r^\top \Bs_0$, we have $\Bp = \widehat{\BA}^\top \Bs_0$, and
\eqref{eq:Acw-u} gives $\Bu = \BLa \Bc - \BLa \diag(\Bbe) \widehat{\BA}^\top
\Bs_0$.
\end{proof}

The correction to the diagonal decay has rank at most $C$, as for DeltaNet,
where $\BP = \BI - \sum_s \Bw_s \Bk_s^\top$. The difference is that the
columns of $\breve{\BA} \BLa \diag(\Bbe)$ are linear combinations of
Kronecker products $\breve{\Bk}_s \otimes \Bv_s$ and are in general not
themselves Kronecker products, so the correction is not of the delta-rule
form $\sum_s (\BS_0 \Bw_s) \Bk_s^\top$, in which the memory enters through
the $d_v$-dimensional read-outs $\BS_0 \Bw_s$. Instead the memory enters
through the $C$ scalars $\Bp$, and each pseudo-value is $\Bv_s$ scaled by a
linear functional of $\BS_0$. Table~\ref{tab:Achunkwise-compare} summarizes
the roles of the chunk quantities in the delta-rule family and in KronLSTM.

\begin{table}[t]
\caption{Chunkwise forms of delta-rule memories and of KronLSTM, in the
orientation $\BS \in \dR^{d_v \times d_k}$ of this paper. $\BLa$ is the
inverse of the unit lower triangular matrix $\BI + \BN$ built from the listed
Gram; ``coupling'' is the term through which the incoming memory $\BS_0$
enters the write of token $r$; $\operatorname{stril}$ is the strictly lower
triangular part. For KDA and Gated DeltaNet-2 the decay-weighted key Gram is
$\BG^{kk}$ of \eqref{eq:Acw-grams}; for Gated DeltaNet-2, $\Bb_r \in
[0,1]^{d_k}$ and $\Bw_r \in [0,1]^{d_v}$ are its channel-wise erase and
write gates.}
\label{tab:Achunkwise-compare}
\centering
\resizebox{\linewidth}{!}{
\begin{tabular}{@{}lllll@{}}
\toprule
 & Gram in the solve & right-hand side & coupling to $\BS_0$ & write of token $r$ \\
\midrule
DeltaNet & $\diag(\Bbe)\operatorname{stril}(\BK\BK^\top)$ &
  $\diag(\Bbe)\BV$ &
  $\BS_0 \Bw_r$, $\;\BW = \BLa\diag(\Bbe)\BK$ &
  $(\Bu_r - \BS_0 \Bw_r)\,\Bk_r^\top$ \\
Gated DeltaNet & same, $\times\, \gamma_r/\gamma_s$ (scalar) & same & same, decayed & same, decayed \\
KDA & $\diag(\Bbe)\,\BG^{kk}$ &
  $\diag(\Bbe)\BV$ &
  $\BS_0 \Bw_r$, $\;\BW = \BLa\diag(\Bbe)\widehat{\BK}$ &
  $(\Bu_r - \BS_0 \Bw_r)\,\breve{\Bk}_r^\top$ \\
Gated DeltaNet-2 & $\operatorname{stril}\big((\Bga \odot \BB \odot \BK)\bar{\BK}^\top\big)$ &
  $\BW \odot \BV$ &
  $\BS_0 \By_r$, $\;\BY = \BLa(\Bga \odot \BB \odot \BK)$ &
  $(\Bu_r - \BS_0 \By_r)\,\breve{\Bk}_r^\top$ \\
KronLSTM & $\diag(\Bbe)\,(\BG^{kk} \odot \BG^{vv})$ &
  $\Bc$ &
  $\big(\BLa\diag(\Bbe)\Bp\big)_r$, $\;p_r = \Bv_r^\top\BS_0\widehat{\Bk}_r$ &
  $u_r\, \Bv_r \breve{\Bk}_r^\top$ \\
\bottomrule
\end{tabular}}
\end{table}

\subsection{Variants}
\label{sec:Achunkwise-variants}

The derivation uses only the structure ``diagonal decay on the key side,
rank-one write along $\Bv_r \Bk_r^\top$ with a coefficient affine in the
scalar score'':
\begin{itemize}[nosep,leftmargin=*]
\item \emph{Coupled erase and write} (variant ``$\beta$ on both sides'' of
  Table~\ref{tab:write-rule-ablation}; target $c_r$ instead of $\tilde c_r
  = c_r / \beta_r$): replace $\Bc$ by $\Bbe \odot \Bc$ in \eqref{eq:Acw-u}.
  \emph{Unit target} (variant ``non-learnable target''): $\Bc = \BOn$, so
  $\Bu^{\circ} = \BLa \BOn$.
\item \emph{State tracking}, $\beta_r \in (0,2)$
  (Section~\ref{sec:state-tracking}): nothing changes. $\BI + \BN$ is unit
  lower triangular for every $\Bbe$, so the solve needs no condition on the
  gates or on the data.
\item \emph{Scalar forget gate} $\bff_r = f_r \BOn$ (as in Gated DeltaNet):
  $\Bga_r$ is a scalar and $\BG^{kk}_{rs} = \mathbbm{1}[{r > s}]\, (\Bk_r^\top \Bk_s)\, \gamma_r / \gamma_s$. \emph{No forget gate}:
  $\BG^{kk} = \operatorname{stril}(\BK \BK^\top)$ and $\breve{\BK} =
  \widehat{\BK} = \bar{\BK} = \BK$.
\item \emph{Values of arbitrary norm.} The kernels accept unnormalized
  values by the substitution $\beta_r \leftarrow \beta_r / \norm{\Bv_r}^2$
  (the erase step is capped so that $\beta_r \norm{\Ba_r}^2 = \beta_r$ for
  unit-norm keys and the contraction condition of
  Property~\ref{prop:Atimecontraction} holds independently of
  $\norm{\Bv_r}$); the write gate is unaffected. With the
  $\ell_2$-normalized values of Section~\ref{subsec:kronlstm-block} this is
  \eqref{eq:Acw-rec}.
\end{itemize}

\subsection{Algorithm, cost and implementation}
\label{sec:Achunkwise-algorithm}

\paragraph{Algorithm (forward pass, one head).}
Given $\BQ, \BK \in \dR^{L \times d_k}$, $\BV \in \dR^{L \times d_v}$,
$\Bbe, \Bc \in \dR^{L}$, $\log \bff \in \dR^{L \times d_k}$, a chunk length
$C$ and an initial memory $\BS_{[0]}$:
\begin{enumerate}[nosep,leftmargin=*]
\item \emph{Chunk-local, in parallel over chunks.} For every chunk:
  $\bm{\ell} \leftarrow \operatorname{cumsum}(\log \bff)$ within the chunk,
  $\Bga \leftarrow \exp(\bm{\ell})$, and $\widehat{\BK}, \bar{\BK},
  \breve{\BK}, \widehat{\BQ}$ as in \eqref{eq:Acw-keys};
  $\BG^{kk}, \BG^{vv}, \BG^{qk}$ as in \eqref{eq:Acw-grams}, the entries of
  $\BG^{kk}$ and $\BG^{qk}$ via $\exp(\bm{\ell}_r - \bm{\ell}_s)$ for $s \le r$;
  $\BN \leftarrow \diag(\Bbe)(\BG^{kk} \odot \BG^{vv})$ and
  $\BLa \leftarrow (\BI + \BN)^{-1}$ by forward substitution;
  $\Bu^{\circ} \leftarrow \BLa \Bc$ and $\BLa_{\beta} \leftarrow \BLa \diag(\Bbe)$.
\item \emph{Scan, sequential over chunks.} $\BS \leftarrow \BS_{[0]}$; for
  $n = 0, 1, \dots$:
  $\Bp \leftarrow \big((\BV \BS) \odot \widehat{\BK}\big) \BOn$;
  $\Bu \leftarrow \Bu^{\circ} - \BLa_{\beta}\, \Bp$;
  $\widetilde{\BV} \leftarrow \diag(\Bu) \BV$;
  $\BH_{[n]} \leftarrow \widehat{\BQ} \BS^\top + \BG^{qk} \widetilde{\BV}$;
  $\BS \leftarrow \BS \diag(\Bga_C) + \widetilde{\BV}^\top \breve{\BK}$.
\item Return $\BH \in \dR^{L \times d_v}$ and the final memory $\BS$.
\end{enumerate}

\paragraph{Cost.}
Per chunk and head the algorithm performs the two $C \times C$ key Grams
($2 C^2 d_k$ multiply--adds), the value Gram ($C^2 d_v$), one unit lower
triangular solve applied to two right-hand sides ($O(C^3)$, in practice on
$16 \times 16$ blocks as in the delta-rule kernels), three products with the
memory, $\BV \BS$, $\widehat{\BQ} \BS^\top$ and $\widetilde{\BV}^\top
\breve{\BK}$ ($3 C d_k d_v$), and the intra-chunk product $\BG^{qk}
\widetilde{\BV}$ ($C^2 d_v$). Over a sequence of length $L$ this is
$O\big(L C (d_k + d_v) + L d_k d_v\big)$ multiply--adds and $O(L C^2)$ for
the solves, the same order as Gated DeltaNet and KDA; relative to KDA the
additional cost is the value Gram $\BG^{vv}$, and the coupling $\Bp$
replaces the product $\BW \BS_0$ of the delta rule at equal cost. The
memory is materialized once per chunk, not per token.

\paragraph{Implementation notes.}
(i) The cumulative decays are kept in log space; the Gram entries use
$\exp(\bm{\ell}_r - \bm{\ell}_s)$ for $s \le r$ only, so all decay factors
lie in $(0,1]$ and no quotient $1/\Bga$ is ever formed (the kernels use
base-2 logarithms and $\exp_2$). (ii) The triangular solve is the
precision-sensitive step and is carried out in IEEE fp32; the remaining
products run on tensor cores. (iii) In the fused kernel the solve is not
materialized as $\Bu^{\circ}$ and $\BLa_{\beta}$ but applied inside the
scan as $\Bu = \BLa(\Bc - \Bbe \odot \Bp)$. (iv) Because $\Bp$ contracts
over the full value dimension, programs that tile the memory over $d_v$
exchange their partial row sums of $(\BV \BS) \odot \widehat{\BK}$ once per
chunk; this is a synchronization, not a bandwidth cost ($C$ floats per
program against a $d_k \times d_v$ memory). (v) Padding a sequence to a
multiple of $C$ with zero $\Bq, \Bk, \Bv$, $\beta = c = 0$ and $\log \bff =
\BZe$ neither writes nor decays. (vi) Queries are scaled by $d_k^{-1/2}$
inside the kernel; the scale is immaterial after the RMSNorm of the block.
In the supplied code, $\widehat{\BK}, \bar{\BK}, \breve{\BK},
\widehat{\BQ}$ are \texttt{k\_up}, \texttt{k\_dn}, \texttt{k\_tail},
\texttt{q\_up}; $\BG^{kk}, \BG^{vv}, \BG^{qk}$ are \texttt{A\_kk},
\texttt{A\_vv}, \texttt{A\_qk}; $\BLa$, $\Bu^{\circ}$, $\BLa_\beta$ are
\texttt{Tinv}, \texttt{u\_const}, \texttt{Tw}; and $\Bp$, $\Bu$,
$\widetilde{\BV}$ are \texttt{p}, \texttt{u}, \texttt{uv}.

\subsection{Backward pass}
\label{sec:Achunkwise-backward}

The backward pass follows the forward decomposition chunk by chunk in
reverse order. Let $\mathrm{d}\BH \in \dR^{C \times d_v}$ and
$\mathrm{d}\BS_C \in \dR^{d_v \times d_k}$ be the upstream gradients of a
chunk, the latter including the gradient $\mathrm{d}\BS_0$ returned by the
following chunk. Write $\Bg := \Bc - \Bbe \odot \Bp$ and $\BY_0 := \BV
\BS_0 \in \dR^{C \times d_k}$. With $\operatorname{stril}$ and
$\operatorname{tril}$ the strictly lower and the lower triangular part, the
vector--Jacobian products of \eqref{eq:Acw-state} and \eqref{eq:Acw-output}
are
\begin{align}
\mathrm{d}\BS_0 &= \mathrm{d}\BS_C \BGa_C + \mathrm{d}\BH^\top \widehat{\BQ}, &
\mathrm{d}\Bga_C &= (\mathrm{d}\BS_C \odot \BS_0)^\top \BOn, &
\mathrm{d}\breve{\BK} &= \widetilde{\BV}\,\mathrm{d}\BS_C, \notag\\
\mathrm{d}\widetilde{\BV} &= \breve{\BK}\,\mathrm{d}\BS_C^\top + \BG^{qk\top}\mathrm{d}\BH, &
\mathrm{d}\widehat{\BQ} &= \mathrm{d}\BH\,\BS_0, &
\mathrm{d}\BG^{qk} &= \operatorname{tril}(\mathrm{d}\BH\,\widetilde{\BV}^\top),
\label{eq:Acw-bwd1}
\end{align}
those of the pseudo-values and of the solve \eqref{eq:Acw-u} are
\begin{align}
\mathrm{d}\Bu &= (\mathrm{d}\widetilde{\BV} \odot \BV)\BOn, &
\mathrm{d}\BV &\mathrel{+}= \diag(\Bu)\,\mathrm{d}\widetilde{\BV}, &
\mathrm{d}\BLa &= \mathrm{d}\Bu\,\Bg^\top, &
\mathrm{d}\Bg &= \BLa^\top \mathrm{d}\Bu, \notag\\
\mathrm{d}\Bc &= \mathrm{d}\Bg, &
\mathrm{d}\Bbe &= -\mathrm{d}\Bg \odot \Bp, &
\mathrm{d}\Bp &= -\mathrm{d}\Bg \odot \Bbe, &
\mathrm{d}\BN &= -\operatorname{stril}\big(\BLa^\top \mathrm{d}\BLa\,\BLa^\top\big),
\label{eq:Acw-bwd2}
\end{align}
those of the scores \eqref{eq:Acw-p} and of the Gram matrices
\eqref{eq:Acw-grams}, \eqref{eq:Acw-N} are
\begin{align}
\mathrm{d}\BV &\mathrel{+}= \diag(\mathrm{d}\Bp)\,\widehat{\BK}\,\BS_0^\top
  + (\mathrm{d}\BG^{vv} + \mathrm{d}\BG^{vv\top})\BV, &
\mathrm{d}\BS_0 &\mathrel{+}= \BV^\top \diag(\mathrm{d}\Bp)\,\widehat{\BK}, \notag\\
\mathrm{d}\Bbe &\mathrel{+}= \big(\mathrm{d}\BN \odot \BG^{kk} \odot \BG^{vv}\big)\BOn, &
\mathrm{d}\BG^{kk} &= \diag(\Bbe)\,\mathrm{d}\BN \odot \BG^{vv}, \qquad
\mathrm{d}\BG^{vv} = \diag(\Bbe)\,\mathrm{d}\BN \odot \BG^{kk}, \notag\\
\mathrm{d}\widehat{\BK} &= \diag(\mathrm{d}\Bp)\,\BY_0 + \mathrm{d}\BG^{kk}\bar{\BK}, &
\mathrm{d}\bar{\BK} &= \mathrm{d}\BG^{kk\top}\widehat{\BK} + \mathrm{d}\BG^{qk\top}\widehat{\BQ}, \qquad
\mathrm{d}\widehat{\BQ} \mathrel{+}= \mathrm{d}\BG^{qk}\bar{\BK},
\label{eq:Acw-bwd3}
\end{align}
and the decay-scaled copies \eqref{eq:Acw-keys} and the cumulative sum
\eqref{eq:Acw-gamma} give
\begin{align}
\mathrm{d}\BK &= \mathrm{d}\widehat{\BK} \odot \Bga
  + \mathrm{d}\bar{\BK} / \Bga
  + \mathrm{d}\breve{\BK} \odot (\Bga_C / \Bga), \qquad
\mathrm{d}\BQ = \mathrm{d}\widehat{\BQ} \odot \Bga, \notag\\
\mathrm{d}\Bga &= \mathrm{d}\widehat{\BK} \odot \BK
  - \mathrm{d}\bar{\BK} \odot \BK / \Bga^{2}
  + \mathrm{d}\widehat{\BQ} \odot \BQ
  - \mathrm{d}\breve{\BK} \odot \BK \odot \Bga_C / \Bga^{2}, \qquad
\mathrm{d}\Bga_C \mathrel{+}= \big(\mathrm{d}\breve{\BK} \odot \BK / \Bga\big)^\top \BOn, \notag\\
\mathrm{d}\bm{\ell}_r &= \mathrm{d}\Bga_r \odot \Bga_r
  + \mathbbm{1}[r = C]\; \mathrm{d}\Bga_C \odot \Bga_C, \qquad
\mathrm{d}\log \bff_i = \sum_{r \ge i} \mathrm{d}\bm{\ell}_r
\quad \text{(reverse cumulative sum).}
\label{eq:Acw-bwd4}
\end{align}
Here $\Bga \in \dR^{C \times d_k}$ stacks the $\Bga_r$ and all products with
it are element-wise. Two points differ from the delta-rule backward. The
values enter the system matrix $\BN$ through $\BG^{vv}$, so $\mathrm{d}\BV$
receives a contribution through the inverse in
\eqref{eq:Acw-bwd2}--\eqref{eq:Acw-bwd3}, whereas in DeltaNet, Gated
DeltaNet and KDA the values appear only in the right-hand side. And since
the pseudo-values are $u_r \Bv_r$, the gradient with respect to the write
coefficients is the row-wise contraction $\mathrm{d}\Bu =
(\mathrm{d}\widetilde{\BV} \odot \BV)\BOn$ over the value dimension,
mirroring the forward contraction that produces $\Bp$. The gate factor
$\diag(\Bbe)$ sits inside the Hadamard product and must be kept inside the
accumulation of $\mathrm{d}\BN$, as observed for the channel-wise gates of
Gated DeltaNet-2 \citep{hatamizadeh2026gdn2}.

\subsection{Verification}
\label{sec:Achunkwise-verification}

We verified Theorem~\ref{thm:Achunkwise} and
\eqref{eq:Acw-bwd1}--\eqref{eq:Acw-bwd4} numerically in double precision
against the token-by-token recurrence \eqref{eq:Acw-rec}: outputs and final
memory agree to machine precision ($\le 10^{-15}$ relative) for chunk
lengths $C \in \{1, 8, 16, L\}$, sequence lengths that are and are not
multiples of $C$, $d_k \ne d_v$, nonzero initial memory, strong and weak
decay, and $\beta_r \in (0,1)$ as well as $\beta_r \in (0,2)$; the operator
identity \eqref{eq:Acw-wy} holds against the explicitly formed product of
$d_k d_v \times d_k d_v$ factors; the no-decay reduction of $\BN$ to the
Gram of $\Ba_r = \Bk_r \otimes \Bv_r$ holds; and the gradients with respect
to $\BQ, \BK, \BV, \Bbe, \Bc, \log \bff$ and $\BS_0$ agree with central
finite differences of the recurrence to $10^{-9}$. The reference
implementations in the supplementary material are checked against the same
recurrence.

\end{appendices}

\end{document}